\documentclass{article}
\usepackage[preprint]{neurips_2025}

\usepackage[title]{appendix}

\usepackage{diagbox}
\usepackage[utf8]{inputenc} % allow utf-8 input
\usepackage[T1]{fontenc}    % use 8-bit T1 fonts
\usepackage{url}            % simple URL typesetting
\usepackage{multirow}
\usepackage{booktabs}       % professional-quality tables
\usepackage{amsfonts,amssymb,amsthm}       % blackboard math symbols
\usepackage{nicefrac}       % compact symbols for 1/2, etc.
\usepackage{microtype}      % microtypography
\usepackage{hhline}
\usepackage{subfigure}
\usepackage{indentfirst}
\usepackage{mathtools}
\usepackage{amsmath}

\usepackage{algpseudocode}

\newtheorem{definition}{Definition}
\newtheorem{theorem}{Theorem}
\newtheorem{remark}{Remark}
\newtheorem{lemma}{Lemma}

\newtheorem{assumption}{Assumption}
\newtheorem{corollary}{Corollary}

\usepackage[T1]{fontenc}
\usepackage[utf8]{inputenc}
\usepackage{pgfplots}
\usepackage{graphicx}
\usepackage{pstricks}
\usepackage{color}
\usepackage{bbm}

\usepackage{tcolorbox}
\usepackage{tikz}
\usepackage{pgfplots}
\usepackage{wrapfig}
\usepackage{lipsum}  
\usetikzlibrary{positioning}
\usepackage{tcolorbox}
\usepackage{commath}
\usepackage{relsize}
\usepackage{bbm}
\usepackage{bm}
\usepackage[font={small}]{caption}
\usepackage{comment,color,soul}
\usepackage{array}
\newcolumntype{?}{!{\vrule width 1pt}}
\usepackage{enumerate} 
\usepackage{url}
\usepackage{algorithm}
\usepackage{mysymbol}
\usepackage{enumitem}
\allowdisplaybreaks

\makeatletter
\def\Let@{\def\\{\notag\math@cr}}
\makeatother

\usepackage{hyperref}
\definecolor{darkred}{RGB}{150,0,0}
\definecolor{darkgreen}{RGB}{0,150,0}
\definecolor{darkblue}{RGB}{0,0,150}
\hypersetup{colorlinks=true, linkcolor=darkred, citecolor=darkblue, urlcolor=darkblue}

\usepackage{amsmath,amsfonts,bm}

\def\eqref#1{equation~\ref{#1}}
\def\1{\bm{1}}
\def\0{\bm{0}}

\def\rt{{\textnormal{t}}}

\def\rv{{\textnormal{v}}}

\def\rvtheta{{\mathbf{\theta}}}
\def\rva{{\mathbf{a}}}
\def\rvb{{\mathbf{b}}}

\def\rvu{{\mathbf{i}}}

\def\rvm{{\mathbf{m}}}

\def\rvs{{\mathbf{s}}}
\def\rvt{{\mathbf{t}}}
\def\rvu{{\mathbf{u}}}
\def\rvv{{\mathbf{v}}}
\def\rvw{{\mathbf{w}}}
\def\rvx{{\mathbf{x}}}
\def\rvy{{\mathbf{y}}}
\def\rvz{{\mathbf{z}}}

\def\rmM{{\mathbf{M}}}
\def\rmN{{\mathbf{N}}}

\def\rmR{{\mathbf{R}}}

\def\vi{{\bm{i}}}

\def\vz{{\bm{z}}}

\def\mI{{\bm{I}}}

\def\mQ{{\bm{Q}}}

\DeclareMathAlphabet{\mathsfit}{\encodingdefault}{\sfdefault}{m}{sl}
\SetMathAlphabet{\mathsfit}{bold}{\encodingdefault}{\sfdefault}{bx}{n}

\def\gL{{\mathcal{L}}}

\def\gN{{\mathcal{N}}}
\def\gO{{\mathcal{O}}}

\def\sB{{\mathbb{B}}}

\def\sP{{\mathbb{P}}}

\def\sR{{\mathbb{R}}}
\def\sS{{\mathbb{S}}}

\newcommand{\E}{\mathbb{E}}

\newcommand{\R}{\mathbb{R}}

\newcommand{\KL}{D_{\mathrm{KL}}}

\usepackage{inputenc}
\usepackage{fontenc}
\usepackage{indentfirst}
\usepackage{cite}
\usepackage{hyperref}
\usepackage{amssymb}
\usepackage{float}
\usepackage{amsmath}
\usepackage{graphicx}
\usepackage{amsthm}
\renewcommand*{\dif}{{\,\mathrm{d}}}

\newtheorem{proposition}{Proposition}

\newcommand{\nm}[1]{\left\|{#1}\right\|}

\renewcommand{\Pr}{\sP}

\renewcommand{\set}[1]{\left\{{#1}\right\}}

\def\rveps{\bm{\epsilon}}

\def\xz{\rvx_0}

\def\muz{\bmu_0}

\def\lss{\delta}
\def\lgv{\rvx}
\def\lz{\lgv_0}
\def\lT{\lgv_T}
\def\lt{\lgv_t}
\def\ltm{\lgv_{t-1}}

\def\epst{\rveps_t}
\def\sig{{\sigma}}

\def\w{w}

\def\wz{\w_0}

\def\hkalpha{\alpha}
\def\hkbeta{\beta}
\def\hkeps{\rveps}

\def\dt{\lss_t}

\def\di{\lss_i}
\def\dj{\lss_j}

\def\r{r}
\def\muz{\rvmu_0}
\def\muo{\rvmu_1}

\def\rvalpha{{\bm{\alpha}}}

\def\var{{\nu}}
\def\nl{{\sigma}}

\def\nlt{\nl_t}
\def\nltm{\nl_{t-1}}
\def\nlz{\nl_0}

\def\nlT{\nl_T}

\def\cnl{c_{\nl}}

\def\cv{c_{\var}}

\def\Sig{{\bm{\Sigma}}}
\def\rvmu{{\bm{\mu}}}
\def\P{P}
\def\wi{w_i}
\def\probi{\P^{(i)}}
\def\mui{\rvmu_i}

\def\probz{\P^{(0)}}
\def\vz{\var_0}
\def\vi{\var_i}
\def\vmax{\var_{\max}}

\def\n{n}
\def\r{r}
\def\N{\rmN}
\renewcommand{\R}{\rmR}
\def\nv{\mathbf{\n}}
\def\rv{\mathbf{\r}}

\def\nt{\nv_t}
\def\nz{\nv_0}
\def\ntm{\nv_{t-1}}

\def\nT{\nv_T}
\def\rt{\rv_t}

\def\rtm{\rv_{t-1}}

\def\hnt{\hat{\nv}_t}
\def\hnz{\hat{\nv}_0}
\def\hntm{\hat{\nv}_{t-1}}

\def\hnT{\hat{\nv}_T}

\def\nj{\nv_j}
\def\njm{\nv_{j-1}}
\def\hnj{\hat{\nv}_j}
\def\hnjm{\hat{\nv}_{j-1}}
\def\ljm{\lgv_{j-1}}

\def\ent{\epst^{(\nv)}}
\def\ert{\epst^{(\rv)}}

\def\rhoi{\rho_i}

\def\lip{L}
\def\lipi{\lip_i}
\def\clip{c_{\lip}}

\def\lipit{\lip_{i,t}}

\def\Ball{\sB}

\def\rvtheta{\bm{\theta}}

\def\TV{\text{TV}}

\def\hP{\widehat{\P}}

\def\S{\sS}

\def\cz{c_0}
\def\lipz{L_0}

\def\LQ{L_Q}
\def\mQ{m_Q}
\def\RQ{R_Q}
\def\rhoQ{\rho_Q}

\def\probo{\P^{(1)}}
\def\probt{\P^{(2)}}

\newcommand{\duk}[1]{\left\|{#1}\right\|_{\set{\mui}_{i \in [k]}}}
\newcommand{\du}[1]{\left\|{#1}\right\|_{\rvmu}}

\title{\textbf{On the Hardness of Sampling  from Mixture Distributions via Langevin Dynamics}}

\date{}

\author{
    Xiwei~Cheng \\ CUHK \\ xwcheng@link.cuhk.edu.hk \And 
    Kexin~Fu \\ Purdue \\ fu448@purdue.edu \And 
    Farzan~Farnia \\  CUHK \\  farnia@cse.cuhk.edu.hk
}

\begin{document}
\maketitle

\begin{abstract}
    The Langevin Dynamics (LD), which aims to sample from a probability distribution using its score function, has been widely used for analyzing and developing score-based generative modeling algorithms. While the convergence behavior of LD in sampling from a uni-modal distribution has been extensively studied in the literature, the analysis of LD under a mixture distribution with distinct modes remains underexplored in the literature. In this work, we analyze LD in sampling from a mixture distribution and theoretically study its convergence properties. Our theoretical results indicate that for general mixture distributions of sub-Gaussian components, LD could fail in finding all the components within a sub-exponential number of steps in the data dimension. Following our result on the complexity of LD in sampling from high-dimensional variables, we propose \emph{Chained Langevin Dynamics (Chained-LD)}, which divides the data vector into patches of smaller sizes and generates every patch sequentially conditioned on the previous patches. Our theoretical analysis of Chained-LD indicates its faster convergence speed to the components of a mixture distribution. We present the results of several numerical experiments on synthetic and real image datasets, validating our theoretical results on the iteration complexities of sample generation from mixture distributions using the vanilla and chained LD algorithms.

\end{abstract}

\section{Introduction} \label{sec:intro}

Langevin dynamics (LD) is a well-established methodology with a wide range of applications to various areas, including Bayesian learning \citep{welling2011bayesian}, non-convex optimization \citep{raginsky2017non, xu2018global}, and molecular-dynamics simulations \citep{paquet2015molecular, gottwald2015parametrizing}. The LD sampling approach leverages the score function of a probability density function (PDF) $P(\lgv)$, defined as the gradient of the PDF logarithm $\nabla \log P(\lgv)$, to perform the following iterative process whose output follows the probability model characterized by $P(\lgv)$
\vspace{-1mm}
\begin{equation*}
    \lt = \ltm + \frac{\dt}{2} \nabla_\lgv \log \P(\ltm) + \sqrt{\dt} \epst,
\end{equation*}
where $\dt$ is the step size and $\epst \sim \gN(\0_d, \mI_d)$ is Gaussian noise. Recently, the LD sampling methodology has found central applications in generative modeling tasks, such as image generation \citep{song2019generative, song2020score}, adversarial training \citep{kamalaruban2020robust, srinivasan2021robustifying}, and imitation learning \citep{reuss2023goal,pearce2023imitating}, which have inspired many theoretical and empirical studies of the LD methodology.

Specifically, several references \citep{durmus2017nonasymptotic, dalalyan2017theoretical, cheng2018underdamped, cheng2018convergence} have studied the convergence properties of the LD sampling process to characterize the iteration complexity of LD sampling from the target PDF $P(\mathbf{x})$. The existing theoretical results mostly focus on demonstrating the satisfactory and speedy convergence of LD assuming a unimodal target distribution consisting of only one distribution component. However, the work of Song and Ermon \citep{song2019generative} has highlighted examples of mixture distributions with multiple separated modes where the vanilla LD sampling struggles in capturing the mode frequencies correctly. This reference \citep{song2019generative} proposes a variant of LD, called \emph{annealed Langevin dynamics (Annealed-LD)}, to address the challenges of LD in sampling from a mixture distribution.

While the Annealed-LD sampling approach provides a satisfactory solution to cover the modes of a mixture model, the theoretical analysis of LD sampling applied to mixture distributions has remained underexplored in the literature. However, a reliable application of LD methods to sample from real-world distributions requires a more solid understanding of their convergence behavior for a target multi-modal distribution, which is commonly present in real-world data due to the different background features of real-world objects and phenomena.

In our work, we aim to study the convergence properties of the LD framework in sampling from a mixture distribution. As displayed in Figure \ref{fig:intro:syn}, we observe that the convergence of LD to capture all three underlying modes becomes more challenging when the dimension $d$ of the sampling space is growing. Our main theoretical result provides a family of mixture distributions, where the LD framework is unlikely to find all the mixture components within a sub-exponential number of iterations in the data dimension $d$.

Specifically, we consider mixture distributions with a low-probability yet high-variance in-between mode, which we refer to as the zeroth mode $P^{(0)}$ (illustrated in Figure~\ref{fig:multimodal_distribution}). Despite a significantly smaller probability mass compared to the other low-variance modes, the in-between mode $P^{(0)}$ surrounds the other low-variance modes and fills the space between them. As a result, Mode 0 dominates the score function in the low-density region, disrupting and slowing down the convergence of the noisy local search in LD to the low-variance modes with greater probability masses.

\begin{figure}
    \centering
    
     \subfigure[ Mixture target distribution] 
         {\includegraphics[width=0.32\textwidth]{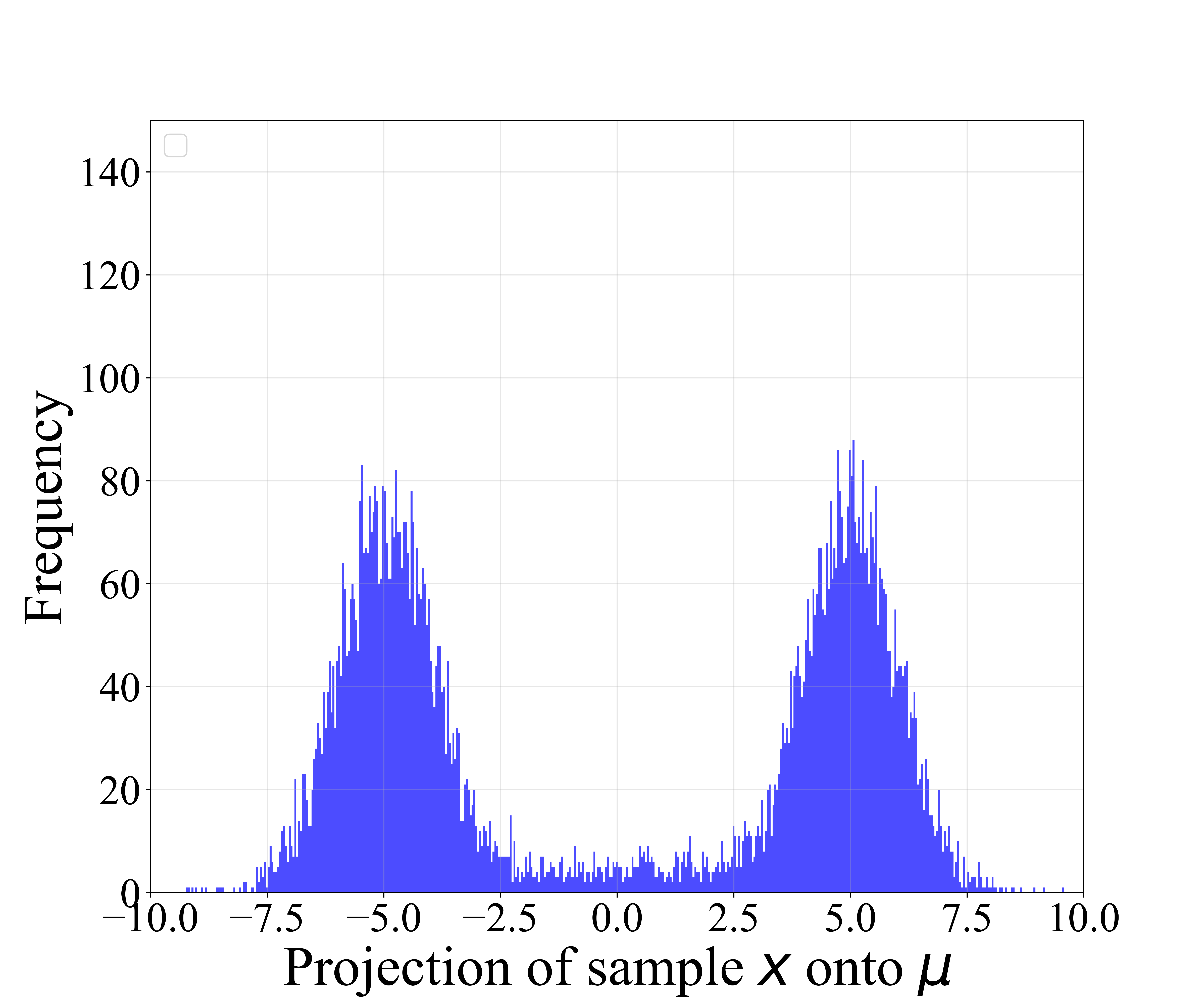}}
     \subfigure[LD samples (scalar), $d=1$] 
         {\includegraphics[width=0.32\textwidth]{figs/od_dim1_step10000.png}}
     \subfigure[LD samples along $\overrightarrow{\1_d}$, $d=10$] 
         {\includegraphics[width=0.32\textwidth]{figs/od_dim10_step10000.png}}
    \caption{Samples by Langevin dynamics from a mixture target distribution $\P = 0.1 \gN(\0_d,10\mI_d) + 0.45 \gN(5 \cdot \1_d, \mI_d) + 0.45 \gN(-5 \cdot \1_d,\mI_d)$ with data dimensions $d=1 $ and $d=10$. The samples are initialized using $\gN(\0_d,\mI_d)$, and Langevin dynamics is applied for $T = 10^4$ iterations. The histogram is plotted by sampling $10^4$ vectors $\lgv$ and projecting them along the mean vector $\1_d$. Stepsizes $\dt$ are selected following \citep{song2019generative}. } \label{fig:intro:syn}
\vspace{-4mm}
\end{figure}

To mitigate the exponential iteration complexity, we introduce a complementary method, {\it Chained Langevin Dynamics (Chained-LD)}, with convergence guarantees in a polynomial number of iterations. Following our theoretical results on the role of high dimensionality in the convergence of LD, we propose applying dimensionality reduction through the Chain Rule: for $\lgv = [x^1, x^2, \cdots, x^d] \in \sR^d$, 
\begin{equation*}
    \P(\lgv) = \P(x^1) \P(x^2 | x^1) \cdots \P(x^d | x^1,\cdots,x^{d-1}). 
\end{equation*}
Chained-LD sequentially samples every element $\lgv^{i}$ for all $i \in [d]$ from the conditional distribution given previous elements, i.e., $\P(x^{i} \mid x^{1}, \cdots x^{i-1})$. Therefore, Chained-LD reduces the effective dimensionality of the sampled variable, which can accelerate the search for missing modes in sampling from a mixture distribution. Furthermore, for mixture distributions $\P$ such that $-\log \P(x^i| x^1,\cdots,x^{i-1})$ is $\LQ$-smooth and $\mQ$-strongly convex for $|x^i|>\RQ$, we theoretically show that Chained-LD converges to the target distribution within $\varepsilon$ total variation distance in $\gO \left( \frac{\LQ^2 d^3}{\mQ^2 \varepsilon^2} \exp(32\LQ \RQ^2) \log \frac{d^3}{\varepsilon^2} \right) = \gO \left( \frac{d^3}{\varepsilon^2} \log \frac{d^3}{\varepsilon^2} \right)$ iterations.

Finally, we present the results of several numerical experiments to validate our theoretical findings. In synthetic experiments, we consider high-dimensional Gaussian mixture models, where LD could not find all components within a million steps, whereas Chained-LD could capture all components with correct frequencies in $\gO(10^4)$ steps. Also, we test the application of Chained-LD as a sampling algorithm in score-based generative modeling for an underlying mixture distribution. In the case of a mixture of original images from the MNIST/Fashion-MNIST dataset (black background and white digits/objects) and flipped images (white background and black digits/objects), our numerical results suggest that Chained-ALD could find both the modes in $\gO(10^5)$ iterations. We summarize the contributions of this work as follows:
\begin{itemize}[leftmargin=*]
    \item Analyzing the iteration complexity of Langevin dynamics under high-dimensional mixture distributions,
    \item Proposing Chained Langevin Dynamics (Chained-LD) with sequential sampling to improve LD's convergence in sampling from mixture distributions,
    \item Providing a theoretical analysis of the convergence of Chained-LD, 
    \item Presenting numerical results validating our theoretical findings on the convergence of LD and Chained-LD. 
\end{itemize}

\textbf{Notations:} We use $[k]$ to denote the set $\set{1,2,\cdots,k}$ and $\set{a_i}_{i \in [k]}$ to denote the set $\{a_1, \cdots, a_k\}$. $\nm{\cdot}$ refers to the $\ell_2$ norm. We use $\0_n$ and $\1_n$ to denote a 0-vector and 1-vector of length $n$. We use $\mI_n$ to denote the identity matrix of size $n \times n$. $\TV$ stands for the total variation distance.

\section{Related Works} \label{sec:related_work}

\textbf{Convergence Guarantees of Langevin Dynamics: }
The convergence guarantees of Langevin diffusion, a continuous version of Langevin dynamics, are classical results that have been extensively studied in the literature \citep{bhattacharya1978criteria, roberts1996exponential, bakry1983diffusions, bakry2008simple}. Langevin dynamics, also known as Langevin Monte Carlo, is a discretization of Langevin diffusion typically modeled as a Markov Chain Monte Carlo (MCMC) method. For uni-modal distributions, e.g., log-concave probability density functions, the convergence of Langevin dynamics is provably fast \citep{dalalyan2017theoretical, durmus2017nonasymptotic, cheng2018underdamped, cheng2018convergence}. However, for multi-modal distributions, the non-asymptotic convergence analysis becomes significantly more challenging. \citep{lee2018beyond} studied Langevin dynamics under mixtures of Gaussians with equal variance and showed that the iteration complexity of Langevin dynamics is $\text{poly}(d,1/\varepsilon)$. For more general distributions, \citep{cheng2018sharp} and \citep{ma2019sampling} analyzed target distributions $p$ that are strongly log-concave outside of a region of radius $R$, proving that the iteration complexity of Langevin dynamics is $\exp(cR^2) \text{poly}(d,1/\varepsilon)$, which can become exponential in $d$ when the radius $R$ scales as $\gO(\sqrt{d})$.

\textbf{Hardness of Langevin Dynamics in Mixture Distributions: }
For continuous Langevin diffusion, \citep{bovier2002metastability, bovier2004metastability, gayrard2005metastability} studied the mean hitting time and provided a lower bound on the transition time between two modes, e.g., two local maxima. In the context of Langevin dynamics, \citep{lee2018beyond} proved the existence of a mixture of two Gaussian distributions with covariance matrices differing by a constant factor, wherein Langevin dynamics cannot find both modes in polynomial time. \citep{song2019generative} studied the slow mixing and incorrect relative weight recovery of Langevin dynamics in bi-modal distributions separated by low-density regions. Additionally, \citep{song2020improved} studied the role of noise levels in annealed Langevin dynamics, showing their effect on sample diversity in multi-modal distributions.

\textbf{Connections between Langevin Dynamics and Score-based Generative Modeling: }
Langevin dynamics and its annealed variant serve as the backbone of score-based generative modeling, which aims to learn the underlying probability distribution of training data and efficiently generate new data from the learned distribution. \citep{song2019generative} proposed learning Noise Conditional Score Networks (NCSN) via score matching to estimate the perturbed score function of the underlying distribution from training data and applied annealed Langevin dynamics with NCSN as the sampling method. \citep{song2020score} unified anneal Langevin dynamics and Denoising diffusion probabilistic modeling (DDPM) \citep{ho2020denoising} via a stochastic differential equation (SDE) and proposed utilizing score-based Markov Chain Monte Carlo (MCMC) approaches, e.g., Langevin dynamics, to sample from the SDE.

\section{Preliminaries} \label{sec:preliminaries}
\begin{figure}
    \centering
     \subfigure[Symmetric 3-Gaussian model] 
         {\includegraphics[width=0.4\textwidth]{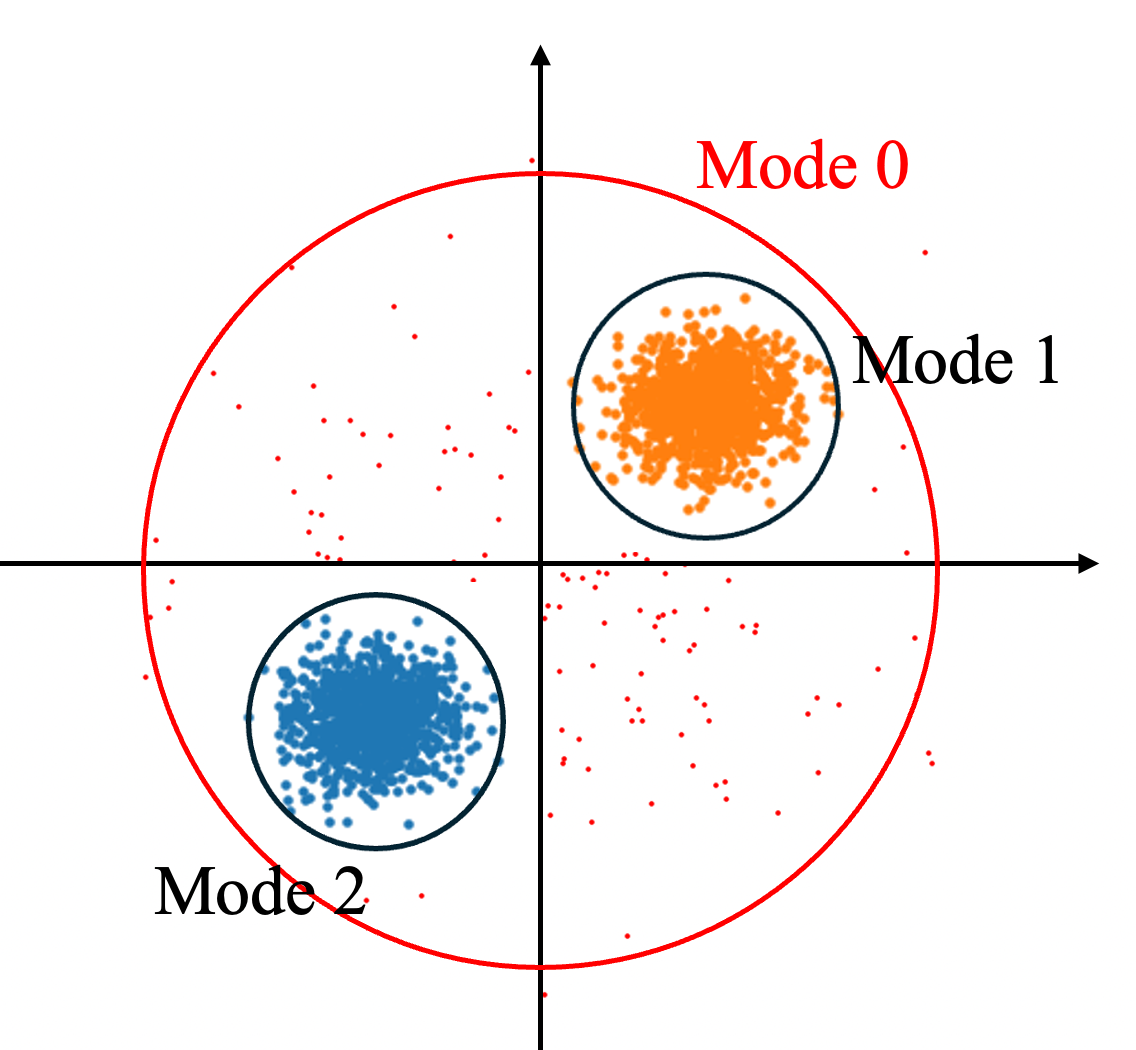}}
     \subfigure[Multimodal distribution] 
         {\includegraphics[width=0.55\textwidth]{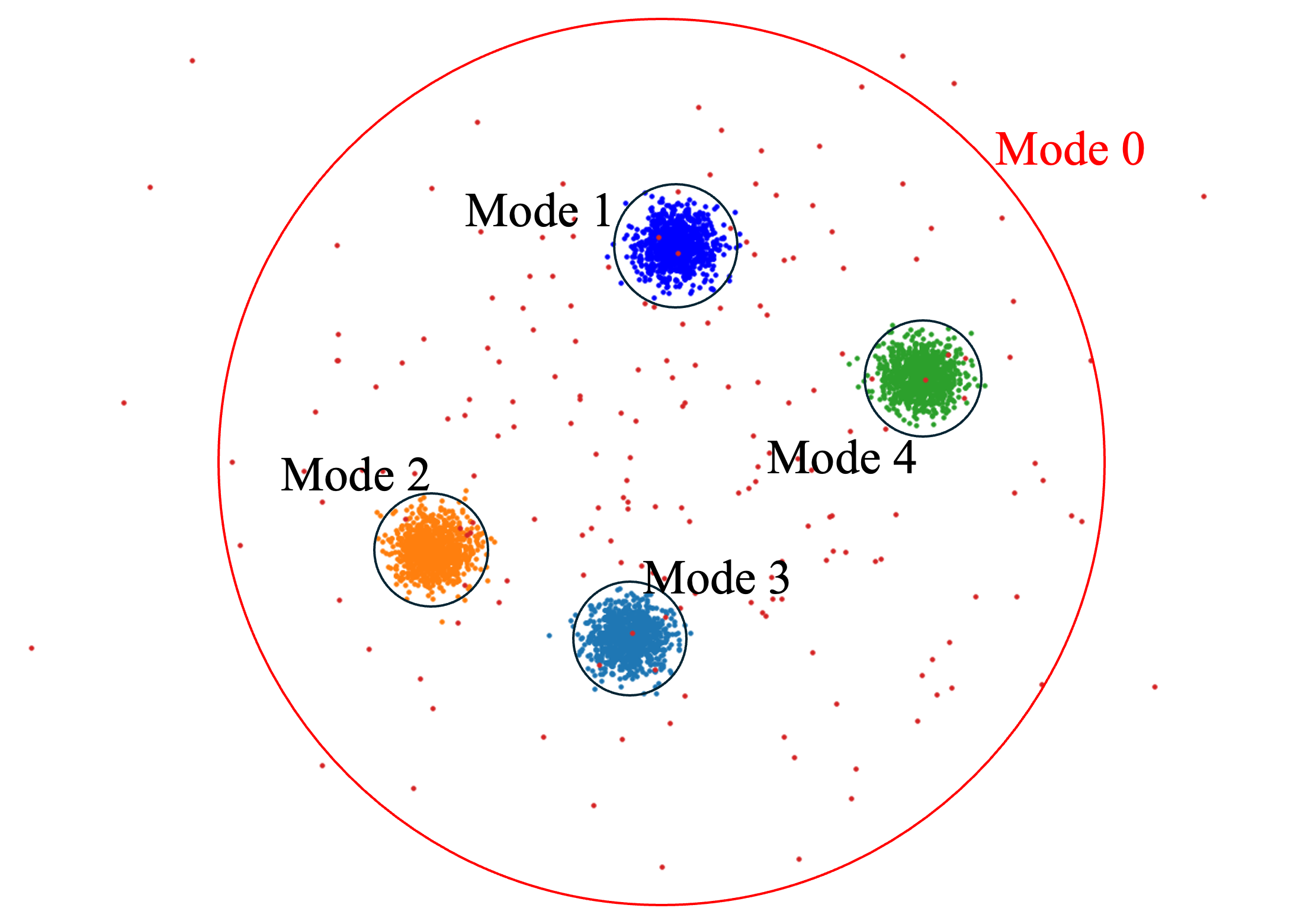}}
    \caption{Our analyzed mixture distribution possessing the in-between mode $P^{(0)}$. $P^{(0)}$ is supposed to contain a minor probability mass, yet with a significantly higher variance than the other modes $P^{(1)},\ldots, P^{(k)}$. } \label{fig:multimodal_distribution}
\end{figure}

\subsection{Langevin Dynamics}
Langevin dynamics aims to produce samples such that their distribution is close to the underlying true distribution $\P$. For a continuously differentiable probability density $\P(\lgv)$ on $\sR^d$, its score function is defined as the gradient of the log probability density function (PDF) $\nabla_\lgv \log \P(\lgv)$. Langevin diffusion is a stochastic process defined by the stochastic differential equation (SDE)
\begin{equation*}
    \dif \lgv_t = \nabla_\lgv \log \P(\lgv_t) \dif t + \sqrt{2} \dif \rvw_t,
\end{equation*}
where $\rvw_t$ is the Wiener process on $\sR^d$. Langevin dynamics, a discretization of the SDE for $T$ iterations, is applied to sample from the target distribution. Each iteration of Langevin dynamics is defined as 
\begin{equation} \label{eq:ld:ld_def}
    \lt = \ltm + \frac{\dt}{2} \nabla_\lgv \log \P(\ltm) + \sqrt{\dt} \epst,
\end{equation}
where $\dt$ is the step size and $\epst \sim \gN(\0_d, \mI_d)$ is Gaussian noise. It has been widely recognized that the continuous Langevin diffusion could take an exponential time to mix without additional assumptions on the probability density \citep{bovier2002metastability, bovier2004metastability, gayrard2005metastability}. To combat the slow mixing, \citep{song2019generative} proposed annealed Langevin dynamics by perturbing the probability density with Gaussian noise of variance $\nl^2$, i.e., 
\begin{equation} \label{eq:ald:perturb}
    \P_\nl (\lgv) := \int \P(\rvz) \gN(\lgv \mid \rvz, \nl^2 \mI_d) \dif \rvz,
\end{equation}
and applying Langevin dynamics on the perturbed data distribution $\P_{\nlt}(\lgv)$ with gradually decreasing noise levels $\nl_{1} \ge \nl_{2} \ge \cdots \ge \nlT$, i.e., 
\begin{equation} \label{eq:ald:ald_def}
    \lt = \ltm + \frac{\dt}{2} \nabla_\lgv \log \P_{\nlt}(\ltm) + \sqrt{\dt} \epst,
\end{equation}
where $\dt$ is the step size and $\epst \sim \gN(\0_d, \mI_d)$ is Gaussian noise. When the noise level $\nl$ is vanishingly small, the perturbed distribution is close to the true distribution, i.e., $\lim_{\nl \to 0} \P_\nl (\lgv) \approx \P(\lgv)$. 

\begin{remark}
    In our theoretical analysis, we assume the sampler has access to the underlying score function $\nabla_\lgv \log \P_{\sig}(\lgv)$. For generative modeling tasks in real-world datasets, since we do not have direct access to the (perturbed) score function, \citep{song2019generative} proposed the Noise Conditional Score Network (NCSN) $\rvs_{\rvtheta}(\lgv, \nl)$ to jointly estimate the scores of all perturbed data distributions, i.e., $\forall \nl \in \set{\nlt}_{t \in [T]},\; \rvs_{\rvtheta}(\lgv, \nl) \approx \nabla_\lgv \log \P_{\nl} (\lgv)$. 
\end{remark}

\subsection{Multi-Modal Distributions}
In this work, we focus on the analysis of Langevin dynamics in multi-modal distributions. We highlight that our work studies Langevin dynamics under multi-modal distributions in a slightly different setting from the standard theory literature on sampling. The existing theoretical literature commonly considers a mixture of well-separated modes with bounded variance. On the other hand, in our analysis, we consider a low-density high-variance mode (referred to as Mode 0 or $P^{(0)}$) surrounding the other modes and filling the low density region between the modes. Specifically, as illustrated in Figure \ref{fig:multimodal_distribution}.(a), we formulate a {\it symmetric 3-Gaussian model} as a hard example for Langevin dynamics, defined as following
\begin{definition} \label{def:3g}
For any given frequency $w \in (0,1)$ and variance $\nu^2 > 1$ of the in-between mode $\probz$, and any mean vector $\rvmu$ of the low-variance mode $\probo$, a symmetric 3-Gaussian model is defined as 
\begin{equation*}
    \P_{w,\nu,\rvmu} = w \gN(\0, \nu^2 \mI) + \frac{1-w}{2} \gN(\rvmu, \mI) + \frac{1-w}{2} \gN(-\rvmu, \mI) . 
\end{equation*}
\end{definition}
More generally, we use $\P = \wz \probz + \sum_{i \in [k]} \wi \probi$ to represent a mixture of $k+1$ modes, where $\probz$ is the in-between mode with high variance as illustrated in Figure \ref{fig:multimodal_distribution}.(b). Here each mode $\probi$ is a probability density with frequency $\wi$ such that $\wi > 0$ for all $i \in [k]$ and $\wz + \sum_{i \in [k]} \wi = 1$.

\section{Theoretical Analysis of the Hardness of Langevin Dynamics} \label{sec:theory}

In this section, we theoretically investigate the iteration complexity of Langevin dynamics. We first introduce a notation $\duk{\lgv}$ to measure the distance between a sample $\lgv \in \sR^d$ to the linear span of mean vectors $\set{\mui}_{i \in [k]}$ of the mixture components in a multi-modal distribution. 

\begin{definition} \label{def:dist}
    For a sample $\lgv \in \sR^d$ and a set of vectors $\rvmu_1, \cdots, \rvmu_k \in \sR^d$, we define $\duk{\lgv}$ as the distance from $\lgv$ to the span of $\set{\mui}_{i \in [k]}$, i.e., the minimum distance from $\lgv$ to any linear combination of $\set{\mui}_{i \in [k]}$:
    \begin{equation} \label{eq:def:dist}
        \duk{\lgv} := \min_{\lambda_1, \cdots, \lambda_k} \big\| {\lgv - \sum_{i=1}^k \lambda_i \mui} \big\|. 
    \end{equation}
\end{definition}

We aim to show that in a mixture distribution $\P$ with a high-variance mode $\probz$, the sampled vector $\lgv$ is likely to be far from the low-variance modes $\probo, \cdots, \P^{(k)}$ in terms of the $\duk{\lgv}$ metric.

\subsection{Langevin Dynamics in Symmetric 3-Gaussian Model}
We begin our theoretical analysis with a simple case: a symmetric 3-Gaussian model consisting of two symmetric Gaussian modes $\probo = \gN(\rvmu, \mI_d)$ and $\probt = \gN(-\rvmu, \mI_d)$, and an in-between mode $\probz = \gN(\0_d, \nu^2 \mI_d)$ with high variance $\nu^2 \ge 3$, as illustrated in Figure \ref{fig:multimodal_distribution}.(a). In the following Theorem \ref{thm:3g:3gaussian}, we show that with high probability, the sampled vector $\lT$ fails to find the symmetric modes $\probo,\probt$ within a sub-exponential number of iterations. The proof of Theorem \ref{thm:3g:3gaussian} is deferred to Appendix \ref{app:prf:3g}.

\begin{theorem} \label{thm:3g:3gaussian}
    Consider a distribution $\P_{w,\nu,\rvmu} = w \gN(\0_d,\nu^2 \mI_d) + \frac{1-w}{2} \gN(\rvmu, \mI_d) + \frac{1-w}{2} \gN(-\rvmu, \mI_d)$ by Definition \ref{def:3g} in dimension $d \ge 250$, such that $w \ge 0.01$, $\nu^2 \ge 3$, and $\nm{\rvmu}^2 \le 0.2 d$. We initialize the sample $\lz$ such that $\du{\lz}^2 \ge \frac{3\nu^2+1}{4} d$ and apply Langevin dynamics for $T$ iterations, then we have 
    $$ \Pr \left( \du{\lT}^2 \ge \frac{\nu^2+1}{2} d \right) \ge 1-T \cdot \exp \left(-\frac{d}{300} \right). $$ 
\end{theorem}

For example, for a symmetric 3-Gaussian model $\P_{0.01, \sqrt{3}, 0.2 \cdot \1_d}$, Theorem \ref{thm:3g:3gaussian} indicates that the sampled vector $\lT$ within $T \le \exp(d/300)$ iterations cannot be $\sqrt{2d}$ close to the center of any low-variance modes with high probability. To interpret Theorem \ref{thm:3g:3gaussian}, we first note that in a high-dimensional space $\sR^d$, the probability mass of a Gaussian distribution $\gN(\rvmu, \mI_d)$ concentrates inside a ball of radius $\sqrt{d}$ centered at $\rvmu$, i.e., $\nm{\lgv-\rvmu}^2 \le d$. On the other hand, the high probability bound $\du{\lT}^2 \ge \frac{\nu^2+1}{2} d$ in Theorem \ref{thm:3g:3gaussian} implies that $\lT$ is far from the center of both symmetric Gaussian modes, i.e., $\nm{\lT-\rvmu}^2 \ge \frac{\nu^2+1}{2} d \ge 2d$. This observation allows us to translate the bound on $\du{\lT}$ into a lower bound in other standard metrics such as total variation distance, as shown in the following Corollary \ref{coro:3g:tv}. 

\begin{corollary} \label{coro:3g:tv}
    Under the same assumptions as in Theorem \ref{thm:3g:3gaussian}, the distribution $\hP_T$ of the sampled vector $\lT$ by Langevin dynamics satisfies 
    \begin{equation*}
        \TV(\hP_T, \P) \ge 0.99 - w - \frac{T}{\exp(-d/300)}. 
    \end{equation*}
\end{corollary}

\subsection{Langevin Dynamics in Gaussian Mixture Models}

We further extend Theorem \ref{thm:3g:3gaussian} to a general Gaussian mixture setting. As illustrated in Figure \ref{fig:multimodal_distribution}.(b), we consider a Gaussian mixture with an in-between mode $\probz$ with high variance. To intuitively understand our theoretical results, we first note that the probability density $p(\rvz)$ of a Gaussian distribution $\gN(\rvmu, \nu^2 \mI_d)$ decays exponentially in terms of $\frac{\nm{\rvz-\rvmu}^2}{\nu^2}$. When a sample $\rvz$ is sufficiently far from one mode $P^{(i)}$, since $P^{(0)}$ has a higher variance, the probability density of $P^{(i)}$ is dominated by mode $P^{(0)}$ and the gradient information from $P^{(i)}$ will be masked by $P^{(0)}$. Hence, the dynamics can only visit $P^{(0)}$ unless the stochastic noise miraculously leads it to the region of another low-variance mode.We formalize this intuition in Theorem \ref{thm:ld:gaussian_mixture} and defer the proof to Appendix \ref{app:prf:ld_gaussian_mixture}.

\begin{theorem} \label{thm:ld:gaussian_mixture}
    Consider a data distribution $\P = \wz \gN(\0_d, \vz^2 \mI_d)+ \sum_{i\in[k]} \wi \gN(\mui,\vi^2 \mI_d)$ in dimension $d \ge 250$. For all low-variance modes $\probo,\cdots,\P^{(k)}$, we assume $\nm{\mui} \le 0.2 d$ and denote $\vmax := \max_{i \in [k]} \vi$. For in-between mode $\probz$, assume $\wz \ge 0.01$ and $\vz^2 \ge 3 \vmax^2$. We initialize the sample $\lz$ such that $\duk{\lz}^2 \ge \frac{3\vz^2+\vmax^2}{4} d$ and apply Langevin dynamics for $T$ iterations, then we have 
    \begin{equation*}
        \Pr \left( \duk{\lT} \ge \frac{\vz^2 + \vmax^2}{2} d \right) \ge 1 - T \cdot \exp \left( -\frac{d}{300} \right). 
    \end{equation*}
\end{theorem}

\subsection{Iteration Complexity of Annealed Langevin Dynamics}
Next, we generalize our theoretical results to annealed Langevin dynamics {\it with bounded noise levels} in Theorem \ref{thm:ald:gaussian_mixture}, under similar assumptions on the target distribution. The proof is deferred to Appendix \ref{app:prf:ald:gaussian-mixture}. Aligning with the analysis in \citep{song2020improved}, we show that bounded noise levels will have a limited impact on Langevin dynamics since they exhibit similar exponential complexity in high-dimensional distributions. On the other hand, as suggested by \citep{song2020improved}, annealed Langevin dynamics with a significantly larger initial noise level $\nlz$ could capture more modes (e.g., $\nlz = \gO(\sqrt{d})$), which is also confirmed by our numerical results in Section \ref{sec:numerical}. 

\begin{theorem} \label{thm:ald:gaussian_mixture}
    Consider a data distribution $\P = \wz \gN(\0_d, \vz^2 \mI_d)+ \sum_{i\in[k]} \wi \gN(\mui,\vi^2 \mI_d)$ in dimension $d \ge 250$. For all low-variance modes $\probo,\cdots,\P^{(k)}$, we assume $\nm{\mui} \le 0.05 d$ and denote $\vmax := \max_{i \in [k]} \vi$. For in-between mode $\probz$, assume $\wz \ge 0.01$ and $\vz^2 \ge 3 \vmax^2$. We initialize the sample $\lz$ such that $\duk{\lz}^2 \ge \frac{3\vz^2 + \vmax^2}{4} d + \nlz^2 d$ and apply annealed Langevin dynamics for $T$ steps with noise levels $\vmax \ge \nlz \ge \cdots \ge \nlT \ge 0$, then we have 
    \begin{equation*}
        \Pr \left( \duk{\lz}^2 \ge \frac{\vz^2 + \vmax^2}{2} d \right) \ge 1 - T \cdot \exp \left( -\frac{d}{1500} \right).
    \end{equation*}
\end{theorem}

Finally, in Appendix \ref{app:prf:sub-Gaussian}, we extend our theoretical results to sub-Gaussian mixtures $\P = \wz \probz + \sum_{i\in[k]} \wi \probi$, where $\probi$ is a sub-Gaussian distribution of mean $\mui$ with parameter $\vi^2$ satisfying that the score function of $\probi$ is Lipschitz. We show that if the sample $\lz$ is initialized far from the mean vectors, Langevin dynamics and annealed Langevin dynamics still exhibit similar exponential complexity to converge to low-variance sub-Gaussian modes in the target distribution.

\section{Chained Langevin Dynamics} \label{sec:chaindld}
To reduce the exponential complexity of Langevin dynamics, we propose Chained Langevin Dynamics (Chained-LD) in Algorithm \ref{alg:ChainedLD}. While Langevin dynamics apply gradient updates to all coordinates of the variable at every step, we decompose the variable into patches of constant size and sample each patch sequentially to alleviate the exponential dependency on the dimensionality. More precisely, we divide a vector $\lgv$ into $d/Q$ patches $\lgv^{(1)}, \cdots \lgv^{(d/Q)}$ of some constant size $Q$, and apply Langevin dynamics to sample each patch $\lgv^{(q)}$ (for $q \in [d/Q]$) from the conditional distribution $\P(\lgv^{(q)} \mid \lgv^{(1)}, \cdots \lgv^{(q-1)})$. Intuitively, vanilla Langevin dynamics needs to explore the entire space (of volume exponentially large in $d$) to find the missing modes, while Chained-LD could significantly lower the volume by dimensionality reduction. 

\begin{algorithm}[t]
\caption{Chained Langevin Dynamics (Chained-LD)}\label{alg:ChainedLD}
\begin{algorithmic}[1]
    \Require Patch size $Q$, dimension $d$, number of iterations $T$, noise levels $\set{\nlt}_{t \in [TQ/d]}$, conditional score function $\nabla \log \P_{\nlt}$, step size $\set{\dt}_{t \in [TQ/d]}$. 
    \State Initialize $\lz$, and divide $\lz$ into $d/Q$ patches $\xz^{(1)}, \cdots \xz^{(d/Q)}$ of equal size $Q$
    \vspace{1mm}
    \For{$q \gets 1$ to $d/Q$}
    \vspace{1mm}
            \For{$t \gets 1$ to $TQ/d$}
    \vspace{1mm}
            \State $\lt^{(q)} \gets \ltm^{(q)} + \frac{\dt}{2} \nabla \log \P_{\nlt} \left( \ltm^{(q)} \mid \lgv^{(1)}, \cdots, \lgv^{(q-1)} \right) + \sqrt{\dt} \epst$, where $\epst \sim \gN(\0_Q, \mI_Q)$
        \EndFor
    \vspace{1mm}
        \State $\lgv^{(q)} \gets \lgv_{TQ/d}^{(q)}$
    \vspace{1mm}
    \EndFor
    \vspace{1mm}
    \State \Return $\lgv$
\end{algorithmic}
\end{algorithm}

We can also apply annealed Langevin dynamics \citep{song2019generative} to facilitate the sampling of each patch, by perturbing it with a series of noise levels $\set{\nlt}_{t \in [TQ/d]}$. Specifically, we refer {\it chained vanilla Langevin dynamics (Chained-VLD)} to Algorithm \ref{alg:ChainedLD} without noise injection (i.e., $\nlt = 0$ for all $t \in [TQ/d]$), and {\it chained annealed Langevin dynamics (Chained-ALD)} otherwise. Ideally, if a sampler perfectly generates every patch, combining all patches gives a vector from the original distribution due to the chain rule 
$$\P(\lgv) = \prod_{q \in [d/Q]} \P(\lgv^{(q)} \mid \lgv^{(1)}, \cdots \lgv^{(q-1)}).$$

In Theorem \ref{thm:chained}, we prove that Chained-LD can provably converge to the target distribution within $\varepsilon$ total variation distance, in a polynomial number of iterations. Similar to \citep{cheng2018sharp, ma2019sampling}, we assume that the log conditional PDF of every patch $\log \P(\lgv^{(q)} | \lgv^{(1)}, \cdots, \lgv^{(q-1)})$ is $\LQ$-smooth and $\mQ$-strongly concave for $\lgv^{(q)} > \RQ$. The details of Assumption \ref{asmpt:chained} and the proof of Theorem \ref{thm:chained} is deferred to Appendix \ref{app:prf:chainedld}.

\begin{theorem} \label{thm:chained}
    Consider a data distribution $\P$ satisfying Assumption \ref{asmpt:chained}. We initialize $\lz \sim \gN(\0_d, \frac{1}{\LQ}\mI_d)$ and apply chained Langevin dynamics in Algorithm \ref{alg:ChainedLD} with constant patch size $Q$, noise level $\nlt = 0$, and step size $\dt = \frac{\mQ\varepsilon^2 Q}{64\LQ^2 d^2} \exp(-16 \LQ \RQ^2)$. Then, for 
    $$T = \frac{128 \LQ^2 d^3}{\mQ^2 Q^2 \varepsilon^2} \exp(32\LQ \RQ^2)  \log \left( \frac{d^3}{\varepsilon^2 Q^2} \right),$$ 
    the output distribution $\hP(\lgv)$ after $T$ iterations satisfies $\TV(\hP(\lgv), \P(\lgv)) \le \varepsilon$ for any constant $\varepsilon > 0$. 
\end{theorem}

We highlight that due to dimension reduction, in general, the parameters $\LQ, \mQ, \RQ$ are constants that do not grow with dimension $d$. To give a concrete example, we consider a symmetric 3-Gaussian model 
\begin{equation*}
    \P_{w,\nu,\1_d} = w \gN(\0_d, \nu^2 \mI_d) + \frac{1-w}{2} \gN(\1_d, \mI_d) + \frac{1-w}{2} \gN(-\1_d, \mI_d). 
\end{equation*}
Then, for every patch $q \in [d/Q]$, the conditional distribution is given by 
\begin{equation*}
    \P \left( \lgv^{(q)} | \lgv^{(1)}, \cdots, \lgv^{(q-1)} \right) = w \gN(\0_Q, \nu^2 \mI_Q) + \frac{1-w}{2} \gN(\1_Q, \mI_d) + \frac{1-w}{2} \gN(-\1_Q, \mI_Q), 
\end{equation*}
which is independent from the dimension $d$ of the whole vector $\lgv$. Therefore, the parameters $\LQ, \mQ, \RQ$ depend only on the patch size $Q$, which is set as a constant. In contrast, without dimension reduction, $-\log \P_{w,\nu,\1_d}(\lgv)$ is non-convex for $\lgv = \1_d$. Therefore, under the assumption that the distribution of the whole vector $-\log \P_{w,\nu,\1_d}(\lgv)$ is strongly-convex for $\nm{\lgv} > R$ where $R > \sqrt{d}$, the upper bound on the iteration complexity of Langevin dynamics obtained by \citep{cheng2018sharp} and \citep{ma2019sampling} scales as $\gO(\exp(cLR^2 \text{poly}(d,1/\varepsilon))) \ge \gO(\exp(cLd))$, which is exponential in dimension $d$.

\section{Numerical Results} \label{sec:numerical}
\begin{figure}[t]
    \centering
    \begin{tabular}{c}
       \includegraphics[width=0.9\textwidth]{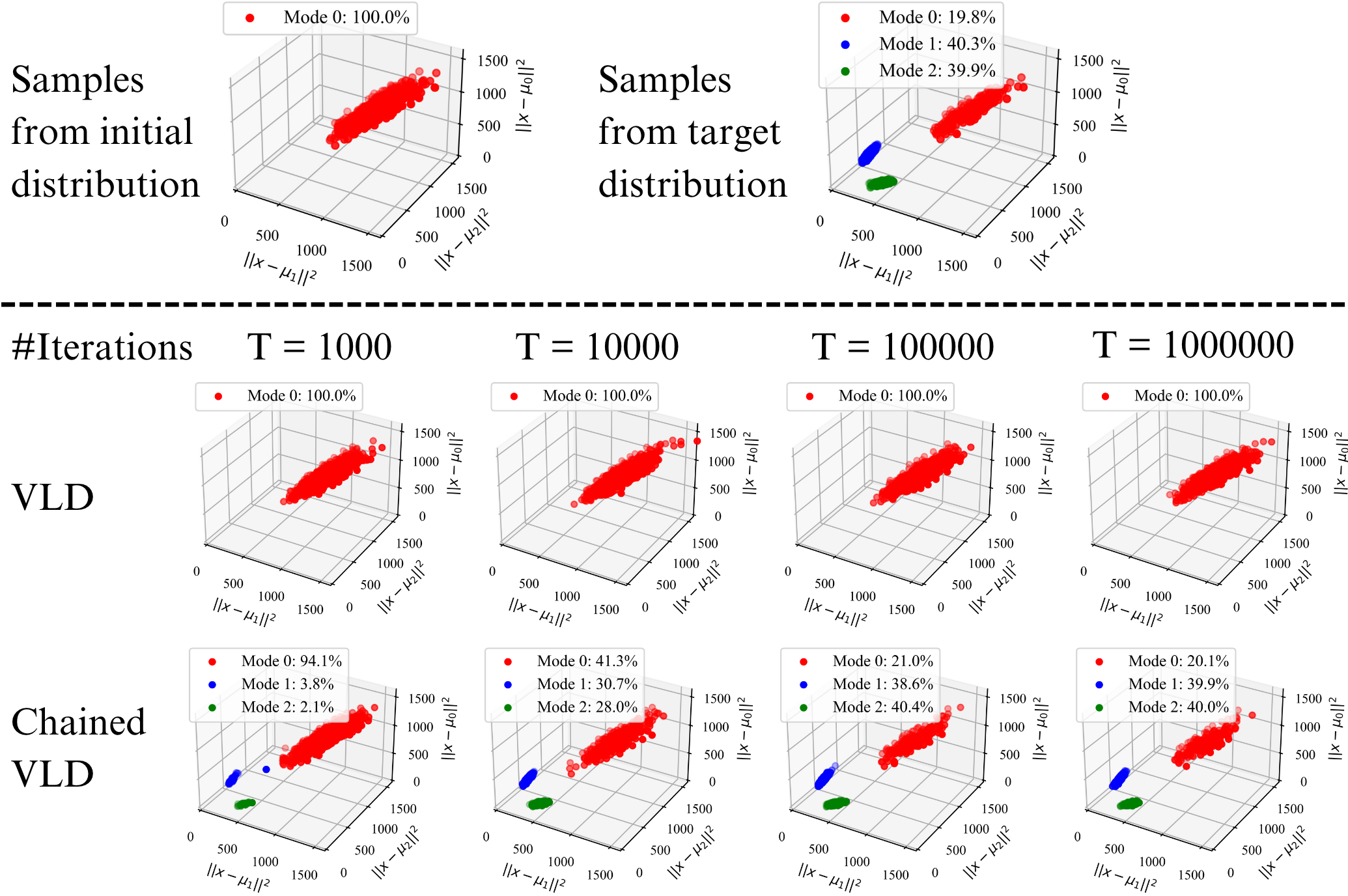}
      \end{tabular}
    \caption{Samples from a mixture of three Gaussian modes generated by vanilla Langevin dynamics (VLD) and chained vanilla Langevin dynamics (Chained-VLD) with patch size $Q=10$. Three axes are $\ell_2$ distance from samples to the mean of the three modes. The samples are initialized in mode 0.} \label{fig:syn:mode0}
\end{figure}
\vspace{-2mm}

In this section, we empirically validated our theoretical findings of vanilla and chained Langevin dynamics. We performed numerical experiments on synthetic Gaussian mixture models and real image datasets including MNIST \citep{lecun1998mnist} and Fashion-MNIST \citep{xiao2017fashion}. Details on the experiment setup are deferred to Appendix \ref{app:exp}.

\textbf{Synthetic Gaussian mixture model:} We consider the data distribution $\P$ as symmetric 3-Gaussian model with $w = 0.2$, $\nu = \sqrt{3}$, and $\rvmu = \1_d$, i.e., 
\begin{equation} \label{eq:def:syn-gaussian}
    \P = 0.2 \P^{(0)} + 0.4 \P^{(1)} + 0.4 \P^{(2)} = 0.2 \gN(\0_d, 3\mI_d) + 0.4 \gN(\1_d, \mI_d) + 0.4 \gN(-\1_d, \mI_d). 
\end{equation}
In the synthetic experiments, we give the samplers access to the true score function calculated from the target distribution. As shown in Figure \ref{fig:syn:mode0}, vanilla Langevin dynamics (VLD) cannot find mode 1 or 2 within $10^6$ iterations if the sample is initialized in mode 0, while chained vanilla Langevin dynamics (Chained-VLD) with patch size $Q=10$ can find the other two modes in 1000 steps and correctly recover their frequencies as gradually increasing the number of iterations. When the sample is initialized in mode 1, as shown in Figure \ref{fig:syn:mode1} in Appendix \ref{app:exp:synthetic}, VLD is also likely to be trapped by the high-variance mode 0 and cannot find mode 2, while Chained-VLD is capable of finding all modes. Additional experiments on samples initialized in mode 2 are presented in Appendix \ref{app:exp:synthetic}, which also verify the convergence hardness of vanilla Langevin dynamics. We also investigated the effect of different choices of patch size $Q$ on the performance of Chained-LD. As shown in Figures \ref{fig:syn:mode0_patch}, \ref{fig:syn:mode1_patch}, and \ref{fig:syn:mode2_patch} in Appendix \ref{app:exp:synthetic}, the convergence of Chained-LD are insensitive to moderate values of constant $Q \in \{1,4,10\}$; for large $Q = 20$, it takes more steps to find the other modes; while for overly large $Q=50$, Chained-LD has convergence hardness similar to LD.

\textbf{Applications of Chained-LD in generative modeling:} We also test the application of Chained-LD as a sampling methodology in generative modeling. We consider a mixture distribution of two modes by using the original images from MNIST/Fashion-MNIST training dataset (black background and white digits/objects) as the first mode and constructing the second mode by i.i.d. randomly flipping an image (white background and black digits/objects) with probability 0.5. Following from \citep{song2019generative}, we train an estimator to approximate the score function from training samples, and apply Chained-LD using the estimated score function. More implementation details are deferred to Appendix \ref{app:score_estimator}.

We numerically validate our theoretical findings of annealed Langevin dynamics (ALD) and Chained-ALD. As shown in Figures \ref{fig:mnist_flip_ald}, ALD with bounded noise levels (i.e., the maximum noise $\sigma_{\max}=1$) tends to sample from the same mode as initialization, aligning with our theoretical analysis in Theorem \ref{thm:ald:gaussian_mixture}. Then, if we apply larger noise levels (i.e., the maximum noise $\sigma_{\max}=50$ as suggested by Technique 1 in \citep{song2020improved}), ALD could generate samples from both modes. On the other hand, Chained-ALD, even with bounded noise levels  (i.e., $\sigma_{\max}=1$), is capable of finding both modes. Further experiments are deferred to Appendix \ref{app:exp:image}.

\begin{figure}
    \centering
    \begin{tabular}{c}
       \includegraphics[width=0.7\textwidth]{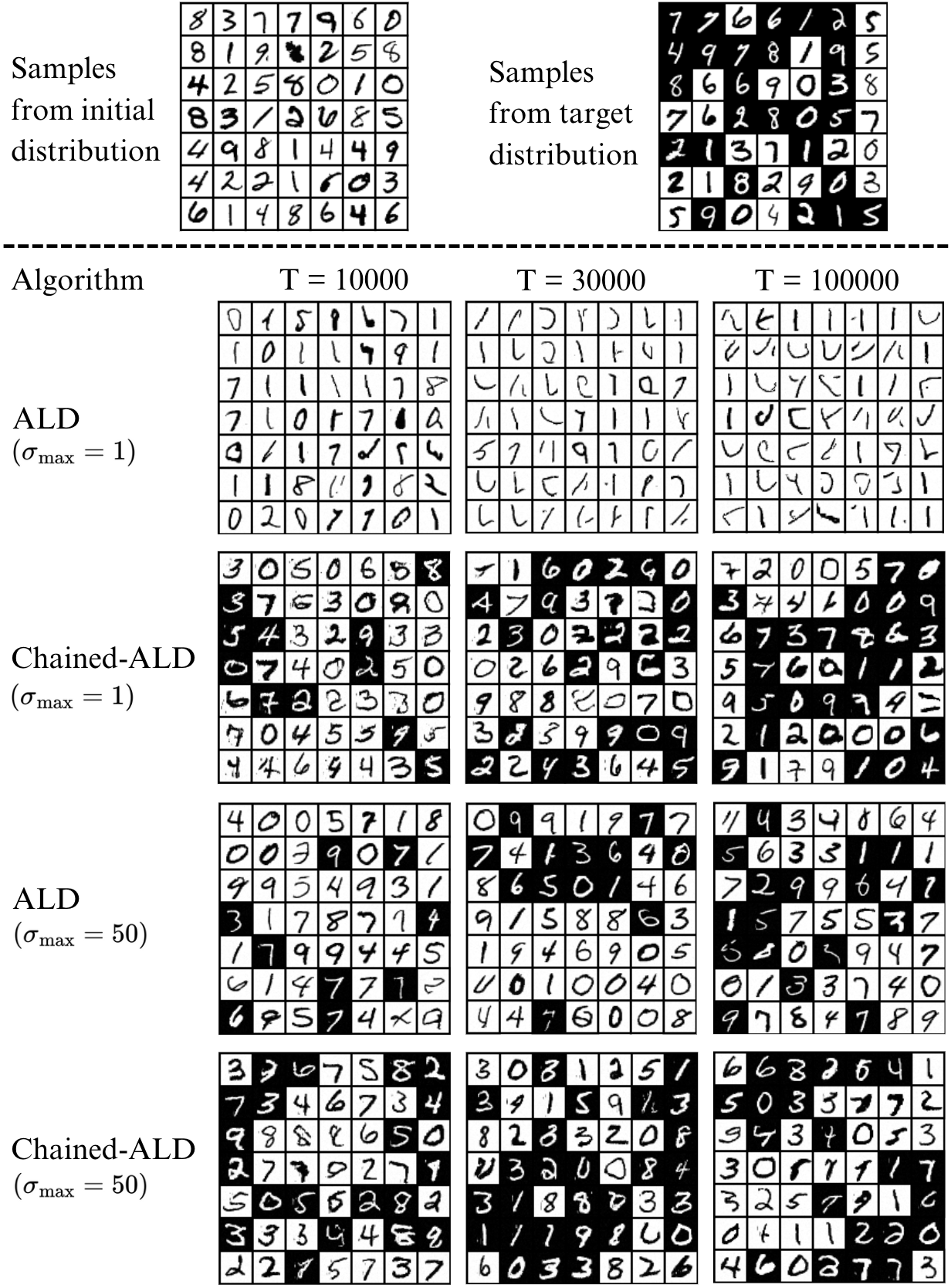}
      \end{tabular}
    \caption{Samples from a mixture distribution of the original and flipped images from the MNIST dataset generated by annealed Langevin dynamics (ALD) and chained annealed Langevin dynamics (Chained-ALD) with patch size $Q=14$ for different numbers of iterations. The maximum noise level $\sigma_{\max}$ is set to be 1 or 50. The samples are initialized as flipped images from MNIST. } \label{fig:mnist_flip_ald}
\end{figure}

\section{Conclusion} \label{sec:conclusion}
In this work, we theoretically and numerically studied the hardness of Langevin dynamics sampling methods under a multi-modal distribution. We characterized Gaussian and sub-Gaussian mixture models under which Langevin dynamics are unlikely to find all the components within a sub-exponential number of iterations. To reduce the exponential iteration complexity of Langevin dynamics, we proposed Chained Langevin Dynamics (Chained-LD), as a complementary solution to Annealed-LD in \citep{song2019generative} and analyzed its convergence behavior. Further investigation on the applications of Chained-LD in generative models will be an interesting topic for future exploration. Another future direction could be to study the convergence of Chained-LD under an imperfect score estimation.

\bibliographystyle{unsrt}
\bibliography{references}

\begin{thebibliography}{10}

\bibitem{welling2011bayesian}
Max Welling and Yee~W Teh.
\newblock Bayesian learning via stochastic gradient langevin dynamics.
\newblock In {\em Proceedings of the 28th international conference on machine learning (ICML-11)}, pages 681--688. Citeseer, 2011.

\bibitem{raginsky2017non}
Maxim Raginsky, Alexander Rakhlin, and Matus Telgarsky.
\newblock Non-convex learning via stochastic gradient langevin dynamics: a nonasymptotic analysis.
\newblock In {\em Conference on Learning Theory}, pages 1674--1703. PMLR, 2017.

\bibitem{xu2018global}
Pan Xu, Jinghui Chen, Difan Zou, and Quanquan Gu.
\newblock Global convergence of langevin dynamics based algorithms for nonconvex optimization.
\newblock {\em Advances in Neural Information Processing Systems}, 31, 2018.

\bibitem{paquet2015molecular}
Eric Paquet and Herna~L Viktor.
\newblock Molecular dynamics, monte carlo simulations, and langevin dynamics: a computational review.
\newblock {\em BioMed research international}, 2015(1):183918, 2015.

\bibitem{gottwald2015parametrizing}
Fabian Gottwald, Sven Karsten, Sergei~D Ivanov, and Oliver K{\"u}hn.
\newblock Parametrizing linear generalized langevin dynamics from explicit molecular dynamics simulations.
\newblock {\em The Journal of chemical physics}, 142(24), 2015.

\bibitem{song2019generative}
Yang Song and Stefano Ermon.
\newblock Generative modeling by estimating gradients of the data distribution.
\newblock {\em Advances in neural information processing systems}, 32, 2019.

\bibitem{song2020score}
Yang Song, Jascha Sohl-Dickstein, Diederik~P Kingma, Abhishek Kumar, Stefano Ermon, and Ben Poole.
\newblock Score-based generative modeling through stochastic differential equations.
\newblock In {\em International Conference on Learning Representations}, 2020.

\bibitem{kamalaruban2020robust}
Parameswaran Kamalaruban, Yu-Ting Huang, Ya-Ping Hsieh, Paul Rolland, Cheng Shi, and Volkan Cevher.
\newblock Robust reinforcement learning via adversarial training with langevin dynamics.
\newblock {\em Advances in Neural Information Processing Systems}, 33:8127--8138, 2020.

\bibitem{srinivasan2021robustifying}
Vignesh Srinivasan, Csaba Rohrer, Arturo Marban, Klaus-Robert M{\"u}ller, Wojciech Samek, and Shinichi Nakajima.
\newblock Robustifying models against adversarial attacks by langevin dynamics.
\newblock {\em Neural Networks}, 137:1--17, 2021.

\bibitem{reuss2023goal}
Moritz Reuss, Maximilian Li, Xiaogang Jia, and Rudolf Lioutikov.
\newblock Goal-conditioned imitation learning using score-based diffusion policies.
\newblock {\em arXiv preprint arXiv:2304.02532}, 2023.

\bibitem{pearce2023imitating}
Tim Pearce, Tabish Rashid, Anssi Kanervisto, Dave Bignell, Mingfei Sun, Raluca Georgescu, Sergio~Valcarcel Macua, Shan~Zheng Tan, Ida Momennejad, Katja Hofmann, and Sam Devlin.
\newblock Imitating human behaviour with diffusion models.
\newblock In {\em The Eleventh International Conference on Learning Representations}, 2023.

\bibitem{durmus2017nonasymptotic}
Alain Durmus and {\'E}ric Moulines.
\newblock Nonasymptotic convergence analysis for the unadjusted langevin algorithm.
\newblock {\em The Annals of Applied Probability}, 27(3):1551--1587, 2017.

\bibitem{dalalyan2017theoretical}
Arnak~S Dalalyan.
\newblock Theoretical guarantees for approximate sampling from smooth and log-concave densities.
\newblock {\em Journal of the Royal Statistical Society Series B: Statistical Methodology}, 79(3):651--676, 2017.

\bibitem{cheng2018underdamped}
Xiang Cheng, Niladri~S Chatterji, Peter~L Bartlett, and Michael~I Jordan.
\newblock Underdamped langevin mcmc: A non-asymptotic analysis.
\newblock In {\em Conference on learning theory}, pages 300--323. PMLR, 2018.

\bibitem{cheng2018convergence}
Xiang Cheng and Peter Bartlett.
\newblock Convergence of langevin mcmc in kl-divergence.
\newblock In {\em Algorithmic learning theory}, pages 186--211. PMLR, 2018.

\bibitem{bhattacharya1978criteria}
RN~Bhattacharya.
\newblock Criteria for recurrence and existence of invariant measures for multidimensional diffusions.
\newblock {\em The Annals of Probability}, pages 541--553, 1978.

\bibitem{roberts1996exponential}
Gareth~O Roberts and Richard~L Tweedie.
\newblock Exponential convergence of langevin distributions and their discrete approximations.
\newblock {\em Bernoulli}, pages 341--363, 1996.

\bibitem{bakry1983diffusions}
D~Bakry and M~{\'E}mery.
\newblock Diffusions hypercontractives.
\newblock {\em Seminaire de Probabilites XIX}, page 177, 1983.

\bibitem{bakry2008simple}
Dominique Bakry, Franck Barthe, Patrick Cattiaux, and Arnaud Guillin.
\newblock A simple proof of the poincar{\'e} inequality for a large class of probability measures.
\newblock {\em Electronic Communications in Probability [electronic only]}, 13:60--66, 2008.

\bibitem{lee2018beyond}
Holden Lee, Andrej Risteski, and Rong Ge.
\newblock Beyond log-concavity: Provable guarantees for sampling multi-modal distributions using simulated tempering langevin monte carlo.
\newblock {\em Advances in neural information processing systems}, 31, 2018.

\bibitem{cheng2018sharp}
Xiang Cheng, Niladri~S Chatterji, Yasin Abbasi-Yadkori, Peter~L Bartlett, and Michael~I Jordan.
\newblock Sharp convergence rates for langevin dynamics in the nonconvex setting.
\newblock {\em arXiv preprint arXiv:1805.01648}, 2018.

\bibitem{ma2019sampling}
Yi-An Ma, Yuansi Chen, Chi Jin, Nicolas Flammarion, and Michael~I Jordan.
\newblock Sampling can be faster than optimization.
\newblock {\em Proceedings of the National Academy of Sciences}, 116(42):20881--20885, 2019.

\bibitem{bovier2002metastability}
Anton Bovier, Michael Eckhoff, V{\'e}ronique Gayrard, and Markus Klein.
\newblock Metastability and low lying spectra in reversible markov chains.
\newblock {\em Communications in mathematical physics}, 228:219--255, 2002.

\bibitem{bovier2004metastability}
Anton Bovier, Michael Eckhoff, V{\'e}ronique Gayrard, and Markus Klein.
\newblock Metastability in reversible diffusion processes i: Sharp asymptotics for capacities and exit times.
\newblock {\em Journal of the European Mathematical Society}, 6(4):399--424, 2004.

\bibitem{gayrard2005metastability}
V{\'e}ronique Gayrard, Anton Bovier, and Markus Klein.
\newblock Metastability in reversible diffusion processes ii: Precise asymptotics for small eigenvalues.
\newblock {\em Journal of the European Mathematical Society}, 7(1):69--99, 2005.

\bibitem{song2020improved}
Yang Song and Stefano Ermon.
\newblock Improved techniques for training score-based generative models.
\newblock {\em Advances in neural information processing systems}, 33:12438--12448, 2020.

\bibitem{ho2020denoising}
Jonathan Ho, Ajay Jain, and Pieter Abbeel.
\newblock Denoising diffusion probabilistic models.
\newblock {\em Advances in neural information processing systems}, 33:6840--6851, 2020.

\bibitem{lecun1998mnist}
Yann LeCun.
\newblock The mnist database of handwritten digits.
\newblock {\em http://yann. lecun. com/exdb/mnist/}, 1998.

\bibitem{xiao2017fashion}
Han Xiao, Kashif Rasul, and Roland Vollgraf.
\newblock Fashion-mnist: a novel image dataset for benchmarking machine learning algorithms.
\newblock {\em arXiv preprint arXiv:1708.07747}, 2017.

\bibitem{laurent2000adaptive}
Beatrice Laurent and Pascal Massart.
\newblock Adaptive estimation of a quadratic functional by model selection.
\newblock {\em Annals of statistics}, pages 1302--1338, 2000.

\bibitem{peters1986convex}
H.J.M. Peters and P.P. Wakker.
\newblock Convex functions on non-convex domains.
\newblock {\em Economics Letters}, 22(2):251--255, 1986.

\bibitem{min2012extension}
Min Yan.
\newblock Extension of convex function.
\newblock {\em Journal of Convex Analysis}, 21, 07 2012.

\bibitem{vempala2019rapid}
Santosh Vempala and Andre Wibisono.
\newblock Rapid convergence of the unadjusted langevin algorithm: Isoperimetry suffices.
\newblock {\em Advances in neural information processing systems}, 32, 2019.

\bibitem{vincent2011connection}
Pascal Vincent.
\newblock A connection between score matching and denoising autoencoders.
\newblock {\em Neural computation}, 23(7):1661--1674, 2011.

\bibitem{sak2014long}
Ha{\c{s}}im Sak, Andrew Senior, and Fran{\c{c}}oise Beaufays.
\newblock Long short-term memory based recurrent neural network architectures for large vocabulary speech recognition.
\newblock {\em arXiv preprint arXiv:1402.1128}, 2014.

\end{thebibliography}

\clearpage
\appendix
\section{Iteration Complexity of Langevin Dynamics in Gaussian Mixture Models} \label{app:prf}

We begin by introducing some well-established lemmas used in our proof. We use the following lemma on the tail bound for multivariate Gaussian random variables. 
\begin{lemma}[Lemma 1, \citep{laurent2000adaptive}] \label{lem:gaussian_tail_bound}
    Suppose that a random variable $\rvz \sim \gN(\0_d, \mI_d)$. Then for any $\lambda > 0$, 
    \begin{align*}
        \Pr \left( \nm{\rvz}^2 \ge d + 2\sqrt{d \lambda} + 2 \lambda \right) &\le \exp(-\lambda), \\ 
        \Pr \left( \nm{\rvz}^2 \le d - 2\sqrt{d \lambda} \right) &\le \exp(-\lambda).
    \end{align*}
\end{lemma}

We also use a tail bound for one-dimensional Gaussian random variables and provide the proof here for completeness. 
\begin{lemma}\label{lem:one_dim_gaussian_tail_bound}
    Suppose a random variable $Z \sim \gN(0,1)$. Then for any $t > 0$,
    \begin{align*}
        \Pr (Z \ge t) = \Pr (Z \le -t) \le \frac{\exp(-t^2/2)}{\sqrt{2\pi} t}.
    \end{align*}
\end{lemma}
\begin{proof}[Proof of Lemma \ref{lem:one_dim_gaussian_tail_bound}]
    Since $\frac{z}{t} \ge 1$ for all $z \in [t,\infty)$, we have 
    \begin{align*}
        \Pr (Z \ge t) = \frac{1}{\sqrt{2\pi}} \int_t^\infty \exp \left( -\frac{z^2}{2} \right) \dif z \le \frac{1}{\sqrt{2\pi}} \int_t^\infty \frac{z}{t} \exp \left( -\frac{z^2}{2} \right) \dif z = \frac{\exp(-t^2/2)}{\sqrt{2\pi} t}. 
    \end{align*}
    Since the Gaussian distribution is symmetric, we have $\Pr (Z \ge t) = \Pr (Z \le -t)$. Hence we obtain the desired bound. 
\end{proof}

\subsection{Proof of Theorem \ref{thm:3g:3gaussian}}
\label{app:prf:3g}

\begin{proof}[Proof of Theorem \ref{thm:3g:3gaussian}]
Denote $\R = \frac{\rvmu}{\nm{\rvmu}} \in \sR^{d \times 1}$, and denote $\N \in \sR^{d \times (d-1)}$ an orthonormal basis of the null space of $\rvmu$. Now consider decomposing the sample $\lt$ by \footnote{To be consistent with the notations in other parts of this work, we abuse the notations of $\R$ and $\rt$ in the proof of Theorem \ref{thm:3g:3gaussian}, i.e., $\R$ is a vector instead of a matrix, and $\rt$ is a scalar instead of a vector. }
\begin{equation*} 
    \rt := \R^T \lt \text{, and } \nt := \N^T \lt, 
\end{equation*}
where $\rt \in \sR$, $\nt \in \sR^{d-1}$. Then we have 
\begin{equation*} 
    \lt = \R \rt + \N \nt. 
\end{equation*}
Similarly, we decompose the noise $\epst$ into 
\begin{equation*} 
    \ert := \R^T \epst \text{, and } \ent := \N^T \epst, 
\end{equation*}
where $\ert \in \sR$, $\ent \in \sR^{d-1}$. Then we have 
\begin{equation*} 
    \epst = \R \ert + \N \ent. 
\end{equation*}
Since a linear combination of a Gaussian random variable still follows Gaussian distribution, by $\epst \sim \gN(\0_d, \mI_d)$, $\R^T \R = 1$, and $\N^T \N = \mI_{d-1}$ we obtain 
\begin{equation*} 
    \ert \sim \gN(0,1) \text{, and } \ent \sim \gN(\0_{d-1}, \mI_{d-1}). 
\end{equation*}
By the definition of Langevin dynamics in \eqref{eq:ld:ld_def}, $\nt$ follow from the update rule: 
\begin{align}
    \nt &= \ntm + \frac{\dt}{2} \N^T \nabla_\lgv \log \P (\ltm) + \sqrt{\dt} \ent \label{eq:3g:nt_update} .
\end{align}
It is worth noting that by Definition \ref{def:dist}, we have 
\begin{equation} \label{eq:3g:conn-nt-du}
    \du{\lt} = \nm{\lt - \frac{\rvmu^T \lt}{\nm{\rvmu}^2} \rvmu} = \nm{\lt - \R \R^T \lt} = \nm{\N \N^T \lt} = \nm{\nt}. 
\end{equation}

To establish a lower bound on $\nm{\nt}$, we consider different cases of the step size $\dt$. Intuitively, when $\dt$ is large enough, $\nt$ will be too noisy due to the introduction of random noise $\sqrt{\dt} \ent$ in \eqref{eq:3g:nt_update}. While for small $\dt$, the update of $\nt$ is bounded and thus we can iteratively analyze $\nt$. We first handle the case of large $\dt$ in the following lemma.

\begin{lemma} \label{lem:3g:large_dt}
    If $\dt > \nu^2$, with probability at least $1-\exp(-0.04d)$, for $\nt$ satisfying \eqref{eq:3g:nt_update}, we have $\nm{\nt}^2 \ge \frac{3\nu^2+1}{4} d$ regardless of the previous state $\ltm$. 
\end{lemma}
\begin{proof}[Proof of Lemma \ref{lem:3g:large_dt}]
    Denote $\rvv := \ntm + \frac{\dt}{2} \N^T \nabla_\lgv \log \P (\ltm)$ for simplicity. Note that $\rvv$ is fixed for any given $\ltm$. We decompose $\ent$ into a vector aligning with $\rvv$ and another vector orthogonal to $\rvv$. Consider an orthonormal matrix $\rmM \in \sR^{d \times (d-1)}$ such that $\rmM^T \rvv = \0_{d-1}$ and $\rmM^T \rmM = \mI_{d-1}$. By denoting $\rvu := \ent - \rmM \rmM^T \ent$ we have $\rmM^T \rvu = \0_{d-1}$, thus we obtain
    \begin{align*}
        \nm{\nt}^2 &= \nm{\rvv + \sqrt{\dt} \ent}^2 \\ 
        &= \nm{\rvv + \sqrt{\dt} \rvu + \sqrt{\dt} \rmM \rmM^T \ent}^2 \\ 
        &= \nm{\rvv + \sqrt{\dt} \rvu}^2 + \nm{\sqrt{\dt} \rmM \rmM^T \ent}^2 \\ 
        &\ge \nm{\sqrt{\dt} \rmM \rmM^T \ent}^2 \\ 
        &\ge \nu^2 \nm{\rmM^T \ent}^2. 
    \end{align*}
    Since $\ent \sim \gN(\0_d, \mI_d)$ and $\rmM^T \rmM = \mI_{d-1}$, we obtain $\rmM^T \ent \sim \gN(\0_{d-1}, \mI_{d-1})$. Therefore, by Lemma \ref{lem:gaussian_tail_bound} we can bound 
    \begin{align*}
        \Pr \left( \nm{\nt}^2 \le \frac{3\nu^2 + 1}{4} d \right) 
        &\le \Pr \left( \nm{\rmM^T \ent}^2 \le \frac{3\nu^2 + 1}{4\nu^2} d \right) \\ 
        &= \Pr \left( \nm{\rmM^T \ent}^2 \le d - 2 \sqrt{d \cdot \left( \frac{\nu^2-1}{8\nu^2} \right)^2 d} \right) \\ 
        &\le \Pr \left( \nm{\rmM^T \ent}^2 \le (d-1) - 2 \sqrt{(d-1) \left( \frac{\nu^2-1}{8\nu^2} \right)^2 \frac{d}{2}} \right) \\ 
        &\le \exp \left( - \left( \frac{\nu^2-1}{8\nu^2} \right)^2 \frac{d}{2} \right) \le \exp \left( -\frac{d}{24} \right).
    \end{align*}
    Hence we complete the proof of Lemma \ref{lem:3g:large_dt}. 
\end{proof}

We then consider the case when $\dt \le \nu^2$. We first show that when $\nm{\nv}^2 \ge \frac{\nu^2+1}{2} d$, $\probo(\lgv)$ and $\probt(\lgv)$ are exponentially smaller than $\probz(\lgv)$ in the following lemma. 

\begin{lemma} \label{lem:3g:ratio_probi_probz}
    Given that $\nm{\nv}^2 \ge \frac{\nu^2+1}{2} d$ and $\nm{\rvmu}^2 \le 0.2d$, we have both $\frac{\probo(\lgv)}{\probz(\lgv)} \le \exp(-0.06d)$ and $\frac{\probt(\lgv)}{\probz(\lgv)} \le \exp(-0.06d)$. 
\end{lemma}
\begin{proof}[Proof of Lemma \ref{lem:3g:ratio_probi_probz}]
    By the density function of Gaussian distribution, we have
    \begin{align*}
        \frac{\probo(\lgv)}{\probz(\lgv)} &= \frac{(2\pi)^{-d/2} \exp \left( -\frac{1}{2} \nm{\lgv - \rvmu}^2 \right)}{(2\pi \nu^2)^{-d/2} \exp \left( -\frac{1}{2\nu^2} \nm{\lgv}^2 \right)} \\ 
        &= \nu^d \exp \left( \frac{1}{2\nu^2} \nm{\lgv}^2 - \frac{1}{2} \nm{\lgv - \rvmu}^2 \right) \\ 
        &= \nu^d \exp \left( \left( \frac{1}{2\nu^2}-\frac{1}{2} \right) \nm{\N\nv}^2 + \left( \frac{\nm{\R\rv}^2}{2\nu^2} - \frac{\nm{\R\rv-\rvmu}^2}{2} \right) \right) \\ 
        &= \nu^d \exp \left( \left( \frac{1}{2\nu^2}-\frac{1}{2} \right) \nm{\nv}^2 + \left( \frac{\nm{\rv}^2}{2\nu^2} - \frac{\nm{\rv-\R^T\rvmu}^2}{2} \right) \right), 
    \end{align*}
    Since $\nu^2 \ge 3$, the quadratic term $\frac{\nm{\rv}^2}{2\nu^2} - \frac{\nm{\rv - \R^T\rvmu}^2}{2}$ is maximized at $\rv = \frac{\nu^2 \R^T \rvmu}{\nu^2-1}$. Therefore, 
    \begin{align*}
        \frac{\nm{\rv}^2}{2\nu^2} - \frac{\nm{\rv - \R^T\rvmu}^2}{2} &\le \frac{\nu^4\nm{\R^T \rvmu}^2}{2\nu^2(\nu^2-1)^2} - \frac{1}{2} \left( \frac{\nu^2}{\nu^2-1}-1 \right)^2 \nm{\R^T \rvmu}^2 = \frac{\nm{\rvmu}^2}{2(\nu^2-1)}. 
    \end{align*}
    Hence, for $\nm{\nv}^2 \ge \frac{\nu^2+1}{2} d$, by $\nu^2 \ge 3$ and $\nm{\rvmu}^2 \le 0.2 d$ we have 
    \begin{align*}
        \frac{\probo(\lgv)}{\probz(\lgv)} &= \nu^d \exp \left( \left( \frac{1}{2\nu^2}-\frac{1}{2} \right) \nm{\nv}^2 + \left( \frac{\nm{\rv}^2}{2\nu^2} - \frac{\nm{\rv-\R^T\rvmu}^2}{2} \right) \right) \\ 
        &\le \exp \left( d \log \nu - \frac{\nu^4-1}{4\nu^2} d + \frac{1}{10(\nu^2-1)} d \right) \le \exp(-0.06 d). 
    \end{align*}
    We can similarly obtain the same result for $\frac{\probt(\lgv)}{\probz(\lgv)}$. Therefore we finish the proof of Lemma \ref{lem:3g:ratio_probi_probz}. 
\end{proof}

Lemma \ref{lem:3g:ratio_probi_probz} implies that when $\nm{\nv}$ is large, the Gaussian mode $\probz$ dominates other modes $\probo$ and $\probt$. To bound $\nm{\nt}$, we first consider a simpler case that $\nm{\ntm}$ is large. Intuitively, the following lemma proves that when the previous state $\ntm$ is far from the low-variance modes, a single step of Langevin dynamics with a bounded step size is not enough to find the modes. 
\begin{lemma} \label{lem:3g:large_ntm_norm}
    Suppose $\dt \le \nu^2$ and $\nm{\ntm}^2 > 36 \nu^2 d$, then for $\nt$ following from \eqref{eq:3g:nt_update}, we have $\nm{\nt}^2 \ge \nu^2 d$ with probability at least $1-\exp(-0.02d)$. 
\end{lemma}
\begin{proof}[Proof of Lemma \ref{lem:3g:large_ntm_norm}]
    From the recursion of $\nt$ in \eqref{eq:3g:nt_update} we have
    \begin{align}
        \nt &= \ntm + \frac{\dt}{2} \N^T \nabla_\lgv \log \P (\ltm) + \sqrt{\dt} \ent \notag \\ 
        &= \ntm - \frac{\dt}{2} \sum_{i=0}^2 \frac{\wi\probi(\ltm)}{\P(\ltm)} \cdot \frac{\N^T (\ltm-\mui)}{\vi^2} + \sqrt{\dt} \ent \notag \\ 
        &= \left( 1 - \frac{\dt}{2} \sum_{i=0}^2 \frac{\wi\probi(\ltm)}{\P(\ltm)} \cdot \frac{1}{\vi^2} \right) \ntm + \sqrt{\dt} \ent. \label{eq:3g:large_ntm_norm:decompose_nt}
    \end{align}
    By Lemma \ref{lem:3g:ratio_probi_probz}, we have $\frac{\probo(\ljm)}{\probz(\ljm)} \le \exp(-0.06d)$ and $\frac{\probt(\ljm)}{\probz(\ljm)} \le \exp(-0.06d)$, therefore 
    \begin{equation}
        1 - \frac{\dt}{2} \sum_{i=0}^2 \frac{\wi\probi(\ltm)}{\P(\ltm)} \cdot \frac{1}{\vi^2} \ge 1 - \frac{\dt}{2} \cdot \frac{1}{\nu^2} - \frac{(1-w)\dt}{2w} \exp(-0.06d) > \frac{1}{3}. \label{eq:3g:large_ntm_norm:bound_param}
    \end{equation}
    On the other hand, from $\ent \sim \gN(\0_{d-1}, \mI_{d-1})$ we know $\frac{\langle \ntm , \ent \rangle}{\nm{\ntm}} \sim \gN(0, 1)$ for any fixed $\ntm \neq \0_\n$, hence by Lemma \ref{lem:one_dim_gaussian_tail_bound} we have 
    \begin{equation} \label{eq:3g:large_ntm_norm:bound_inner_product}
        \Pr \left( \frac{\langle \ntm , \ent \rangle}{\nm{\ntm}} \ge \frac{\sqrt{d}}{4} \right) = \Pr \left( \frac{\langle \ntm , \ent \rangle}{\nm{\ntm}} \le -\frac{\sqrt{d}}{4} \right) \le \frac{4}{\sqrt{2\pi d}} \exp \left( -\frac{d}{32} \right)
    \end{equation}
    Combining \eqref{eq:3g:large_ntm_norm:decompose_nt}, \eqref{eq:3g:large_ntm_norm:bound_param} and \eqref{eq:3g:large_ntm_norm:bound_inner_product} gives that 
    \begin{align*}
        \nm{\nt}^2 &\ge \left( \frac{1}{3} \right)^2 \nm{\ntm}^2 - 2 \nu \lvert \langle \ntm, \ent \rangle \rvert \\ 
        &\ge \frac{1}{9} \nm{\ntm}^2 - \frac{\nu \sqrt{d}}{2} \nm{\ntm} \\ 
        &\ge \frac{1}{9} \cdot 36 \nu^2 d - \frac{\nu \sqrt{d}}{2} \cdot 6 \nu \sqrt{d} \\ 
        &= \nu^2 d
    \end{align*}
    with probability at least $1-\frac{8}{\sqrt{2\pi d}} \exp \left( -\frac{d}{32} \right) \ge 1 - \exp(-0.02d)$. This proves Lemma \ref{lem:3g:large_ntm_norm}. 
\end{proof}

We then proceed to bound $\nm{\nt}$ iteratively for $\nm{\ntm}^2 \le 36\nu^2 d$. Recall that \eqref{eq:3g:nt_update} gives 
\begin{equation*}
    \nt = \ntm + \frac{\dt}{2} \N^T \nabla_\lgv \log \P (\ltm) + \sqrt{\dt} \ent. 
\end{equation*}
We notice that the difficulty of solving $\nt$ exhibits in the dependence of $\log \P (\ltm)$ on $\rtm$. Since $\P = \sum_{i=0}^2 \wi \probi = \sum_{i=0}^2 \wi \gN(\mui, \vi^2 \mI_d)$, we can rewrite the score function as 
\begin{align}
    \nabla_\lgv \log \P(\lgv) &= \frac{\nabla_\lgv \P(\lgv)}{\P(\lgv)} = -\sum_{i=0}^2 \frac{\wi\probi(\lgv)}{\P(\lgv)} \cdot (\lgv-\mui) \notag \\
    &= -\frac{\lgv}{\nu^2} + \sum_{i=1}^2 \frac{\wi\probi(\lgv)}{\P(\lgv)} \left( \frac{\lgv}{\nu^2} - (\lgv-\mui) \right). \label{eq:3g:score_rewrite}
\end{align}
Now, instead of directly working with $\nt$, we consider a surrogate recursion $\hnt$ such that $\hnz = \nz$ and for all $t \ge 1$, 
\begin{equation} \label{eq:3g:def:hnt}
    \hnt = \hntm - \frac{\dt}{2\nu^2} \hntm + \sqrt{\dt} \ent. 
\end{equation}
The advantage of the surrogate recursion is that $\hnt$ is independent of $\rv$, thus we can obtain the closed-form solution to $\hnt$. Before we proceed to bound $\hnt$, we first show that $\hnt$ is sufficiently close to the original recursion $\nt$ in the following lemma.

\begin{lemma} \label{lem:3g:hnt_close_to_nt}
    For any $t \ge 1$, given that $\dj \le \nu^2$ and $\frac{\nu^2+1}{2} d \le \nm{\njm}^2 \le 36\nu^2 d$ for all $j \in [t]$ and $\nm{\rvmu}^2 \le 0.2d$, we have $\nm{\hnt - \nt} \le \frac{t}{\exp(0.04d)} \sqrt{d}$. 
\end{lemma}
\begin{proof}[Proof of Lemma \ref{lem:3g:hnt_close_to_nt}]
    Upon comparing \eqref{eq:3g:nt_update} and \eqref{eq:3g:def:hnt}, by \eqref{eq:3g:score_rewrite} we have that for all $j \in [t]$,
    \begin{align*}
        \nm{\hnj - \nj} &= \nm{\hnjm - \frac{\dj}{2\nu^2} \hnjm - \njm - \frac{\dj}{2} \N^T \nabla_\lgv \log \P (\ljm)} \\ 
        &= \nm{ \left( 1 - \frac{\dj}{2\nu^2} \right)(\hnjm - \njm) + \frac{\dj}{2} \sum_{i=1}^2 \frac{\wi\probi(\ljm)}{\P(\ljm)} \left( 1 - \frac{1}{\nu^2} \right) \njm } \\ 
        &\le \left( 1 - \frac{\dj}{2\nu^2} \right) \nm{\hnjm-\njm} + \sum_{i=1}^2 \frac{\dj}{2} \frac{\wi\probi(\ljm)}{\P(\ljm)} \left( 1 - \frac{1}{\nu^2} \right) \nm{\njm} \\ 
        &\le \nm{\hnjm-\njm} + \sum_{i=1}^2 \frac{\dj}{2} \frac{\wi\probi(\ljm)}{\wz\probz(\ljm)} \left( 1 - \frac{1}{\nu^2} \right) 6\nu \sqrt{d} . 
    \end{align*}
    By Lemma \ref{lem:3g:ratio_probi_probz}, we have $\frac{\probo(\ljm)}{\probz(\ljm)} \le \exp(-0.06d)$ and $\frac{\probt(\ljm)}{\probz(\ljm)} \le \exp(-0.06d)$, hence we obtain a recursive bound
    \begin{equation*}
        \nm{\hnj - \nj} \le \nm{\hnjm-\njm} + \frac{1}{\exp(0.04d)} \sqrt{d}. 
    \end{equation*}
    Finally, by $\hnz = \nz$, we have 
    \begin{equation*}
        \nm{\hnt-\nt} = \sum_{j\in[t]} \left( \nm{\hnj - \nj} - \nm{\hnjm-\njm} \right) \le \frac{t}{\exp(0.04d)} \sqrt{d}.
    \end{equation*}
    Hence we obtain Lemma \ref{lem:3g:hnt_close_to_nt}. 
\end{proof}

We then proceed to analyze $\hnt$, The following lemma gives us the closed-form solution of $\hnt$. We slightly abuse the notations here, e.g., $\prod_{i=c_1}^{c_2} \left( 1-\frac{\di}{2\nu^2} \right) = 1$ and $\sum_{j=c_1}^{c_2} \dj = 0$ for $c_1 > c_2$. 
\begin{lemma} \label{lem:3g:hnt_solution}
    For all $t \ge 0$, $\hnt \sim \gN \left( \prod_{i=1}^t \left( 1-\frac{\di}{2\nu^2} \right) \nz ,\; \sum_{j=1}^t \prod_{i=j+1}^t \left( 1-\frac{\di}{2\nu^2} \right)^2 \dj \mI_{d-1} \right)$, where the mean and covariance satisfy $\prod_{i=1}^t \left( 1-\frac{\di}{2\nu^2} \right)^2 + \frac{1}{\nu^2} \sum_{j=1}^t \prod_{i=j+1}^t \left( 1-\frac{\di}{2\nu^2} \right)^2 \dj \ge 1$. 
\end{lemma}
\begin{proof}[Proof of Lemma \ref{lem:3g:hnt_solution}]
    We prove the two properties by induction. When $t=0$, they are trivial. Suppose they hold for $t-1$, then for the distribution of $\hnt$, we have
    \begin{align*}
        \hnt &= \hntm - \frac{\dt}{2\nu^2} \hntm + \sqrt{\dt} \ent \\ 
        &\sim \gN \left( \left( 1-\frac{\dt}{2\nu^2} \right) \prod_{i=1}^{t-1} \left( 1-\frac{\di}{2\nu^2} \right) \nz ,\; \left( 1-\frac{\dt}{2\nu^2} \right)^2 \sum_{j=1}^{t-1} \prod_{i=j+1}^{t-1} \left( 1-\frac{\di}{2\nu^2} \right)^2 \dj \mI_{d-1} + \dt \mI_{d-1} \right) \\ 
        &= \gN \left( \prod_{i=1}^t \left( 1-\frac{\di}{2\nu^2} \right) \nz ,\; \sum_{j=1}^t \prod_{i=j+1}^t \left( 1-\frac{\di}{2\nu^2} \right)^2 \dj \mI_{d-1} \right). 
    \end{align*}
    For the second property, 
    \begin{align*}
        &\quad\prod_{i=1}^t \left( 1-\frac{\di}{2\nu^2} \right)^2 + \frac{1}{\nu^2} \sum_{j=1}^t \prod_{i=j+1}^t \left( 1-\frac{\di}{2\nu^2} \right)^2 \dj \\ 
        &= \left( 1-\frac{\dt}{2\nu^2} \right)^2 \left( \prod_{i=1}^{t-1} \left( 1-\frac{\di}{2\nu^2} \right)^2 + \frac{1}{\nu^2} \sum_{j=1}^{t-1} \prod_{i=j+1}^{t-1} \left( 1-\frac{\di}{2\nu^2} \right)^2 \dj \right) + \frac{1}{\nu^2} \dt \\ 
        &\ge \left( 1-\frac{\dt}{2\nu^2} \right)^2 + \frac{1}{\nu^2} \dt = 1 + \frac{\dt^2}{4\nu^4} \ge 1. 
    \end{align*}
    Hence we finish the proof of Lemma \ref{lem:3g:hnt_solution}. 
\end{proof}

Armed with Lemma \ref{lem:3g:hnt_solution}, we are now ready to establish the lower bound on $\nm{\hnt}$. For simplicity, denote $\hkalpha := \prod_{i=1}^t \left( 1-\frac{\di}{2\nu^2} \right)^2$ and $\hkbeta := \frac{1}{\nu^2} \sum_{j=1}^t \prod_{i=j+1}^t \left( 1-\frac{\di}{2\nu^2} \right)^2 \dj$. By Lemma \ref{lem:3g:hnt_solution} we know $\hnt \sim \gN(\hkalpha \nz, \hkbeta \nu^2 \mI_{d-1})$, so we can write $\hnt = \hkalpha \nz + \sqrt{\hkbeta} \nu \hkeps$, where $\hkeps \sim \gN(\0_{d-1}, \mI_{d-1})$.

\begin{lemma} \label{lem:3g:hnt_norm_lower_bound}
    Given that $\nm{\hnz}^2 \ge \frac{3\nu^2 + 1}{4} d$, we have $\nm{\hnt}^2 \ge \frac{5\nu^2+3}{8} d$ with probability at least $1-\exp \left( -d/300 \right)$. 
\end{lemma}
\begin{proof}[Proof of Lemma \ref{lem:3g:hnt_norm_lower_bound}]
    By $\hnt = \hkalpha \nz + \sqrt{\hkbeta} \nu \hkeps$ we have 
    \begin{equation*}
        \nm{\hnt}^2 = \hkalpha^2 \nm{\nz}^2 + \hkbeta \nu^2 \nm{\hkeps}^2 + 2\hkalpha \sqrt{\hkbeta} \nu \langle \nz , \hkeps \rangle
    \end{equation*}
    By Lemma \ref{lem:gaussian_tail_bound} we can bound 
    \begin{align*}
        \Pr \left( \nm{\hkeps}^2 \le \frac{3\nu^2+1}{4\nu^2} d \right) &= \Pr \left( \nm{\hkeps}^2 \le d - 2 \sqrt{ d \cdot \left( \frac{\nu^2-1}{8\nu^2} \right)^2 d} \right) \\ 
        &\le \Pr \left( \nm{\hkeps}^2 \le (d-1) - 2 \sqrt{(d-1) \left( \frac{\nu^2-1}{8\nu^2} \right)^2 \frac{d}{2}} \right) \\ 
        &\le \exp \left( - \left( \frac{\nu^2-1}{8\nu^2} \right)^2 \frac{d}{2} \right) \le \exp(-d/288).
    \end{align*}
    Since $\hkeps \sim \gN(\0_{d-1}, \mI_{d-1})$, we know $\frac{\langle \nz , \hkeps \rangle}{\nm{\nz}} \sim \gN(0, 1)$. Therefore by Lemma \ref{lem:one_dim_gaussian_tail_bound}, 
    \begin{equation*}
        \Pr \left( \frac{\langle \nz , \hkeps \rangle}{\nm{\nz}} \le -\frac{\nu^2-1}{4\nu \sqrt{3\nu^2+1}} \sqrt{d} \right) \le \frac{4\nu \sqrt{3\nu^2+1}}{\sqrt{2\pi}(\nu^2-1) \sqrt{d}} \exp \left( -\frac{(\nu^2-1)^2 d}{32\nu^2 (3\nu^2+1)} \right) \le \exp(-0.004d). 
    \end{equation*}
    Conditioned on $\nm{\hnz}^2 \ge \frac{3\nu^2 + 1}{4} d$,  $\nm{\hkeps}^2 > \frac{3\nu^2+1}{4\nu^2} d$ and $\frac{1}{\nm{\nz}} \langle \nz , \hkeps \rangle > -\frac{\nu^2-1}{4\nu \sqrt{3\nu^2+1}} \sqrt{d}$, since Lemma \ref{lem:3g:hnt_solution} gives $\hkalpha^2 + \hkbeta \ge 1$ we have
    \begin{align*}
        \nm{\hnt}^2 &= \hkalpha^2 \nm{\nz}^2 + \hkbeta \nu^2 \nm{\hkeps}^2 + 2\hkalpha \sqrt{\hkbeta} \nu \langle \nz , \hkeps \rangle \\ 
        &\ge \hkalpha^2 \nm{\nz}^2 + \hkbeta \nu^2 \nm{\hkeps}^2 - 2\hkalpha \sqrt{\hkbeta} \nu \nm{\nz} \frac{\nu^2-1}{4\nu \sqrt{3\nu^2+1}} \sqrt{d} \\ 
        &\ge \hkalpha^2 \nm{\nz}^2 + \hkbeta \nu^2 \nm{\hkeps}^2 - 2\hkalpha \sqrt{\hkbeta} \nu \nm{\nz} \nm{\hkeps} \cdot \frac{\nu^2-1}{6\nu^2+21} \\ 
        &\ge \left( 1 - \frac{\nu^2-1}{6\nu^2+21} \right) \left( \hkalpha^2 \nm{\nz}^2 + \hkbeta \nu^2 \nm{\hkeps}^2 \right) \\ 
        &\ge \frac{5\nu^2+3}{6\nu^2+21} \left( \hkalpha^2 + \hkbeta \right) \cdot \frac{3\nu^2 + 1}{4} d \\ 
        &\ge \frac{5\nu^2+3}{8} d.
    \end{align*}
    Hence by union bound, we complete the proof of Lemma \ref{lem:3g:hnt_norm_lower_bound}. 
\end{proof}

Upon having all the above lemmas, we are now ready to establish Theorem \ref{thm:3g:3gaussian} by induction. Suppose the theorem holds for all $T$ values of $1, \cdots, T-1$. We consider the following 3 cases:
\begin{itemize}
    \item If there exists some $t \in [T]$ such that $\dt > \nu^2$, by Lemma \ref{lem:3g:large_dt} we know that with probability at least $1-\exp(-d/25)$, we have $\nm{\nt}^2 \ge \frac{3\nu^2+1}{4}d$, thus the problem reduces to the two sub-arrays $\nz,\cdots,\ntm$ and $\nt,\cdots,\nT$, which can be solved by induction. 
    \item Suppose $\dt \le \nu^2$ for all $t \in [T]$. If there exists some $t \in [T]$ such that $\nm{\ntm}^2 > 36 \nu^2 d$, by Lemma \ref{lem:3g:large_ntm_norm} we know that with probability at least $1-\exp(-d/50)$, we have $\nm{\nt}^2 \ge \nu^2 d > \frac{3\nu^2+1}{4}d$, thus the problem similarly reduces to the two sub-arrays $\nz,\cdots,\ntm$ and $\nt,\cdots,\nT$, which can be solved by induction. 
    \item Suppose $\dt \le \nu^2$ and $\nm{\ntm}^2 \le 36 \nu^2 d$ for all $t \in [T]$. Conditioned on $\nm{\ntm}^2 > \frac{\nu^2+1}{2} d$ for all $t \in [T]$, by Lemma \ref{lem:3g:hnt_close_to_nt} we have that for $T \le \exp(d/300)$, 
        \begin{equation*}
            \nm{\hnT - \nT} < \left( \sqrt{\frac{5\nu^2+3}{8}} - \sqrt{\frac{\nu^2+1}{2}} \right) \sqrt{d}. 
        \end{equation*}
        By Lemma \ref{lem:3g:hnt_norm_lower_bound} we have that with probability at least $1-\exp(-d/300)$,
        \begin{equation*}
            \nm{\hnT}^2 \ge \frac{5\nu^2+3}{8} d.
        \end{equation*}
        Combining the two inequalities implies the desired bound
        \begin{equation*}
            \nm{\nT} \ge \nm{\hnT} - \nm{\hnT-\nT} > \sqrt{\frac{\nu^2+1}{2} d}.
        \end{equation*} 
        Hence by induction we obtain $\nm{\nt}^2 > \frac{\nu^2+1}{2} d$ for all $t \in [T]$ with probability at least 
        \begin{equation*}
            (1-(T-1)\exp(-d/300)) \cdot (1-\exp(-d/300)) \ge 1-T\exp(-d/300). 
        \end{equation*}
\end{itemize}
Therefore we complete the proof of Theorem \ref{thm:3g:3gaussian}.

\end{proof}

\begin{proof}[Proof of Corollary \ref{coro:3g:tv}]
    By the definition of total variation distance, we have 
    \begin{equation*}
        \TV(\hP_T, \P) = \sup_A | \hP_T(A) - \P(A) |. 
    \end{equation*}
    Specifically, by choosing the event $A$ as $\set{\lgv: \du{\lgv}^2 \ge \frac{\nu^2+1}{2} d}$, from Theorem \ref{thm:3g:3gaussian} we know $\hP_T(A) \ge 1 - T \cdot \exp(-d/300)$. On the other hand, by Definition \ref{def:dist}, $\du{\lgv}^2 \ge \frac{\nu^2+1}{2} d$ implies $\nm{\lgv-\rvmu}^2 \ge \frac{\nu^2+1}{2} d$. Therefore, from Lemma \ref{lem:gaussian_tail_bound} we have 
    \begin{align*}
        \probo(A) &\le \probo \left(\nm{\lgv-\rvmu}^2 \ge \frac{\nu^2+1}{2} d \right) \le \exp \left( - \left( \frac{\nu-1}{2} \right)^2 d \right) \le \exp \left(-\frac{d}{10} \right). 
    \end{align*}
    From the same derivation we can obtain $\probt(A) \le \exp(-d/10)$. Combining all bounds gives an lower bound on the total variation distance 
    \begin{align*}
        \TV(\hP_T,\P) &\ge \hP_T(A) - \P(A) \ge 1- T \cdot \exp \left(-\frac{d}{300}\right) - \P(A) \\ 
        &\ge 1- T \cdot \exp \left(-\frac{d}{300}\right) - \left( w \probz(A) + \frac{1-w}{2} \probo(A) + \frac{1-w}{2} \probt(A) \right) \\ 
        &\ge 1- T \cdot \exp \left(-\frac{d}{300}\right) - \left(w + (1-w) \exp \left(-\frac{d}{10} \right) \right) \\
        &\ge 0.99 - w - \frac{T}{\exp(-d/300)}. 
    \end{align*}
\end{proof}

\subsection{Proof of Theorem \ref{thm:ld:gaussian_mixture}} \label{app:prf:ld_gaussian_mixture}

\begin{proof}[Proof of Theorem \ref{thm:ld:gaussian_mixture}]
The proof of Theorem \ref{thm:ld:gaussian_mixture} follows from a similar framework to the proof of Theorem \ref{thm:3g:3gaussian}.  Let $\r$ and $\n$ respectively denote the rank and nullity of the vector space $\set{\mui}_{i \in [k]}$, then we have $\r+\n = d$ and $0 \le \r \le k = o(d)$. Denote $\R \in \sR^{d \times \r}$ an orthonormal basis of the vector space $\set{\mui}_{i \in [k]}$, and denote $\N \in \sR^{d \times \n}$ an orthonormal basis of the null space of $\set{\mui}_{i \in [k]}$. Now consider decomposing the sample $\lt$ by 
\begin{equation*} 
    \rt := \R^T \lt \text{, and } \nt := \N^T \lt, 
\end{equation*}
where $\rt \in \sR^\r$, $\nt \in \sR^\n$. Then we have 
\begin{equation*} 
    \lt = \R \rt + \N \nt. 
\end{equation*}
Similarly, we decompose the noise $\epst$ into 
\begin{equation*} 
    \ert := \R^T \epst \text{, and } \ent := \N^T \epst, 
\end{equation*}
where $\ert \in \sR^\r$, $\ent \in \sR^\n$. Then we have 
\begin{equation*} 
    \epst = \R \ert + \N \ent. 
\end{equation*}
Since a linear combination of a Gaussian random variable still follows Gaussian distribution, by $\epst \sim \gN(\0_d, \mI_d)$, $\R^T \R = \mI_{\r}$, and $\N^T \N = \mI_{\n}$ we obtain 
\begin{equation*} 
    \ert \sim \gN(\0_{\r}, \mI_{\r}) \text{, and } \ent \sim \gN(\0_{\n}, \mI_{\n}). 
\end{equation*}
By the definition of Langevin dynamics in \eqref{eq:ld:ld_def}, $\nt$ follow from the update rule: 
\begin{align}
    \nt &= \ntm + \frac{\dt}{2} \N^T \nabla_\lgv \log \P (\ltm) + \sqrt{\dt} \ent \label{eq:ld:nt_update} .
\end{align}
By Definition \ref{def:dist}, since $\nt$ is the projection onto the null space of $\set{\mui}_{i \in [k]}$, we have 
\begin{equation*}
    \duk{\lt} = \min_{\lambda_1, \cdots, \lambda_k} \nm{\lt - \sum_{i=1}^k \lambda_i \mui} = \nm{\nt}. 
\end{equation*}
Then, with the assumption that the initialization satisfies $\nm{\nz}^2 \ge \frac{3\vz^2+\vmax^2}{4} d$, the objective is to show that $\nm{\nt}$ remains large with high probability.

To establish a lower bound on $\nm{\nt}$, we consider different cases of the step size $\dt$. Intuitively, when $\dt$ is large enough, $\nt$ will be too noisy due to the introduction of random noise $\sqrt{\dt} \ent$ in \eqref{eq:ld:nt_update}. While for small $\dt$, the update of $\nt$ is bounded and thus we can iteratively analyze $\nt$. We first handle the case of large $\dt$ in the following lemma.

\begin{lemma} \label{lem:ld:large_dt}
    If $\dt > \vz^2$, with probability at least $1-\exp(-0.04d)$, for $\nt$ satisfying \eqref{eq:ld:nt_update}, we have $\nm{\nt}^2 \ge \frac{3\vz^2+\vmax^2}{4} d$ regardless of the previous state $\ltm$. 
\end{lemma}
\begin{proof}[Proof of Lemma \ref{lem:ld:large_dt}]
    Denote $\rvv := \ntm + \frac{\dt}{2} \N^T \nabla_\lgv \log \P (\ltm)$ for simplicity. Note that $\rvv$ is fixed for any given $\ltm$. We decompose $\ent$ into a vector aligning with $\rvv$ and another vector orthogonal to $\rvv$. Consider an orthonormal matrix $\rmM \in \sR^{\n \times (\n-1)}$ such that $\rmM^T \rvv = \0_{\n-1}$ and $\rmM^T \rmM = \mI_{\n-1}$. By denoting $\rvu := \ent - \rmM \rmM^T \ent$ we have $\rmM^T \rvu = \0_{\n-1}$, thus we obtain
    \begin{align*}
        \nm{\nt}^2 &= \nm{\rvv + \sqrt{\dt} \ent}^2 \\ 
        &= \nm{\rvv + \sqrt{\dt} \rvu + \sqrt{\dt} \rmM \rmM^T \ent}^2 \\ 
        &= \nm{\rvv + \sqrt{\dt} \rvu}^2 + \nm{\sqrt{\dt} \rmM \rmM^T \ent}^2 \\ 
        &\ge \nm{\sqrt{\dt} \rmM \rmM^T \ent}^2 \\ 
        &\ge \vz^2 \nm{\rmM^T \ent}^2. 
    \end{align*}
    Since $\ent \sim \gN(\0_\n, \mI_\n)$ and $\rmM^T \rmM = \mI_{\n-1}$, we obtain $\rmM^T \ent \sim \gN(\0_{\n-1}, \mI_{\n-1})$. Therefore, by Lemma \ref{lem:gaussian_tail_bound} we can bound 
    \begin{align*}
        \Pr \left( \nm{\nt}^2 \le \frac{3\vz^2 + \vmax^2}{4} d \right) 
        &\le \Pr \left( \nm{\rmM^T \ent}^2 \le \frac{3\vz^2 + \vmax^2}{4\vz^2} d \right) \\ 
        &= \Pr \left( \nm{\rmM^T \ent}^2 \le d - 2 \sqrt{d \cdot \left( \frac{\vz^2-\vmax^2}{8\vz^2} \right)^2 d} \right) \\ 
        &\le \Pr \left( \nm{\rmM^T \ent}^2 \le (\n-1) - 2 \sqrt{(\n-1) \left( \frac{\vz^2-\vmax^2}{8\vz^2} \right)^2 \frac{d}{2}} \right) \\ 
        &\le \exp \left( - \left( \frac{\vz^2-\vmax^2}{8\vz^2} \right)^2 \frac{d}{2} \right) \le \exp \left( -\frac{d}{24} \right),
    \end{align*}
    Hence we complete the proof of Lemma \ref{lem:ld:large_dt}. 
\end{proof}

We then consider the case when $\dt \le \vz^2$. We first show that when $\nm{\nv}^2 \ge \frac{\vz^2+\vmax^2}{2} d$, $\probi(\lgv)$ is exponentially smaller than $\probz(\lgv)$ for all $i \in [k]$ in the following lemma. 

\begin{lemma} \label{lem:ld:ratio_probi_probz}
    Given that $\nm{\nv}^2 \ge \frac{\vz^2+\vmax^2}{2} d$ and $\nm{\mui}^2 \le 0.2 d$ for all $i \in [k]$, we have $\frac{\probi(\lgv)}{\probz(\lgv)} \le \exp(-0.06d)$ for all $i \in [k]$. 
\end{lemma}
\begin{proof}[Proof of Lemma \ref{lem:ld:ratio_probi_probz}]
    For all $i \in [k]$, define $\rhoi(\lgv) := \frac{\probi(\lgv)}{\probz(\lgv)}$, then 
    \begin{align*}
        \rhoi(\lgv) &= \frac{\probi(\lgv)}{\probz(\lgv)} = \frac{(2\pi \vi^2)^{-d/2} \exp \left( -\frac{1}{2\vi^2} \nm{\lgv - \mui}^2 \right)}{(2\pi \vz^2)^{-d/2} \exp \left( -\frac{1}{2\vz^2} \nm{\lgv}^2 \right)} \\ 
        &= \left( \frac{\vz^2}{\vi^2} \right)^{d/2} \exp \left( \frac{1}{2\vz^2} \nm{\lgv}^2 - \frac{1}{2\vi^2} \nm{\lgv - \mui}^2 \right) \\ 
        &= \left( \frac{\vz^2}{\vi^2} \right)^{d/2} \exp \left( \left( \frac{1}{2\vz^2}-\frac{1}{2\vi^2} \right) \nm{\N\nv}^2 + \left( \frac{\nm{\R\rv}^2}{2\vz^2} - \frac{\nm{\R\rv-\mui}^2}{2\vi^2} \right) \right) \\ 
        &= \left( \frac{\vz^2}{\vi^2} \right)^{d/2} \exp \left( \left( \frac{1}{2\vz^2}-\frac{1}{2\vi^2} \right) \nm{\nv}^2 + \left( \frac{\nm{\rv}^2}{2\vz^2} - \frac{\nm{\rv-\R^T\mui}^2}{2\vi^2} \right) \right), 
    \end{align*}
    where the last step follows from the definition that $\R \in \sR^{d \times \r}$ an orthonormal basis of the vector space $\set{\mui}_{i \in [k]}$ and $\N^T \N = \mI_\n$. Since $\vz^2 > \vi^2$, the quadratic term $\frac{\nm{\rv}^2}{2\vz^2} - \frac{\nm{\rv - \R^T\mui}^2}{2\vi^2}$ is maximized at $\rv = \frac{\vz^2 \R^T \mui}{\vz^2-\vi^2}$. Therefore, 
    \begin{align*}
        \frac{\nm{\rv}^2}{2\vz^2} - \frac{\nm{\rv - \R^T\mui}^2}{2\vi^2} &\le \frac{\vz^4\nm{\R^T \mui}^2}{2\vz^2(\vz^2-\vi^2)^2} - \frac{1}{2\vi^2} \left( \frac{\vz^2}{\vz^2-\vi^2}-1 \right)^2 \nm{\R^T \mui}^2 = \frac{\nm{\mui}^2}{2(\vz^2-\vi^2)}. 
    \end{align*}
    Hence, for $\nm{\nv}^2 \ge \frac{\vz^2+\vmax^2}{2} d$ and $\nm{\mui}^2 \le 0.2d \le \frac{\vz^2-\vi^2}{2} \left( \log \left( \frac{\vi^2}{\vz^2} \right) - \frac{\vi^2}{2\vz^2} + \frac{\vz^2}{2\vi^2} \right) d$, we have 
    \begin{align*}
        \rhoi(\lgv) &= \left( \frac{\vz^2}{\vi^2} \right)^{d/2} \exp \left( \left( \frac{1}{2\vz^2}-\frac{1}{2\vi^2} \right) \nm{\nv}^2 + \left( \frac{\nm{\rv}^2}{2\vz^2} - \frac{\nm{\rv-\R^T\mui}^2}{2\vi^2} \right) \right) \\ 
        &\le \left( \frac{\vz^2}{\vi^2} \right)^{d/2} \exp \left( \left( \frac{1}{2\vz^2}-\frac{1}{2\vi^2} \right) \frac{\vz^2+\vi^2}{2} d + \frac{\nm{\mui}^2}{2(\vz^2-\vi^2)} \right) \\ 
        &= \exp \left( -\left( \log \left( \frac{\vi^2}{\vz^2} \right) - \frac{\vi^2}{2\vz^2} + \frac{\vz^2}{2\vi^2} \right) \frac{d}{2} + \frac{\nm{\mui}^2}{2(\vz^2-\vi^2)} \right) \\
        &\le \exp \left( -\left( \log \left( \frac{\vi^2}{\vz^2} \right) - \frac{\vi^2}{2\vz^2} + \frac{\vz^2}{2\vi^2} \right) \frac{d}{4} \right) \le \exp(-0.06d). 
    \end{align*}
    Therefore we finish the proof of Lemma \ref{lem:ld:ratio_probi_probz}. 
\end{proof}

Lemma \ref{lem:ld:ratio_probi_probz} implies that when $\nm{\nv}$ is large, the Gaussian mode $\probz$ dominates other modes $\probi$. To bound $\nm{\nt}$, we first consider a simpler case that $\nm{\ntm}$ is large. Intuitively, the following lemma proves that when the previous state $\ntm$ is far from a mode, a single step of Langevin dynamics with bounded step size is not enough to find the mode. 
\begin{lemma} \label{lem:ld:large_ntm_norm}
    Suppose $\dt \le \vz^2$ and $\nm{\ntm}^2 > 36 \vz^2 d$, then for $\nt$ following from \eqref{eq:ld:nt_update}, we have $\nm{\nt}^2 \ge \vz^2 d$ with probability at least $1-\exp(-0.02d)$. 
\end{lemma}
\begin{proof}[Proof of Lemma \ref{lem:ld:large_ntm_norm}]
    From the recursion of $\nt$ in \eqref{eq:ld:nt_update} we have
    \begin{align}
        \nt &= \ntm + \frac{\dt}{2} \N^T \nabla_\lgv \log \P (\ltm) + \sqrt{\dt} \ent \notag \\ 
        &= \ntm - \frac{\dt}{2} \sum_{i=0}^k \frac{\wi\probi(\ltm)}{\P(\ltm)} \cdot \frac{\N^T (\ltm-\mui)}{\vi^2} + \sqrt{\dt} \ent \notag \\ 
        &= \left( 1 - \frac{\dt}{2} \sum_{i=0}^k \frac{\wi\probi(\ltm)}{\P(\ltm)} \cdot \frac{1}{\vi^2} \right) \ntm + \sqrt{\dt} \ent. \label{eq:ld:large_ntm_norm:decompose_nt}
    \end{align}
    By Lemma \ref{lem:ld:ratio_probi_probz}, we have $\frac{\probi(\ljm)}{\probz(\ljm)} \le \exp(-0.06d)$ for all $i \in [k]$, therefore 
    \begin{equation}
        1 - \frac{\dt}{2} \sum_{i=0}^k \frac{\wi\probi(\ltm)}{\P(\ltm)} \cdot \frac{1}{\vi^2} \ge 1 - \frac{\dt}{2} \cdot \frac{1}{\vz^2} - \frac{\dt(1-w)}{2\vi^2 w} \exp (-0.06d) > \frac{1}{3}. \label{eq:ld:large_ntm_norm:bound_param}
    \end{equation}
    On the other hand, from $\ent \sim \gN(\0_\n, \mI_\n)$ we know $\frac{\langle \ntm , \ent \rangle}{\nm{\ntm}} \sim \gN(0, 1)$ for any fixed $\ntm \neq \0_\n$, hence by Lemma \ref{lem:one_dim_gaussian_tail_bound} we have 
    \begin{equation} \label{eq:ld:large_ntm_norm:bound_inner_product}
        \Pr \left( \frac{\langle \ntm , \ent \rangle}{\nm{\ntm}} \ge \frac{\sqrt{d}}{4} \right) = \Pr \left( \frac{\langle \ntm , \ent \rangle}{\nm{\ntm}} \le -\frac{\sqrt{d}}{4} \right) \le \frac{4}{\sqrt{2\pi d}} \exp \left( -\frac{d}{32} \right)
    \end{equation}
    Combining \eqref{eq:ld:large_ntm_norm:decompose_nt}, \eqref{eq:ld:large_ntm_norm:bound_param} and \eqref{eq:ld:large_ntm_norm:bound_inner_product} gives that 
    \begin{align*}
        \nm{\nt}^2 &\ge \left( \frac{1}{3} \right)^2 \nm{\ntm}^2 - 2 \vz \lvert \langle \ntm, \ent \rangle \rvert \\ 
        &\ge \frac{1}{9} \nm{\ntm}^2 - \frac{\vz \sqrt{d}}{2} \nm{\ntm} \\ 
        &\ge \frac{1}{9} \cdot 36 \vz^2 d - \frac{\vz \sqrt{d}}{2} \cdot 6 \vz \sqrt{d} \\ 
        &= \vz^2 d
    \end{align*}
    with probability at least $1-\frac{8}{\sqrt{2\pi d}} \exp \left( -\frac{d}{32} \right) \ge 1 - \exp(-0.02d)$. This proves Lemma \ref{lem:ld:large_ntm_norm}. 
\end{proof}

We then proceed to bound $\nm{\nt}$ iteratively for $\nm{\ntm}^2 \le 36\vz^2 d$. Recall that \eqref{eq:ld:nt_update} gives 
\begin{equation*}
    \nt = \ntm + \frac{\dt}{2} \N^T \nabla_\lgv \log \P (\ltm) + \sqrt{\dt} \ent. 
\end{equation*}
We notice that the difficulty of solving $\nt$ exhibits in the dependence of $\log \P (\ltm)$ on $\rtm$. Since $\P = \sum_{i=0}^k \wi \probi = \sum_{i=0}^k \wi \gN(\mui, \vi^2 \mI_d)$, we can rewrite the score function as 
\begin{align}
    \nabla_\lgv \log \P(\lgv) &= \frac{\nabla_\lgv \P(\lgv)}{\P(\lgv)} = -\sum_{i=0}^k \frac{\probi(\lgv)}{\P(\lgv)} \cdot \frac{\lgv-\mui}{\vi^2} 
    = -\frac{\lgv}{\vz^2} + \sum_{i\in[k]} \frac{\probi(\lgv)}{\P(\lgv)} \left( \frac{\lgv}{\vz^2} - \frac{\lgv-\mui}{\vi^2} \right). \label{eq:ld:score_rewrite}
\end{align}
Now, instead of directly working with $\nt$, we consider a surrogate recursion $\hnt$ such that $\hnz = \nz$ and for all $t \ge 1$, 
\begin{equation} \label{eq:ld:def:hnt}
    \hnt = \hntm - \frac{\dt}{2\vz^2} \hntm + \sqrt{\dt} \ent. 
\end{equation}
The advantage of the surrogate recursion is that $\hnt$ is independent of $\rv$, thus we can obtain the closed-form solution to $\hnt$. Before we proceed to bound $\hnt$, we first show that $\hnt$ is sufficiently close to the original recursion $\nt$ in the following lemma.

\begin{lemma} \label{lem:ld:hnt_close_to_nt}
    For any $t \ge 1$, given that $\dj \le \vz^2$ and $\frac{\vz^2+\vmax^2}{2} d \le \nm{\njm}^2 \le 36\vz^2 d$ for all $j \in [t]$ and $\nm{\mui}^2 \le 0.2 d$ for all $i \in [k]$, we have $\nm{\hnt - \nt} \le \frac{t}{\exp(0.04d)} \sqrt{d}$. 
\end{lemma}
\begin{proof}[Proof of Lemma \ref{lem:ld:hnt_close_to_nt}]
    Upon comparing \eqref{eq:ld:nt_update} and \eqref{eq:ld:def:hnt}, by \eqref{eq:ld:score_rewrite} we have that for all $j \in [t]$,
    \begin{align*}
        \nm{\hnj - \nj} &= \nm{\hnjm - \frac{\dj}{2\vz^2} \hnjm - \njm - \frac{\dj}{2} \N^T \nabla_\lgv \log \P (\ljm)} \\ 
        &= \nm{ \left( 1 - \frac{\dj}{2\vz^2} \right)(\hnjm - \njm) + \frac{\dj}{2} \sum_{i\in[k]} \frac{\wi\probi(\ljm)}{\P(\ljm)} \left( \frac{1}{\vi^2} - \frac{1}{\vz^2} \right) \njm } \\ 
        &\le \left( 1 - \frac{\dj}{2\vz^2} \right) \nm{\hnjm-\njm} + \sum_{i \in [k]} \frac{\dj}{2} \frac{\wi\probi(\ljm)}{\P(\ljm)} \left( \frac{1}{\vi^2} - \frac{1}{\vz^2} \right) \nm{\njm} \\ 
        &\le \nm{\hnjm-\njm} + \sum_{i \in [k]} \frac{\dj}{2} \frac{\wi\probi(\ljm)}{\wz\probz(\ljm)} \left( \frac{1}{\vi^2} - \frac{1}{\vz^2} \right) 6\vz \sqrt{d} . 
    \end{align*}
    By Lemma \ref{lem:ld:ratio_probi_probz}, we have $\frac{\probi(\ljm)}{\probz(\ljm)} \le \exp(-0.06d)$ for all $i \in [k]$, hence we obtain a recursive bound
    \begin{equation*}
        \nm{\hnj - \nj} \le \nm{\hnjm-\njm} + \frac{1}{\exp(0.04d)} \sqrt{d}. 
    \end{equation*}
    Finally, by $\hnz = \nz$, we have 
    \begin{equation*}
        \nm{\hnt-\nt} = \sum_{j\in[t]} \left( \nm{\hnj - \nj} - \nm{\hnjm-\njm} \right) \le \frac{t}{\exp(0.04d)} \sqrt{d}.
    \end{equation*}
    Hence we obtain Lemma \ref{lem:ld:hnt_close_to_nt}. 
\end{proof}

We then proceed to analyze $\hnt$, The following lemma gives us the closed-form solution of $\hnt$. We slightly abuse the notations here, e.g., $\prod_{i=c_1}^{c_2} \left( 1-\frac{\di}{2\vz^2} \right) = 1$ and $\sum_{j=c_1}^{c_2} \dj = 0$ for $c_1 > c_2$. 
\begin{lemma} \label{lem:ld:hnt_solution}
    For all $t \ge 0$, $\hnt \sim \gN \left( \prod_{i=1}^t \left( 1-\frac{\di}{2\vz^2} \right) \nz ,\; \sum_{j=1}^t \prod_{i=j+1}^t \left( 1-\frac{\di}{2\vz^2} \right)^2 \dj \mI_{\n} \right)$, where the mean and covariance satisfy $\prod_{i=1}^t \left( 1-\frac{\di}{2\vz^2} \right)^2 + \frac{1}{\vz^2} \sum_{j=1}^t \prod_{i=j+1}^t \left( 1-\frac{\di}{2\vz^2} \right)^2 \dj \ge 1$. 
\end{lemma}
\begin{proof}[Proof of Lemma \ref{lem:ld:hnt_solution}]
    We prove the two properties by induction. When $t=0$, they are trivial. Suppose they hold for $t-1$, then for the distribution of $\hnt$, we have
    \begin{align*}
        \hnt &= \hntm - \frac{\dt}{2\vz^2} \hntm + \sqrt{\dt} \ent \\ 
        &\sim \gN \left( \left( 1-\frac{\dt}{2\vz^2} \right) \prod_{i=1}^{t-1} \left( 1-\frac{\di}{2\vz^2} \right) \nz ,\; \left( 1-\frac{\dt}{2\vz^2} \right)^2 \sum_{j=1}^{t-1} \prod_{i=j+1}^{t-1} \left( 1-\frac{\di}{2\vz^2} \right)^2 \dj \mI_{\n} + \dt \mI_{\n} \right) \\ 
        &= \gN \left( \prod_{i=1}^t \left( 1-\frac{\di}{2\vz^2} \right) \nz ,\; \sum_{j=1}^t \prod_{i=j+1}^t \left( 1-\frac{\di}{2\vz^2} \right)^2 \dj \mI_{\n} \right). 
    \end{align*}
    For the second property, 
    \begin{align*}
        &\quad\prod_{i=1}^t \left( 1-\frac{\di}{2\vz^2} \right)^2 + \frac{1}{\vz^2} \sum_{j=1}^t \prod_{i=j+1}^t \left( 1-\frac{\di}{2\vz^2} \right)^2 \dj \\ 
        &= \left( 1-\frac{\dt}{2\vz^2} \right)^2 \left( \prod_{i=1}^{t-1} \left( 1-\frac{\di}{2\vz^2} \right)^2 + \frac{1}{\vz^2} \sum_{j=1}^{t-1} \prod_{i=j+1}^{t-1} \left( 1-\frac{\di}{2\vz^2} \right)^2 \dj \right) + \frac{1}{\vz^2} \dt \\ 
        &\ge \left( 1-\frac{\dt}{2\vz^2} \right)^2 + \frac{1}{\vz^2} \dt = 1 + \frac{\dt^2}{4\vz^4} \ge 1. 
    \end{align*}
    Hence we finish the proof of Lemma \ref{lem:ld:hnt_solution}. 
\end{proof}

Armed with Lemma \ref{lem:ld:hnt_solution}, we are now ready to establish the lower bound on $\nm{\hnt}$. For simplicity, denote $\hkalpha := \prod_{i=1}^t \left( 1-\frac{\di}{2\vz^2} \right)^2$ and $\hkbeta := \frac{1}{\vz^2} \sum_{j=1}^t \prod_{i=j+1}^t \left( 1-\frac{\di}{2\vz^2} \right)^2 \dj$. By Lemma \ref{lem:ld:hnt_solution} we know $\hnt \sim \gN(\hkalpha \nz, \hkbeta \vz^2 \mI_{\n})$, so we can write $\hnt = \hkalpha \nz + \sqrt{\hkbeta} \vz \hkeps$, where $\hkeps \sim \gN(\0_\n, \mI_{\n})$.

\begin{lemma} \label{lem:ld:hnt_norm_lower_bound}
    Given that $\nm{\hnz}^2 \ge \frac{3\vz^2 + \vmax^2}{4} d$, we have $\nm{\hnt}^2 \ge \frac{5\vz^2+3\vmax^2}{8} d$ with probability at least $1-\exp \left( -d/300 \right)$. 
\end{lemma}
\begin{proof}[Proof of Lemma \ref{lem:ld:hnt_norm_lower_bound}]
    By $\hnt = \hkalpha \nz + \sqrt{\hkbeta} \vz \hkeps$ we have 
    \begin{equation*}
        \nm{\hnt}^2 = \hkalpha^2 \nm{\nz}^2 + \hkbeta \vz^2 \nm{\hkeps}^2 + 2\hkalpha \sqrt{\hkbeta} \vz \langle \nz , \hkeps \rangle
    \end{equation*}
    By Lemma \ref{lem:gaussian_tail_bound} we can bound 
    \begin{align*}
        \Pr \left( \nm{\hkeps}^2 \le \frac{3\vz^2+\vmax^2}{4\vz^2} d \right) &= \Pr \left( \nm{\hkeps}^2 \le d - 2 \sqrt{ d \cdot \left( \frac{\vz^2-\vmax^2}{8\vz^2} \right)^2 d} \right) \\ 
        &\le \Pr \left( \nm{\hkeps}^2 \le \n - 2 \sqrt{\n \left( \frac{\vz^2-\vmax^2}{8\vz^2} \right)^2 \frac{d}{2}} \right) \\ 
        &\le \exp \left( - \left( \frac{\vz^2-\vmax^2}{8\vz^2} \right)^2 \frac{d}{2} \right) \le \exp(-d/288),
    \end{align*}
    where the second last step follows from the assumption $d-\n = \r = o(d)$. Since $\hkeps \sim \gN(\0_\n, \mI_{\n})$, we know $\frac{\langle \nz , \hkeps \rangle}{\nm{\nz}} \sim \gN(0, 1)$. Therefore by Lemma \ref{lem:one_dim_gaussian_tail_bound}, 
    \begin{align*}
        \Pr \left( \frac{\langle \nz , \hkeps \rangle}{\nm{\nz}} \le -\frac{\vz^2-\vmax^2}{4\vz \sqrt{3\vz^2+\vmax^2}} \sqrt{d} \right) &\le \frac{4\vz \sqrt{3\vz^2+\vmax^2}}{\sqrt{2\pi}(\vz^2-\vmax^2) \sqrt{d}} \exp \left( -\frac{(\vz^2-\vmax^2)^2 d}{32\vz^2 (3\vz^2+\vmax^2)} \right)  \\ 
        &\le \exp(-0.004d).
    \end{align*}
    Conditioned on $\nm{\hnz}^2 \ge \frac{3\vz^2 + \vmax^2}{4} d$,  $\nm{\hkeps}^2 > \frac{3\vz^2+\vmax^2}{4\vz^2} d$ and $\frac{1}{\nm{\nz}} \langle \nz , \hkeps \rangle > -\frac{\vz^2-\vmax^2}{4\vz \sqrt{3\vz^2+\vmax^2}} \sqrt{d}$, since Lemma \ref{lem:ld:hnt_solution} gives $\hkalpha^2 + \hkbeta \ge 1$ we have
    \begin{align*}
        \nm{\hnt}^2 &= \hkalpha^2 \nm{\nz}^2 + \hkbeta \vz^2 \nm{\hkeps}^2 + 2\hkalpha \sqrt{\hkbeta} \vz \langle \nz , \hkeps \rangle \\ 
        &\ge \hkalpha^2 \nm{\nz}^2 + \hkbeta \vz^2 \nm{\hkeps}^2 - 2\hkalpha \sqrt{\hkbeta} \vz \nm{\nz} \frac{\vz^2-\vmax^2}{4\vz \sqrt{3\vz^2+\vmax^2}} \sqrt{d} \\ 
        &\ge \hkalpha^2 \nm{\nz}^2 + \hkbeta \vz^2 \nm{\hkeps}^2 - 2\hkalpha \sqrt{\hkbeta} \vz \nm{\nz} \nm{\hkeps} \cdot \frac{\vz^2-\vmax^2}{6\vz^2+2\vmax^2} \\ 
        &\ge \left( 1 - \frac{\vz^2-\vmax^2}{6\vz^2+2\vmax^2} \right) \left( \hkalpha^2 \nm{\nz}^2 + \hkbeta \vz^2 \nm{\hkeps}^2 \right) \\ 
        &\ge \frac{5\vz^2+3\vmax^2}{6\vz^2+2\vmax^2} \left( \hkalpha^2 + \hkbeta \right) \cdot \frac{3\vz^2 + \vmax^2}{4} d \\ 
        &\ge \frac{5\vz^2+3\vmax^2}{8} d.
    \end{align*}
    Hence by union bound, we complete the proof of Lemma \ref{lem:ld:hnt_norm_lower_bound}. 
\end{proof}

Upon having all the above lemmas, we are now ready to establish Theorem \ref{thm:ld:gaussian_mixture} by induction. Suppose the theorem holds for all $T$ values of $1, \cdots, T-1$. We consider the following 3 cases:
\begin{itemize}
    \item If there exists some $t \in [T]$ such that $\dt > \vz^2$, by Lemma \ref{lem:ld:large_dt} we know that with probability at least $1-\exp(-d/25)$, we have $\nm{\nt}^2 \ge \frac{3\vz^2+\vmax^2}{4}d$, thus the problem reduces to the two sub-arrays $\nz,\cdots,\ntm$ and $\nt,\cdots,\nT$, which can be solved by induction. 
    \item Suppose $\dt \le \vz^2$ for all $t \in [T]$. If there exists some $t \in [T]$ such that $\nm{\ntm}^2 > 36 \vz^2 d$, by Lemma \ref{lem:ld:large_ntm_norm} we know that with probability at least $1-\exp(-d/50)$, we have $\nm{\nt}^2 \ge \vz^2 d > \frac{3\vz^2+\vmax^2}{4}d$, thus the problem similarly reduces to the two sub-arrays $\nz,\cdots,\ntm$ and $\nt,\cdots,\nT$, which can be solved by induction. 
    \item Suppose $\dt \le \vz^2$ and $\nm{\ntm}^2 \le 36 \vz^2 d$ for all $t \in [T]$. Conditioned on $\nm{\ntm}^2 > \frac{\vz^2+\vmax^2}{2} d$ for all $t \in [T]$, by Lemma \ref{lem:ld:hnt_close_to_nt} we have that for $T = \exp(\gO(d))$, 
        \begin{equation*}
            \nm{\hnT - \nT} < \left( \sqrt{\frac{5\vz^2+3\vmax^2}{8}} - \sqrt{\frac{\vz^2+\vmax^2}{2}} \right) \sqrt{d}. 
        \end{equation*}
        By Lemma \ref{lem:ld:hnt_norm_lower_bound} we have that with probability at least $1-\exp(-d/300)$,
        \begin{equation*}
            \nm{\hnT}^2 \ge \frac{5\vz^2+3\vmax^2}{8} d.
        \end{equation*}
        Combining the two inequalities implies the desired bound
        \begin{equation*}
            \nm{\nT} \ge \nm{\hnT} - \nm{\hnT-\nT} > \sqrt{\frac{\vz^2+\vmax^2}{2} d}.
        \end{equation*} 
        Hence by induction we obtain $\nm{\nt}^2 > \frac{\vz^2+\vmax^2}{2} d$ for all $t \in [T]$ with probability at least 
        \begin{equation*}
            (1-(T-1)\exp(-d/300)) \cdot (1-\exp(-d/300)) \ge 1-T\exp(-d/300). 
        \end{equation*}
\end{itemize}
Therefore we complete the proof of Theorem \ref{thm:ld:gaussian_mixture}. 

\end{proof}

\subsection{Proof of Theorem \ref{thm:ald:gaussian_mixture}}
\label{app:prf:ald:gaussian-mixture}

\begin{proof}[Proof of Theorem \ref{thm:ald:gaussian_mixture}]
From \eqref{eq:ald:perturb} we note that the perturbed distribution is the convolution of the original distribution and a Gaussian random variable, i.e., for random variables $\rvz \sim p$ and $\rvt \sim \gN(\0_d, \mI_d)$, their sum $\rvz + \rvt \sim p_\nl$ follows the perturbed distribution with noise level $\nl$. Therefore, a perturbed (sub)Gaussian distribution remains (sub)Gaussian. We formalize this property in Proposition \ref{prop:perturb_property}.

\begin{proposition} \label{prop:perturb_property}
    Suppose the perturbed distribution of a $d$-dimensional probability distribution $p$ with noise level $\nl$ is $p_\nl$, then the mean of the perturbed distribution is the same as the original distribution, i.e., $\E_{\rvz \sim p_\nl} [\rvz] = \E_{\rvz \sim p} [\rvz]$. If $p = \gN(\rvmu, \Sig)$ is a Gaussian distribution, $p_\nl = \gN(\rvmu, \Sig + \nl^2 \mI_d)$ is also a Gaussian distribution. If $p$ is a sub-Gaussian distribution with parameter $\var^2$, $p_\nl$ is a sub-Gaussian distribution with parameter $(\var^2+\nl^2)$. 
\end{proposition}
\begin{proof}[Proof of Proposition \ref{prop:perturb_property}] \label{prf:prop:perturb_property}
    By the definition in \eqref{eq:ald:perturb}, we have
    \begin{equation*}
        p_\nl (\rvz) = \int p(\rvt) \gN(\rvz \mid \rvt, \nl^2 \mI_d) \dif \rvt = \int p(\rvt) \gN(\rvz - \rvt \mid \0_d, \nl^2 \mI_d) \dif \rvt. 
    \end{equation*}
    For random variables $\rvt \sim p$ and $\rvy \sim \gN(\0_d, \mI_d)$, their sum $\rvz = \rvt + \rvy \sim p_\nl$ follows the perturbed distribution with noise level $\nl$. Therefore, 
    \begin{equation*}
        \E_{\rvz \sim p_\nl} [\rvz] = \E_{(\rvt+\rvy) \sim p_\nl} [\rvt + \rvy] = \E_{\rvt \sim p} [\rvt] + \E_{\rvy \sim \gN(\0_d, \mI_d)} [\rvy] = \E_{\rvt \sim p} [\rvt]. 
    \end{equation*}
    If $\rvt \sim p = \gN(\rvmu, \Sig)$ follows a Gaussian distribution, we have $\rvz = \rvt + \rvy \sim p_\nl = \gN(\rvmu, \Sig + \nl^2 \mI_d)$. If $p$ is a sub-Gaussian distribution with parameter $\var^2$, we have $\rvz = \rvt + \rvy \sim p_\nl$ is a sub-Gaussian distribution with parameter $(\var^2+\nl^2)$. Hence we obtain Proposition \ref{prop:perturb_property}. 
\end{proof}

To establish Theorem \ref{thm:ald:gaussian_mixture}, we first note from Proposition \ref{prop:perturb_property} that perturbing a Gaussian distribution $\gN(\rvmu, \var^2 \mI_d)$ with noise level $\nl$ results in a Gaussian distribution $\gN(\rvmu, (\var^2+\nl^2) \mI_d)$. Therefore, for a Gaussian mixture $\P = \sum_{i=0}^k \wi \probi = \sum_{i=0}^k \wi \gN(\mui, \vi^2 \mI_d)$, the perturbed distribution of noise level $\nl$ is 
\begin{equation*}
    \P_\nl = \sum_{i=0}^k \wi \gN(\mui, (\vi^2 + \nl^2) \mI_d). 
\end{equation*}
Similar to the proof of Theorem \ref{thm:ld:gaussian_mixture}, we decompose 
\begin{equation*}
    \lt = \R \rt + \N \nt \text{, and } \epst = \R \ert + \N \ent, 
\end{equation*}
where $\R \in \sR^{d \times \r}$ an orthonormal basis of the vector space $\set{\mui}_{i \in [k]}$ and $\N \in \sR^{d \times \n}$ an orthonormal basis of the null space of $\set{\mui}_{i \in [k]}$. Now, we prove Theorem \ref{thm:ald:gaussian_mixture} by applying the techniques developed in Appendix \ref{app:prf:ld_gaussian_mixture} via substituting $\var^2$ with $\var^2+\nlt^2$ at time step $t$.

By Definition \ref{def:dist}, since $\nt$ is the projection onto the null space of $\set{\mui}_{i \in [k]}$, we have 
\begin{equation*}
    \duk{\lt} = \min_{\lambda_1, \cdots, \lambda_k} \nm{\lt - \sum_{i=1}^k \lambda_i \mui} = \nm{\nt}. 
\end{equation*}
We prove Theorem \ref{thm:ald:gaussian_mixture} by induction. Suppose the theorem holds for all $T$ values of $1, \cdots, T-1$. We consider the following 3 cases:
\begin{itemize}
    \item If there exists some $t \in [T]$ such that $\dt > \vz^2 + \nlt^2$, by Lemma \ref{lem:ld:large_dt} we know that with probability at least $1-\exp \left(- \left( \frac{\vz^2-\vmax^2}{8(\vz^2+\nlt^2)} \right)^2 \frac{d}{2} \right) \ge 1 - \exp (-d/32)$, we have $\nm{\nt}^2 \ge \frac{3(\vz^2+\nlt^2)+(\vmax^2+\nlt^2)}{4}d = \frac{3\vz^2+\vmax^2+4\nlt^2}{4}d$, thus the problem reduces to the two sub-arrays $\nz,\cdots,\ntm$ and $\nt,\cdots,\nT$, which can be solved by induction. 
    \item Suppose $\dt \le \vz^2 + \nlt^2$ for all $t \in [T]$. If there exists some $t \in [T]$ such that $\nm{\ntm}^2 > 36 (\vz^2 + \nltm^2) d \ge 36 (\vz^2 + \nlt^2) d$, by Lemma \ref{lem:ld:large_ntm_norm} we know that with probability at least 
    \begin{align*}
        &1-\exp \left( - \left( \log \frac{\vi^2+\nlt^2}{\vz^2+\nlt^2} - \frac{\vi^2+\nlt^2}{2(\vz^2 + \nlt^2) + \frac{\vz^2 + \nlt^2}{2(\vi^2 + \nlt^2)}} \right) \frac{d}{4} \right) - \frac{4}{\sqrt{2\pi d}} \exp \left( -\frac{d}{32} \right) \\ &\ge 1 - \exp(-0.01d),
    \end{align*} 
    we have $\nm{\nt}^2 \ge (\vz^2+\nlt^2) d > \frac{3\vz^2+\vmax^2 + 4\nlt^2}{4}d$, thus the problem similarly reduces to the two sub-arrays $\nz,\cdots,\ntm$ and $\nt,\cdots,\nT$, which can be solved by induction. 
    \item Suppose $\dt \le \vz^2 + \nlt^2$ and $\nm{\ntm}^2 \le 36 (\vz^2 + \nltm^2) d$ for all $t \in [T]$. Consider a surrogate sequence $\hnt$ such that $\hnz = \nz$ and for all $t \ge 1$, 
    \begin{equation*} 
        \hnt = \hntm - \frac{\dt}{2\vz^2 + 2\nlt^2} \hntm + \sqrt{\dt} \ent. 
    \end{equation*}
    Conditioned on $\nm{\ntm}^2 > \frac{\vz^2+\vmax^2 + 2\nltm^2}{2} d$ for all $t \in [T]$, by Lemma \ref{lem:ld:hnt_close_to_nt} we have that for $T \le \exp(d/150)$, 
    \begin{equation*}
        \nm{\hnT - \nT} < \left( \sqrt{\frac{5\vz^2+3\vmax^2 + 8\nlT^2}{8}} - \sqrt{\frac{\vz^2+\vmax^2 + 2\nlT^2}{2}} \right) \sqrt{d}. 
    \end{equation*}
    By Lemma \ref{lem:ld:hnt_norm_lower_bound} we have
    \begin{equation*}
        \nm{\hnT}^2 \ge \frac{5\vz^2+3\vmax^2 + 8\nlT^2}{8} d
    \end{equation*}
    with probability at least 
    \begin{align*}
        &1- \exp \left( -\left( \frac{\vz^2-\vmax^2}{8\vz^2 + 8\nlz^2} \right)^2 \frac{d}{2} \right) - \frac{4\sqrt{7}}{\sqrt{\pi d}} \exp \left( -\frac{(\vz^2-\vmax^2)^2 d}{32(\vz^2+\nlz^2)(3\vz^2+\vmax^2+4\nlz^2)} \right) \\ 
        &\ge 1 - \exp \left( -\frac{d}{512} \right) - \frac{4\sqrt{7}}{\sqrt{\pi d}} \exp \left( -\frac{d}{448} \right) \ge 1 - \exp \left( -\frac{d}{1500} \right).
    \end{align*}
    Combining the two inequalities implies the desired bound
    \begin{equation*}
        \nm{\nT} \ge \nm{\hnT} - \nm{\hnT-\nT} > \sqrt{\frac{\vz^2+\vmax^2 + 2\nlT^2}{2} d} \ge \sqrt{\frac{\vz^2+\vmax^2}{2} d}.
    \end{equation*} 
    Hence by induction we obtain $\nm{\nt}^2 > \frac{\vz^2+\vmax^2}{2} d$ for all $t \in \set{0} \cup [T]$ with probability at least 
    \begin{equation*}
        (1-(T-1)\exp(-d/1500)) \cdot (1-\exp(-d/1500)) \ge 1-T\exp(-d/1500). 
    \end{equation*}
\end{itemize}
Therefore we complete the proof of Theorem \ref{thm:ald:gaussian_mixture}. 
\end{proof}

\section{Iteration Complexity of Langevin Dynamics in sub-Gaussian Mixtures}
\label{app:prf:sub-Gaussian}

A probability distribution $p(\rvz)$ of dimension $d$ is defined as a sub-Gaussian distribution with parameter $\var^2$ if, given the mean vector $\rvmu := \E_{\rvz \sim p} [\rvz]$, the moment generating function (MGF) of $p$ satisfies the following inequality for every vector $\rvalpha \in \sR^d$:
\begin{equation} \label{eq:def:subgaussian}
    \E_{\rvz \sim p} \left[ \exp \left( \rvalpha^T (\rvz - \rvmu \right) \right] \le \exp \Bigl( \frac{\var^2 \nm{\rvalpha}_2^2}{2} \Bigr). 
\end{equation}

\begin{assumption} \label{asmpt:ld:subgaussian_mixture}
    Consider a data distribution $\P := \sum_{i=0}^k \wi \probi$ as a mixture of sub-Gaussian distributions, where $1 \le k = o(d)$ and $\wi > 0$ is a positive constant such that $\sum_{i=0}^k \wi = 1$. Suppose that $\probz = \gN(\muz, \vz^2 \mI_d)$ is Gaussian and for all $i \in [k]$, $\probi$ satisfies 
    \begin{enumerate}[label=\roman*., leftmargin=*, noitemsep]
        \vspace{-2.5mm}
        \item $\probi$ is a sub-Gaussian distribution of mean $\mui$ with parameter $\vi^2$, \label{asmpt:ld:pi:subgaussian}
        \vspace{0.9mm}
        \item $\probi$ is differentiable and $\nabla \probi(\mui) = \0_d$, \label{asmpt:ld:pi:differentiable}
        \vspace{0.9mm}
        \item the score function of $\probi$ is $\lipi$-Lipschitz such that $\lipi \le \frac{\clip}{\vi^2}$ for some constant $\clip > 0$, \label{asmpt:ld:pi:score_lipschitz}
        \vspace{0.9mm}
        \item $\vz^2 > \max \set{ 1, \frac{4(\clip^2+\cv\clip)}{\cv (1-\cv)} } \frac{\vmax^2}{1-\cv}$ for constant $\cv \in (0,1)$, where $\vmax := \max_{i \in [k]} \vi$, \label{asmpt:ld:var} 
        \vspace{0.9mm}
        \item $\nm{\mui - \muz}^2 \le \frac{(1-\cv)\vz^2 - \vi^2}{2(1-\cv)} \left( \log \frac{\cv \vi^2}{(\clip^2 + \cv \clip)\vz^2} - \frac{\vi^2}{2(1-\cv) \vz^2} + \frac{(1-\cv)\vz^2}{2\vi^2} \right) d$. \label{asmpt:ld:mui} 
    \end{enumerate}
\end{assumption}

The feasibility of Assumption \ref{asmpt:ld:subgaussian_mixture}.\ref{asmpt:ld:mui} is validated by Lemma \ref{lem:ld:subgaussian:validate_asmpt_ld_mui} in Appendix \ref{app:prf:ld_subgaussian_mixture}. With Assumption \ref{asmpt:ld:subgaussian_mixture}, we show the hardness of Langevin dynamics under sub-Gaussian distributions in Theorem \ref{thm:ld:subgaussian_mixture} and defer the proof to Appendix \ref{app:prf:ld_subgaussian_mixture}. 

\vspace{2mm}
\begin{theorem} \label{thm:ld:subgaussian_mixture}
    Consider a data distribution $\P$ satisfying Assumption \ref{asmpt:ld:subgaussian_mixture}. We initialize the sample $\lz$ such that $\duk{\lz}^2 \ge \left( \frac{3\vz^2}{4} + \frac{\vmax^2}{4(1-\cv)} \right) d$ and apply Langevin dynamics for $T$ steps, then 
    $$ \Pr \left( \duk{\lT}^2 \ge \left( \frac{\vz^2}{2} + \frac{\vmax^2}{2(1-\cv)} \right) d \right) \ge 1-T \cdot \exp \left(-\Omega(d) \right). $$  
\end{theorem}

Then, we slightly modify Assumption \ref{asmpt:ld:subgaussian_mixture} and extend our results to annealed Langevin dynamics (with bounded noise levels) under sub-Gaussian mixtures in Theorem \ref{thm:ald:subgaussian_mixture}. The proof of Theorem \ref{thm:ald:subgaussian_mixture} is deferred to Appendix \ref{app:prf:ald_subgaussian_mixture}.

\begin{assumption} \label{asmpt:ald:subgaussian_mixture}
    Consider a data distribution $\P := \sum_{i=0}^k \wi \probi$ as a mixture of sub-Gaussian distributions, where $1 \le k = o(d)$ and $\wi > 0$ is a positive constant such that $\sum_{i=0}^k \wi = 1$. Suppose that $\probz = \gN(\muz, \vz^2 \mI_d)$ is Gaussian and for all $i \in [k]$, $\probi$ satisfies 
    \begin{enumerate}[label=\roman*., leftmargin=*, noitemsep]
        \vspace{-2.5mm}
        \item $\probi$ is a sub-Gaussian distribution of mean $\mui$ with parameter $\vi^2$, \label{asmpt:ald:pi:subgaussian}
        \vspace{0.9mm}
        \item $\probi$ is differentiable and $\nabla \probi_{\nlt}(\mui) = \0_d$ for all $t \in \set{0}\cup[T]$, \label{asmpt:ald:pi:differentiable}
        \vspace{0.9mm}
        \item for all $t \in \set{0}\cup[T]$, the score function of $\probi_{\nlt}$ is $\lipit$-Lipschitz such that $\lipit \le \frac{\clip}{\vi^2+\nlt^2}$ for some constant $\clip > 0$, \label{asmpt:ald:pi:score_lipschitz}
        \vspace{0.9mm}
        \item $\vz^2 > \max \set{ 1, \frac{4(\clip^2+\cv\clip)}{\cv (1-\cv)} } \frac{\vmax^2 + \cnl^2}{1-\cv} - \cnl^2$ for constant $\cv \in (0,1)$, where $\vmax := \max_{i \in [k]} \vi$, \label{asmpt:ald:var} 
        \vspace{0.9mm}
        \item $\nm{\mui - \muz}^2 \le \frac{(1-\cv)\vz^2 - \vi^2 -\cv\cnl^2}{2(1-\cv)} \left( \log \frac{\cv (\vi^2 + \cnl^2)}{(\clip^2 + \cv \clip)(\vz^2 + \cnl^2)} - \frac{(\vi^2 + \cnl^2)}{2(1-\cv) (\vz^2 + \cnl^2)} + \frac{(1-\cv)(\vz^2 + \cnl^2)}{2(\vi^2 + \cnl^2)} \right) d$.\label{asmpt:ald:mui} 
    \end{enumerate}
\end{assumption}
\vspace{2mm}

\begin{theorem} \label{thm:ald:subgaussian_mixture}
    Consider a data distribution $\P$ satisfying Assumption \ref{asmpt:ald:subgaussian_mixture}. We initialize the sample $\lz$ such that $\duk{\lz}^2 \ge \left( \frac{3\vz^2+3\cnl^2}{4} + \frac{\vmax^2+\cnl^2}{4(1-\cv)} \right) d$ and apply annealed Langevin dynamics for $T$ steps with noise levels $\cnl \ge \nlz \ge \cdots \ge \nlT \ge 0$, then 
    $$ \Pr \left( \duk{\lT}^2 \ge \left( \frac{\vz^2}{2} + \frac{\vmax^2}{2(1-\cv)} \right) d \right) \ge 1-T \cdot \exp \left(-\Omega(d) \right). $$ 
\end{theorem}

We noticed that a central requirement of Theorems \ref{thm:ld:gaussian_mixture} and \ref{thm:ld:subgaussian_mixture} is that the initial sample $\lz$ must be far from the low-variance modes. In the following Theorem \ref{thm:ng:subgaussian}, we relax this constraint by considering low-variance modes $\P^{(1)}, \P^{(2)}, \cdots, \P^{(k)}$ with random mean vectors, as characterized by Assumption \ref{asmpt:ng:subgaussian}. The proof of Theorem \ref{thm:ng:subgaussian} is deferred to Appendix \ref{app:prf:ng-subgaussian}

\begin{assumption} \label{asmpt:ng:subgaussian}
    Consider a data distribution $\P := \sum_{i=0}^k \wi \probi$, where $k \ge 1$ and $\wi > 0$ are positive constants such that $\sum_{i=0}^k \wi = 1$. Suppose the density of mode 0 is lower bounded by $\probz(\lgv) \ge (2\pi \vz^2)^{-d/2} \exp \left( -\frac{(1+\cz)\nm{\lgv}^2}{2\vz^2} \right)$ for some constant $\vz$ and $\cz \ge 0$. In addition, assume $\log \probz(\lgv)$ is concave and $\nm{\nabla_\lgv \log \probz(\0_d)} = \exp(o(d))$, and its score function $\nabla_\lgv \log \probz(\lgv)$ is $\lipz$-Lipschitz. For all $i \in [k]$, suppose $\probi$ satisfies 
    \begin{enumerate}[label=\roman*., leftmargin=*, noitemsep]
        \item the mean $\mui$ of $\probi$ is i.i.d. uniform over an $\ell_2$ ball $\S$ centered at $\0_d$ of radius $\r$, \label{asmpt:ng:pi:random}
        \item $\probi$ is a sub-Gaussian distribution with parameter $\vi^2$. \label{asmpt:ng:pi:subgaussian}
        \vspace{0.9mm}
        \item $\probi$ is differentiable and $\nabla \probi(\mui) = \0_d$, \label{asmpt:ng:pi:differentiable}
        \vspace{0.9mm}
        \item the score function of $\probi$ is $\lipi$-Lipschitz such that $\lipi \le \frac{\clip}{\vi^2}$ for some constant $\clip > 0$, \label{asmpt:ng:pi:score_lipschitz}
        \vspace{0.9mm}
        \item $\vi$ satisfies $\left( \frac{1-\cv}{2\vi^2} - \frac{1+\cz}{\vz^2} \right) \left( \frac{\vz^2}{2} + \frac{\vi^2}{2(1-\cv)} \right) - \frac{1+\cz}{\vz^2} r^2 - \frac{1}{2} \log \frac{\cv\vi^2}{(\clip^2+\cv\clip)\vz^2} > 0$ for some constant $\cv \in (0,1)$
    \end{enumerate}
\end{assumption}

\begin{theorem} \label{thm:ng:subgaussian}
    Consider a data distribution $\P$ satisfying Assumption \ref{asmpt:ng:subgaussian}. For any initial sample $\lz$, we follow Langevin dynamics for $T$ steps with step size $\dt \le 4/\lipz$, then 
    $$ \Pr \left( \nm{\lT}^2 \ge \left( \frac{\vz^2}{2} + \frac{\vmax^2}{2(1-\cv)} \right) d \right) \ge 1-T \cdot \exp \left(-\Omega(d) \right). $$ 
\end{theorem}

\subsection{Proof of Theorem \ref{thm:ld:subgaussian_mixture}} \label{app:prf:ld_subgaussian_mixture}
\begin{proof}[Proof of Theorem \ref{thm:ld:subgaussian_mixture}]
The proof framework is similar to the proof of Theorem \ref{thm:ld:gaussian_mixture}. To begin with, we validate Assumption \ref{asmpt:ld:subgaussian_mixture}.\ref{asmpt:ld:mui} in the following lemma:
\begin{lemma} \label{lem:ld:subgaussian:validate_asmpt_ld_mui}
    For constants $\vz, \vi, \cv, \clip$ satisfying Assumptions \ref{asmpt:ld:subgaussian_mixture}.\ref{asmpt:ld:pi:score_lipschitz} and \ref{asmpt:ld:subgaussian_mixture}.\ref{asmpt:ld:var}, we have $\frac{(1-\cv)\vz^2 - \vi^2}{2(1-\cv)} > 0$ and $\log \frac{\cv \vi^2}{(\clip^2 + \cv \clip)\vz^2} - \frac{\vi^2}{2(1-\cv) \vz^2} + \frac{(1-\cv)\vz^2}{2\vi^2} > 0$ are both positive constants. 
\end{lemma}
\begin{proof}[Proof of Lemma \ref{lem:ld:subgaussian:validate_asmpt_ld_mui}]
    From Assumption \ref{asmpt:ld:subgaussian_mixture}.\ref{asmpt:ld:var} that $\vz^2 > \frac{\vmax^2}{1-\cv} \ge \frac{\vi^2}{1-\cv}$, we easily obtain $\frac{(1-\cv)\vz^2 - \vi^2}{2(1-\cv)} > 0$ is a positive constant. For the second property, let $f(z) := \log \frac{\cv \vi^2}{(\clip^2 + \cv \clip)z} - \frac{\vi^2}{2(1-\cv) z} + \frac{(1-\cv)z}{2\vi^2}$. For any $z > \frac{\vi^2}{1-\cv}$, the derivative of $f(z)$ satisfies 
    \begin{equation*}
        \frac{\dif}{\dif z} f(z) = -\frac{1}{z} + \frac{\vi^2}{2(1-\cv) z^2} + \frac{1-\cv}{2\vi^2} = \frac{\vi^2}{2(1-\cv)} \left( \frac{1-\cv}{\vi^2} - \frac{1}{z} \right)^2 > 0. 
    \end{equation*}
    Therefore, when $\frac{4(\clip^2+\cv\clip)}{\cv(1-\cv)} \le 1$, we have 
    \begin{equation*}
        f(\vz^2) > f\left( \frac{\vi^2}{1-\cv} \right) = \log \frac{\cv(1-\cv)}{\clip^2+\cv\clip} \ge \log 4 > 0. 
    \end{equation*}
    When $\frac{4(\clip^2+\cv\clip)}{\cv(1-\cv)} > 1$, we have 
    \begin{align*}
        f(\vz^2) &> f \left( \frac{4(\clip^2 + \cv\clip)}{\cv(1-\cv)} \frac{\vi^2}{1-\cv} \right) = 2 \log \frac{\cv(1-\cv)}{2(\clip^2+\cv\clip)} - \frac{\cv(1-\cv)}{8(\clip^2 + \cv\clip)} + \frac{2(\clip^2 + \cv\clip)}{\cv(1-\cv)} \\ 
        &\ge 2 -2\log2 - \frac{2(\clip^2+\cv\clip)}{\cv(1-\cv)} - \frac{\cv(1-\cv)}{8(\clip^2 + \cv\clip)} + \frac{2(\clip^2 + \cv\clip)}{\cv(1-\cv)} > 2 - 2\log 2 - \frac{1}{2} > 0.
    \end{align*}
    Thus we obtain Lemma \ref{lem:ld:subgaussian:validate_asmpt_ld_mui}. 
\end{proof}

Without loss of generality, we assume $\muz = \0_d$. Similar to the proof of Theorem \ref{thm:ld:gaussian_mixture}, we decompose 
\begin{equation*}
    \lt = \R \rt + \N \nt \text{, and } \epst = \R \ert + \N \ent, 
\end{equation*}
where $\R \in \sR^{d \times \r}$ an orthonormal basis of the vector space $\set{\mui}_{i \in [k]}$ and $\N \in \sR^{d \times \n}$ an orthonormal basis of the null space of $\set{\mui}_{i \in [k]}$. Then, conditioned on $\nm{\nz}^2 \ge \left( \frac{3 \vz^2}{4} + \frac{\vmax^2}{4(1-\cv)} \right) d$, we prove that $\nm{\nt}$ remains large with high probability.

Firstly, by Lemma \ref{lem:ld:large_dt}, if $\dt > \vz^2$, since $\vz^2 > \frac{\vmax^2}{1-\cv}$, we similarly have that $\nm{\nt}^2 \ge \left( \frac{3 \vz^2}{4} + \frac{\vmax^2}{4(1-\cv)} \right)d$ with probability at least $1-\exp(-\Omega(d))$ regardless of the previous state $\ltm$. We then consider the case when $\dt \le \vz^2$. Intuitively, we aim to prove that the score function is close to $-\frac{\lgv}{\vz^2}$ when $\nm{\nv}^2 \ge \left( \frac{\vz^2}{2} + \frac{\vmax^2}{2(1-\cv)} \right) d$. Towards this goal, we first show that $\probz(\lgv)$ is exponentially larger than $\probi(\lgv)$ for all $i \in [k]$ in the following lemma:

\begin{lemma} \label{lem:ld:subgaussian_exp_decay}
    Suppose $\P$ satisfies Assumption \ref{asmpt:ld:subgaussian_mixture}. Then for any $\nm{\nv}^2 \ge \left( \frac{\vz^2}{2} + \frac{\vmax^2}{2(1-\cv)} \right) d$, we have $\frac{\probi(\lgv)}{\probz(\lgv)} \le \exp(-\Omega(d))$ and $\frac{\nm{\nabla_\lgv \probi(\lgv)}}{\P(\lgv)} \le \exp(-\Omega(d))$ for all $i \in [k]$. 
\end{lemma}
\begin{proof}[Proof of Lemma \ref{lem:ld:subgaussian_exp_decay}]
    We first give an upper bound on the sub-Gaussian probability density. For any vector $\rvv \in \sR^d$, by considering some vector $\rvm \in \sR^d$, from Markov's inequality and the definition in \eqref{eq:def:subgaussian} we can bound 
    \begin{align*}
        \Pr_{\rvz \sim \probi} \left( \rvm^T (\rvz-\mui) \ge \rvm^T (\rvv-\mui) \right) &\le \frac{\E_{\rvz \sim \probi} \left[ \exp \left( \rvm^T (\rvz-\mui) \right) \right]}{\exp \left( \rvm^T (\rvv-\mui) \right)} \\
        &\le \exp \left( \frac{\vi^2 \nm{\rvm}^2}{2} - \rvm^T (\rvv-\mui) \right). 
    \end{align*}
    Upon optimizing the last term at $\rvm = \frac{\rvv-\mui}{\vi^2}$, we obtain 
    \begin{equation} \label{eq:ld:subgaussian_exp_decay:step1}
        \Pr_{\rvz \sim \probi} \left( (\rvv-\mui)^T (\rvv-\rvz) \le 0 \right) \le \exp \left( -\frac{\nm{\rvv-\mui}^2}{2 \vi^2} \right). 
    \end{equation}
    Denote $\Ball := \set{\rvz: (\rvv-\mui)^T (\rvv-\rvz) \le 0}$. To bound $\Pr_{\rvz \sim \probi} (\rvz \in \Ball)$, we first note that 
    \begin{align*}
        &\log \probi(\rvv) - \log \probi(\rvz) \\
        &= \int_0^1 \langle \rvv-\rvz, \nabla \log \probi(\rvv + \lambda(\rvz-\rvv)) \rangle \dif \lambda \\ 
        &= \langle \rvv-\rvz, \nabla \log \probi(\rvv) \rangle + \int_0^1 \langle \rvv-\rvz, \nabla \log \probi(\rvv + \lambda(\rvz-\rvv)) - \nabla \log \probi(\rvv) \rangle \dif \lambda \\ 
        &\le \nm{\rvv-\rvz} \nm{\nabla \log \probi(\rvv)} + \int_0^1 \nm{\rvv-\rvz} \nm{\nabla \log \probi(\rvv + \lambda(\rvz-\rvv)) - \nabla \log \probi(\rvv)} \dif \lambda \\ 
        &\le \nm{\rvv-\rvz} \cdot \lipi \nm{\rvv-\mui} + \int_0^1 \nm{\rvv-\rvz} \cdot \lipi \nm{\lambda (\rvz-\rvv)} \dif \lambda \\ 
        &\le \frac{\lipi\cv}{2\clip} \nm{\rvv-\mui}^2 + \left( \frac{\clip + \cv}{2\cv} \right) \lipi \nm{\rvv-\rvz}^2, 
    \end{align*}
    where the second last inequality follows from Assumption \ref{asmpt:ld:subgaussian_mixture}.\ref{asmpt:ld:pi:differentiable} that $\nabla \log \probi(\mui) = \0_d$ and Assumption \ref{asmpt:ld:subgaussian_mixture}.\ref{asmpt:ld:pi:score_lipschitz} that the score function $\nabla \log \probi$ is $\lipi$-Lipschitz. Therefore we obtain 
    \begin{align}
        &\Pr_{\rvz \sim \probi} (\rvz \in \Ball) = \int_{\rvz \in \Ball} \probi(\rvz) \dif \rvz \notag \\ 
        &\ge \int_{\rvz \in \Ball} \probi(\rvv) \exp \left( - \frac{\lipi\cv}{2\clip} \nm{\rvv-\mui}^2 - \frac{\clip + \cv}{2\cv} \lipi \nm{\rvv-\rvz}^2 \right) \dif \rvz \notag \\ 
        &= \probi(\rvv) \exp \left( - \frac{\lipi\cv}{2\clip} \nm{\rvv-\mui}^2 \right) \int_{\rvz \in \Ball} \exp \left( - \frac{\clip + \cv}{2\cv} \lipi \nm{\rvv-\rvz}^2 \right) \dif \rvz.  \label{eq:ld:subgaussian_exp_decay:step2}
    \end{align}
    By observing that $g: \Ball \to \set{\rvz: (\rvv-\mui)^T (\rvv-\rvz) \ge 0}$ with $g(\rvz) = 2\rvv-\rvz$ is a bijection such that $\nm{\rvv-\rvz} = \nm{\rvv-g(\rvz)}$ for any $\rvz \in \Ball$, we have 
    \begin{align}
        \int_{\rvz \in \Ball} \exp \left( - \frac{\clip + \cv}{2\cv} \lipi \nm{\rvv-\rvz}^2 \right) \dif \rvz &= \frac{1}{2} \int_{\rvz \in \sR^d} \exp \left( - \frac{\clip + \cv}{2\cv} \lipi \nm{\rvv-\rvz}^2 \right) \dif \rvz \notag \\ 
        &= \frac{1}{2} \left( \frac{2\pi\cv}{(\clip+\cv)\lipi} \right)^{\frac{d}{2}}. \label{eq:ld:subgaussian_exp_decay:step3}
    \end{align}
    Hence, by combining \eqref{eq:ld:subgaussian_exp_decay:step1}, \eqref{eq:ld:subgaussian_exp_decay:step2}, and \eqref{eq:ld:subgaussian_exp_decay:step3}, we obtain 
    \begin{align*}
        \exp \left( -\frac{\nm{\rvv-\mui}^2}{2 \vi^2} \right) &\ge \Pr_{\rvz \sim \probi} \left( (\rvv-\mui)^T (\rvv-\rvz) \le 0 \right) \\ 
        &\ge \probi(\rvv) \exp \left( - \frac{\lipi\cv}{2\clip} \nm{\rvv-\mui}^2 \right) \cdot \frac{1}{2} \left( \frac{2\pi\cv}{(\clip+\cv)\lipi} \right)^{\frac{d}{2}}. 
    \end{align*}
    By Assumption \ref{asmpt:ld:subgaussian_mixture}.\ref{asmpt:ld:pi:score_lipschitz} that $\lipi \le \frac{\clip}{\vi^2}$ we obtain the following bound on the probability density:
    \begin{equation} \label{eq:ld:subgaussian_bound_pdf}
        \probi(\rvv) \le 2 \left( \frac{2\pi\cv \vi^2}{(\clip+\cv)\clip} \right)^{-\frac{d}{2}} \exp \left( -\frac{1-\cv}{2\vi^2} \nm{\rvv-\mui}^2 \right). 
    \end{equation}
    Then we can bound the ratio of $\probi$ and $\probz$. For all $i \in [k]$, define $\rhoi(\lgv) := \frac{\probi(\lgv)}{\probz(\lgv)}$, then we have 
    \begin{align*}
        &\rhoi(\lgv) = \frac{\probi(\lgv)}{\probz(\lgv)} \le \frac{ 2 (2\pi\cv\vi^2/(\clip^2+\cv\clip))^{-d/2} \exp \left( -(1-\cv) \nm{\lgv-\mui}^2 / 2\vi^2 \right) }{(2\pi \vz^2)^{-d/2} \exp \left( - \nm{\lgv}^2 / 2\vz^2 \right)} \\ 
        &= 2 \left( \frac{(\clip^2+\cv\clip)\vz^2}{\cv\vi^2} \right)^{\frac{d}{2}} \exp \left( \frac{\nm{\lgv}^2}{2\vz^2} - \frac{(1-\cv)\nm{\lgv-\mui}^2}{2\vi^2} \right) \\ 
        &= 2 \left( \frac{(\clip^2+\cv\clip)\vz^2}{\cv\vi^2} \right)^{\frac{d}{2}} \exp \left( \left( \frac{1}{2\vz^2}-\frac{1-\cv}{2\vi^2} \right) \nm{\N\nv}^2 + \left( \frac{\nm{\R\rv}^2}{2\vz^2} - \frac{(1-\cv)\nm{\R\rv - \mui}^2}{2\vi^2} \right) \right) \\ 
        &= 2 \left( \frac{(\clip^2+\cv\clip)\vz^2}{\cv\vi^2} \right)^{\frac{d}{2}} \exp \left( \left( \frac{1}{2\vz^2}-\frac{1-\cv}{2\vi^2} \right) \nm{\nv}^2 + \left( \frac{\nm{\rv}^2}{2\vz^2} - \frac{(1-\cv)\nm{\rv - \R^T\mui}^2}{2\vi^2} \right) \right),
    \end{align*}
    where the last step follows from the definition that $\R \in \sR^{d \times \r}$ an orthogonal basis of the vector space $\set{\mui}_{i \in [k]}$ and $\N^T \N = \mI_\n$. Since $\vi^2 < (1-\cv)\vz^2$, the quadratic term $\frac{\nm{\rv}^2}{2\vz^2} - \frac{(1-\cv)\nm{\rv - \R^T\mui}^2}{2\vi^2}$ is maximized at $\rv = \frac{(1-\cv)\vz^2 \R^T \mui}{(1-\cv)\vz^2 - \vi^2}$. Therefore, we obtain
    \begin{align*}
        \frac{\nm{\rv}^2}{2\vz^2} - \frac{(1-\cv)\nm{\rv - \R^T\mui}^2}{2\vi^2} \le \frac{(1-\cv) \nm{\mui}^2}{2((1-\cv)\vz^2 - \vi^2)}. 
    \end{align*}
    Hence, for $\nm{\mui - \muz}^2 \le \frac{(1-\cv)\vz^2 - \vi^2}{2(1-\cv)} \left( \log \frac{\cv \vi^2}{(\clip^2 + \cv \clip)\vz^2} - \frac{\vi^2}{2(1-\cv) \vz^2} + \frac{(1-\cv)\vz^2}{2\vi^2} \right) d$ and $\nm{\nv}^2 \ge \left( \frac{\vz^2}{2} + \frac{\vmax^2}{2(1-\cv)} \right) d$, we have 
    \begin{align*}
        &\rhoi(\lgv) \le 2 \left( \frac{(\clip^2+\cv\clip)\vz^2}{\cv\vi^2} \right)^{\frac{d}{2}} \exp \left( \left( \frac{1}{2\vz^2}-\frac{1-\cv}{2\vi^2} \right) \nm{\nv}^2 + \frac{(1-\cv) \nm{\mui}^2}{2((1-\cv)\vz^2 - \vi^2)} \right) \\ 
        &\le 2 \left( \frac{(\clip^2+\cv\clip)\vz^2}{\cv\vi^2} \right)^{\frac{d}{2}} \exp \left( \left( \frac{1}{2\vz^2}-\frac{1-\cv}{2\vi^2} \right) \left( \frac{\vz^2}{2} + \frac{\vi^2}{2(1-\cv)} \right) d + \frac{(1-\cv) \nm{\mui}^2}{2((1-\cv)\vz^2 - \vi^2)} \right) \\ 
        &= 2 \exp \left( - \left( \log \frac{\cv \vi^2}{(\clip^2 + \cv \clip)\vz^2} - \frac{\vi^2}{2(1-\cv) \vz^2} + \frac{(1-\cv)\vz^2}{2\vi^2} \right) \frac{d}{2} + \frac{(1-\cv) \nm{\mui}^2}{2((1-\cv)\vz^2 - \vi^2)} \right) \\ 
        &\le 2 \exp \left( - \left( \log \frac{\cv \vi^2}{(\clip^2 + \cv \clip)\vz^2} - \frac{\vi^2}{2(1-\cv) \vz^2} + \frac{(1-\cv)\vz^2}{2\vi^2} \right) \frac{d}{4} \right). 
    \end{align*}
    From Lemma \ref{lem:ld:subgaussian:validate_asmpt_ld_mui}, we obtain $\rhoi(\lgv) \le \exp(-\Omega(d))$. 

    To show $\frac{\nm{\nabla_\lgv \probi(\lgv)}}{\P(\lgv)} \le \exp(-\Omega(d))$, from Assumptions \ref{asmpt:ld:subgaussian_mixture}.\ref{asmpt:ld:pi:differentiable} and \ref{asmpt:ld:subgaussian_mixture}.\ref{asmpt:ld:pi:score_lipschitz} we have 
    \begin{align*}
        \nm{ \frac{\nabla_\lgv \probi(\lgv)}{\probi(\lgv)} } &= \nm{ \frac{\nabla_\lgv \probi(\lgv)}{\probi(\lgv)} -  \frac{\nabla_\lgv \probi(\mui)}{\probi(\mui)}} = \nm{\nabla_\lgv \log \probi(\lgv) - \nabla_\lgv \log \probi(\mui)} \\ 
        &\le \lipi \nm{\lgv - \mui} \le \frac{\clip}{\vi^2} \nm{\lgv - \mui}. 
    \end{align*}
    Therefore, we can bound $\frac{\nm{\nabla_\lgv \probi(\lgv)}}{\P(\lgv)} \le \frac{\clip}{\vi^2} \rhoi(\lgv) \nm{\lgv - \mui}$. When $\nm{\lgv - \mui} = \exp(o(d))$ is small, by $\rhoi(\lgv) \le \exp(-\Omega(d))$ we directly have $\frac{\nm{\nabla_\lgv \probi(\lgv)}}{\P(\lgv)} \le \exp(-\Omega(d))$. When $\nm{\lgv - \mui} = \exp(\Omega(d))$ is exceedingly large, from \eqref{eq:ld:subgaussian_bound_pdf} we have 
    \begin{align*}
        \frac{\nm{\nabla_\lgv \probi(\lgv)}}{\P(\lgv)} &\le \frac{2\clip}{\vi^2} \left( \frac{(\clip^2+\cv\clip)\vz^2}{\cv\vi^2} \right)^{\frac{d}{2}} \exp \left( \frac{\nm{\lgv}^2}{2\vz^2} - \frac{(1-\cv)\nm{\lgv-\mui}^2}{2\vi^2} \right) \nm{\lgv - \mui}. 
    \end{align*}
    Since $\vz^2 > \frac{\vi^2}{1-\cv}$, when $\nm{\lgv - \mui} = \exp(\Omega(d)) \gg \nm{\mui}$ we have 
    \begin{equation*}
        \exp \left( \frac{\nm{\lgv}^2}{2\vz^2} - \frac{(1-\cv)\nm{\lgv-\mui}^2}{2\vi^2} \right) = \exp( -\Omega(\nm{\lgv-\mui}^2) ). 
    \end{equation*}
    Therefore $\frac{\nm{\nabla_\lgv \probi(\lgv)}}{\P(\lgv)} \le \exp(-\Omega(d))$. Thus we complete the proof of Lemma \ref{lem:ld:subgaussian_exp_decay}. 
\end{proof}

Similar to Lemma \ref{lem:ld:large_ntm_norm}, the following lemma proves that when the previous state $\ntm$ is far from a mode, a single step of Langevin dynamics with bounded step size is not enough to find the mode. 
\begin{lemma} \label{lem:ld:subgaussian:large_ntm_norm}
    Suppose $\dt \le \vz^2$ and $\nm{\ntm}^2 > 36 \vz^2 d$, then we have $\nm{\nt}^2 \ge \vz^2 d$ with probability at least $1-\exp(-\Omega(d))$. 
\end{lemma}
\begin{proof}[Proof of Lemma \ref{lem:ld:subgaussian:large_ntm_norm}]
    For simplicity, denote $\rvv := \ntm + \frac{\dt}{2} \N^T \nabla_\lgv \log \P(\ltm)$. Since $\P = \sum_{i=0}^k \wi \probi$ and $\probz = \gN(\muz, \vz^2 \mI_d)$, the score function can be written as 
    \begin{align}
        \nabla_\lgv \log \P(\lgv) &= \frac{\nabla_\lgv \P(\lgv)}{\P(\lgv)} = \frac{\nabla_\lgv \wz \probz(\lgv)}{\P(\lgv)} + \sum_{i \in [k]} \frac{\nabla_\lgv \wi \probi(\lgv)}{\P(\lgv)} \notag \\ 
        &= - \frac{\wz \probz(\lgv)}{\P(\lgv)} \cdot \frac{\lgv}{\vz^2} + \sum_{i \in [k]} \frac{\wi \nabla_\lgv \probi(\lgv)}{\P(\lgv)} \notag \\ 
        &= -\frac{\lgv}{\vz^2} + \sum_{i \in [k]} \frac{\wi \probi(\lgv)}{\P(\lgv)} \cdot \frac{\lgv}{\vz^2} + \sum_{i \in [k]} \frac{\wi \nabla_\lgv \probi(\lgv)}{\P(\lgv)}. \label{eq:ld:subgaussian:score_rewrite}
    \end{align}
    For $\nm{\ntm}^2 > 36 \vz^2 d$ by Lemma \ref{lem:ld:subgaussian_exp_decay} we have  $\frac{\nm{\nabla_\lgv \probi(\ltm)}}{\P(\ltm)} \le \exp(-\Omega(d))$. Since $\dt \le \vz^2$, we can bound the norm of $\rvv$ by 
    \begin{align*}
        \nm{\rvv} &= \nm{\ntm + \frac{\dt}{2} \N^T \nabla_\lgv \log \P(\ltm)} \\ 
        &= \nm{\ntm - \frac{\dt}{2\vz^2} \ntm + \sum_{i \in [k]} \frac{\wi\dt}{2\vz^2} \frac{\probi(\ltm)}{\P(\ltm)} \ntm + \sum_{i \in [k]} \frac{\wi\dt}{2} \frac{\N^T \nabla_\lgv \probi(\ltm)}{\P(\ltm)}} \\ 
        &\ge \nm{\left( 1 - \frac{\dt}{2\vz^2} + \sum_{i \in [k]} \frac{\wi\dt}{2\vz^2} \frac{\probi(\ltm)}{\P(\ltm)} \right) \ntm} - \sum_{i \in [k]} \frac{\wi\dt}{2} \frac{\nm{\nabla_\lgv \probi(\ltm)}}{\P(\ltm)} \\ 
        &\ge \frac{1}{2} \nm{\ntm} - \sum_{i \in [k]} \frac{\wi\dt}{2} \exp(-\Omega(d)) \\ 
        &> 2\vz \sqrt{d}. 
    \end{align*}
    On the other hand, from $\ent \sim \gN(\0_\n, \mI_\n)$ we know $\frac{\langle \rvv , \ent \rangle}{\nm{\rvv}} \sim \gN(0, 1)$ for any fixed $\rvv \neq \0_\n$, hence by Lemma \ref{lem:one_dim_gaussian_tail_bound} we have 
    \begin{equation*} 
        \Pr \left( \frac{\langle \rvv , \ent \rangle}{\nm{\rvv}} \ge \frac{\sqrt{d}}{4} \right) = \Pr \left( \frac{\langle \rvv , \ent \rangle}{\nm{\rvv}} \le -\frac{\sqrt{d}}{4} \right) \le \frac{4}{\sqrt{2\pi d}} \exp \left( -\frac{d}{32} \right)
    \end{equation*}
    Combining the above inequalities gives 
    \begin{align*}
        \nm{\nt}^2 &= \nm{\rvv + \sqrt{\dt} \ent}^2 \ge \nm{\rvv}^2 - 2\vz \lvert\langle \rvv, \ent \rangle\rvert \ge \nm{\rvv}^2 - \frac{\vz \sqrt{d}}{2} \nm{\rvv} > \vz^2 d
    \end{align*}
    with probability at least $1-\frac{8}{\sqrt{2\pi d}} \exp \left( -\frac{d}{32} \right) = 1 - \exp(-\Omega(d))$. This proves Lemma \ref{lem:ld:subgaussian:large_ntm_norm}. 
\end{proof}

When $\nm{\ntm}^2 \le 36\vz^2 d$, similar to Theorem \ref{thm:ld:gaussian_mixture}, we consider a surrogate recursion $\hnt$ such that $\hnz = \nz$ and for all $t \ge 1$, 
\begin{equation} \label{eq:ld:def:subgaussian_hnt}
    \hnt = \hntm - \frac{\dt}{2\vz^2} \hntm + \sqrt{\dt} \ent. 
\end{equation}
The following Lemma shows that $\hnt$ is sufficiently close to the original recursion $\nt$. 

\begin{lemma} \label{lem:ld:subgaussian_hnt_close_to_nt}
    For any $t \ge 1$, given that for all $j \in [t]$, $\dj \le \vz^2$ and $\left( \frac{\vz^2}{2} + \frac{\vmax^2}{2(1-\cv)} \right) d \le \nm{\njm}^2 \le 36\vz^2 d$, if $\mui$ satisfies Assumption \ref{asmpt:ld:subgaussian_mixture}.\ref{asmpt:ld:mui} for all $i \in [k]$, we have $\nm{\hnt - \nt} \le \frac{t}{\exp(\Omega(d))} \sqrt{d}$. 
\end{lemma}
\begin{proof}[Proof of Lemma \ref{lem:ld:subgaussian_hnt_close_to_nt}]
    By \eqref{eq:ld:subgaussian:score_rewrite} we have that for all $j \in [t]$,
    \begin{align*}
        \nm{\hnj - \nj} &= \nm{\hnjm - \njm - \frac{\dj}{2\vz^2} \hnjm - \frac{\dj}{2} \N^T \nabla_\lgv \log \P (\ljm)} \\ 
        &= \nm{ \hnjm - \njm - \sum_{i \in [k]} \frac{\wi \probi(\ljm)}{\vz^2\P(\ljm)} \njm - \sum_{i \in [k]} \frac{\wi \N^T \nabla_\lgv \probi(\ljm)}{\P(\ljm)} } \\ 
        &\le \nm{\hnjm-\njm} + \sum_{i \in [k]} \frac{\wi \probi(\ljm)}{\vz^2\P(\ljm)} \nm{\njm} + \sum_{i \in [k]} \frac{\wi \nm{\nabla_\lgv \probi(\ljm)}}{\P(\ljm)} . 
    \end{align*}
    By Lemma \ref{lem:ld:subgaussian_exp_decay}, we have $\frac{\probi(\ljm)}{\probz(\ljm)} \le \exp(-\Omega(d))$ and $\frac{\nm{\nabla_\lgv \probi(\ljm)}}{\P(\ljm)} \le \exp(-\Omega(d))$ for all $i \in [k]$, hence from $\nm{\njm} \le 6\vz \sqrt{d}$ we obtain a recursive bound
    \begin{equation*}
        \nm{\hnj - \nj} \le \nm{\hnjm-\njm} + \frac{1}{\exp(\Omega(d))} \sqrt{d}. 
    \end{equation*}
    Finally, by $\hnz = \nz$, we have 
    \begin{equation*}
        \nm{\hnt-\nt} = \sum_{j\in[t]} \left( \nm{\hnj - \nj} - \nm{\hnjm-\njm} \right) \le \frac{t}{\exp(\Omega(d))} \sqrt{d}.
    \end{equation*}
    Hence we obtain Lemma \ref{lem:ld:subgaussian_hnt_close_to_nt}. 
\end{proof}

Armed with the above lemmas, we are now ready to establish Theorem \ref{thm:ld:subgaussian_mixture} by induction. Please note that we recycle some lemmas from the proof of Theorem \ref{thm:ld:gaussian_mixture} by substituting $\vmax^2$ with $\frac{\vmax^2}{1-\cv}$. Suppose the theorem holds for all $T$ values of $1, \cdots, T-1$. We consider the following 3 cases:
\begin{itemize}
    \item If there exists some $t \in [T]$ such that $\dt > \vz^2$, by Lemma \ref{lem:ld:large_dt} we know that with probability at least $1-\exp(-\Omega(d))$, we have $\nm{\nt}^2 \ge \left( \frac{3\vz^2}{4} + \frac{\vmax^2}{4(1-\cv)} \right) d$, thus the problem reduces to the two sub-arrays $\nz,\cdots,\ntm$ and $\nt,\cdots,\nT$, which can be solved by induction. 
    \item Suppose $\dt \le \vz^2$ for all $t \in [T]$. If there exists some $t \in [T]$ such that $\nm{\ntm}^2 > 36 \vz^2 d$, by Lemma \ref{lem:ld:subgaussian:large_ntm_norm} we know that with probability at least $1-\exp(-\Omega(d))$, we have $\nm{\nt}^2 \ge \vz^2 d > \left( \frac{3\vz^2}{4} + \frac{\vmax^2}{4(1-\cv)} \right) d$, thus the problem similarly reduces to the two sub-arrays $\nz,\cdots,\ntm$ and $\nt,\cdots,\nT$, which can be solved by induction. 
    \item Suppose $\dt \le \vz^2$ and $\nm{\ntm}^2 \le 36 \vz^2 d$ for all $t \in [T]$. Conditioned on $\nm{\ntm}^2 > \left( \frac{\vz^2}{2} + \frac{\vmax^2}{2(1-\cv)} \right) d$ for all $t \in [T]$, by Lemma \ref{lem:ld:subgaussian_hnt_close_to_nt} we have that for $T = \exp(\gO(d))$, 
        \begin{equation*}
            \nm{\hnT - \nT} < \left( \sqrt{ \frac{5\vz^2}{8} + \frac{3\vmax^2}{8(1-\cv)} } - \sqrt{\frac{\vz^2}{2} + \frac{\vmax^2}{2(1-\cv)}} \right) \sqrt{d}. 
        \end{equation*}
        By Lemma \ref{lem:ld:hnt_norm_lower_bound} we have that with probability at least $1-\exp(-\Omega(d))$,
        \begin{equation*}
            \nm{\hnT}^2 \ge \left( \frac{5\vz^2}{8} + \frac{3\vmax^2}{8(1-\cv)} \right) d.
        \end{equation*}
        Combining the two inequalities implies the desired bound
        \begin{equation*}
            \nm{\nT} \ge \nm{\hnT} - \nm{\hnT-\nT} > \sqrt{ \left( \frac{\vz^2}{2} + \frac{\vmax^2}{2(1-\cv)} \right) d}.
        \end{equation*} 
        Hence by induction we obtain $\nm{\nt}^2 > \left( \frac{\vz^2}{2} + \frac{\vmax^2}{2(1-\cv)} \right) d$ for all $t \in [T]$ with probability at least 
        \begin{equation*}
            (1-(T-1)\exp(-\Omega(d))) \cdot (1-\exp(-\Omega(d))) \ge 1-T\exp(-\Omega(d)). 
        \end{equation*}
\end{itemize}
Therefore we complete the proof of Theorem \ref{thm:ld:subgaussian_mixture}. 
\end{proof}

\subsection{Proof of Theorem \ref{thm:ald:subgaussian_mixture}} \label{app:prf:ald_subgaussian_mixture}
\begin{proof}[Proof of Theorem \ref{thm:ald:subgaussian_mixture}]
The feasibility of Assumption \ref{asmpt:ald:subgaussian_mixture}.\ref{asmpt:ald:mui} can be validated by substituting $\var^2$ in Lemma \ref{lem:ld:subgaussian:validate_asmpt_ld_mui} with $\var^2 + \cnl^2$.  To establish Theorem \ref{thm:ald:subgaussian_mixture}, we first note from Proposition \ref{prop:perturb_property} that for a sub-Gaussian mixture $\P = \sum_{i=0}^k \wi \probi$, the perturbed distribution of noise level $\nl$ is $\P_\nl = \sum_{i=0}^k \wi \probi_{\nl}$, where $\probz=\gN(\muz, (\vi^2+\nl^2)\mI_d)$ and $\probi$ is a sub-Gaussian distribution with mean $\mui$ and sub-Gaussian parameter $(\vi^2+\nl^2)$. Similar to the proof of Theorem \ref{thm:ld:gaussian_mixture}, we decompose 
\begin{equation*}
    \lt = \R \rt + \N \nt \text{, and } \epst = \R \ert + \N \ent, 
\end{equation*}
where $\R \in \sR^{d \times \r}$ an orthonormal basis of the vector space $\set{\mui}_{i \in [k]}$ and $\N \in \sR^{d \times \n}$ an orthonormal basis of the null space of $\set{\mui}_{i \in [k]}$. Now, we prove Theorem \ref{thm:ald:subgaussian_mixture} by applying the techniques developed in Appendix \ref{app:prf:ld_gaussian_mixture} and \ref{app:prf:ld_subgaussian_mixture} via substituting $\var^2$ and $\frac{\var^2}{1-\cv}$ with $\frac{\var^2+\nlt^2}{1-\cv}$ at time step $t$. Note that for all $t \in \set{0}\cup[T]$, Assumption \ref{asmpt:ald:subgaussian_mixture}.\ref{asmpt:ald:var} implies $\vz^2 + \nlt^2 > \max \set{ 1, \frac{4(\clip^2+\cv\clip)}{\cv (1-\cv)} } \frac{\vmax^2 + \nlt^2}{1-\cv}$ because $\cnl \ge \nlt$.

Then, with the assumption that the initialization satisfies $\nm{\nz}^2 \ge \left( \frac{3 (\vz^2 + \nlz^2)}{4} + \frac{\vmax^2 + \nlz^2}{4(1-\cv)} \right) d$, we prove Theorem \ref{thm:ald:subgaussian_mixture} via showing that 
$$\Pr \left( \forall t \in [T], \,  \nm{\nt}^2 \ge \left( \frac{\vz^2+\nlt^2}{2} + \frac{\vmax^2+\nlt^2}{2(1-\cv)} \right) d \right) \ge 1-T \cdot \exp \left(-\Omega(d) \right) . $$ 
Suppose the theorem holds for all $T$ values of $1, \cdots, T-1$. We consider the following 3 cases:
\begin{itemize}
    \item If there exists some $t \in [T]$ such that $\dt > \vz^2 + \nlt^2$, by Lemma \ref{lem:ld:large_dt} we know that with probability at least $1-\exp(-\Omega(d))$, we have $\nm{\nt}^2 \ge \left( \frac{3 (\vz^2+\nlt^2)}{4} + \frac{\vmax^2+\nlt^2}{4(1-\cv)} \right) d$, thus the problem reduces to the two sub-arrays $\nz,\cdots,\ntm$ and $\nt,\cdots,\nT$, which can be solved by induction. 
    \item Suppose $\dt \le \vz^2 + \nlt^2$ for all $t \in [T]$. If there exists some $t \in [T]$ such that $\nm{\ntm}^2 > 36 (\vz^2 + \nltm^2) d \ge 36 (\vz^2 + \nlt^2) d$, by Lemma \ref{lem:ld:subgaussian:large_ntm_norm} we know that with probability at least $1-\exp(-\Omega(d))$, we have $\nm{\nt}^2 \ge (\vz^2+\nlt^2) d > \left( \frac{3 (\vz^2+\nlt^2)}{4} + \frac{\vmax^2+\nlt^2}{4(1-\cv)} \right) d$, thus the problem similarly reduces to the two sub-arrays $\nz,\cdots,\ntm$ and $\nt,\cdots,\nT$, which can be solved by induction. 
    \item Suppose $\dt \le \vz^2 + \nlt^2$ and $\nm{\ntm}^2 \le 36 (\vz^2 + \nltm^2) d$ for all $t \in [T]$. Consider a surrogate sequence $\hnt$ such that $\hnz = \nz$ and for all $t \ge 1$, 
    \begin{equation*} 
        \hnt = \hntm - \frac{\dt}{2\vz^2 + 2\nlt^2} \hntm + \sqrt{\dt} \ent. 
    \end{equation*}
    Since $\vz > \vi$ and $\cnl \ge \nlt$ for all $t \in \set{0} \cup [T]$, we have $\frac{\vi^2+\cnl^2}{\vz^2+\cnl^2} > \frac{\vi^2+\nlt^2}{\vz^2+\nlt^2}$. Notice that for function $f(z) = \log z - \frac{z}{2} + \frac{1}{2z}$, we have $\frac{\dif}{\dif z} f(z) = \frac{1}{z} - \frac{1}{2} - \frac{1}{2z^2} = -\frac{1}{2} \left( \frac{1}{z} - 1 \right)^2 \le 0$.
    
    Thus, by Assumption \ref{asmpt:ald:subgaussian_mixture}.\ref{asmpt:ald:mui} we have that for all $t \in [T]$,
    \begin{align*}
        \nm{\mui - \muz}^2 \le \frac{(1-\cv)\vz^2 - \vi^2 -\cv\cnl^2}{2(1-\cv)} &\left( \log \frac{\cv (\vi^2 + \cnl^2)}{(\clip^2 + \cv \clip)(\vz^2 + \cnl^2)} \right. \\ 
        &- \left. \frac{(\vi^2 + \cnl^2)}{2(1-\cv) (\vz^2 + \cnl^2)} + \frac{(1-\cv)(\vz^2 + \cnl^2)}{2(\vi^2 + \cnl^2)} \right) d \\ 
        \le \frac{(1-\cv)\vz^2 - \vi^2 -\cv\nlt^2}{2(1-\cv)} &\left( \log \frac{\cv (\vi^2 + \nlt^2)}{(\clip^2 + \cv \clip)(\vz^2 + \nlt^2)} \right. \\ 
        &- \left. \frac{(\vi^2 + \nlt^2)}{2(1-\cv) (\vz^2 + \nlt^2)} + \frac{(1-\cv)(\vz^2 + \nlt^2)}{2(\vi^2 + \nlt^2)} \right) d 
    \end{align*}
    Conditioned on $\nm{\ntm}^2 > \left( \frac{\vz^2 + \nltm^2}{2} + \frac{\vmax^2 + \nltm^2}{2(1-\cv)} \right) d$ for all $t \in [T]$, by Lemma \ref{lem:ld:subgaussian_hnt_close_to_nt} we have that for $T = \exp(\gO(d))$, 
    \begin{equation*}
        \nm{\hnT - \nT} < \left( \sqrt{\frac{5 (\vz^2 + \nlT^2)}{8} + \frac{3(\vmax^2 + \nlT^2)}{8(1-\cv)}} - \sqrt{\frac{\vz^2 + \nlT^2}{2} + \frac{\vmax^2 + \nlT^2}{2(1-\cv)}} \right) \sqrt{d}. 
    \end{equation*}
    By Lemma \ref{lem:ld:hnt_norm_lower_bound} we have that with probability at least $1-\exp(-\Omega(d))$,
    \begin{equation*}
        \nm{\hnT}^2 \ge \left( \frac{5 (\vz^2 + \nlT^2)}{8} + \frac{3(\vmax^2 + \nlT^2)}{8(1-\cv)} \right) d.
    \end{equation*}
    Combining the two inequalities implies the desired bound
    \begin{equation*}
        \nm{\nT} \ge \nm{\hnT} - \nm{\hnT-\nT} > \sqrt{\left( \frac{\vz^2 + \nlT^2}{2} + \frac{\vmax^2 + \nlT^2}{2(1-\cv)} \right) d}.
    \end{equation*} 
    Hence by induction we obtain $\nm{\nt}^2 > \left( \frac{\vz^2 + \nlt^2}{2} + \frac{\vmax^2 + \nlt^2}{2(1-\cv)} \right) d$ for all $t \in [T]$ with probability at least 
    \begin{equation*}
        (1-(T-1)\exp(-\Omega(d))) \cdot (1-\exp(-\Omega(d))) \ge 1-T\exp(-\Omega(d)). 
    \end{equation*}
\end{itemize}
Therefore we complete the proof of Theorem \ref{thm:ald:subgaussian_mixture}.
\end{proof}

\subsection{Proof of Theorem \ref{thm:ng:subgaussian}}
\label{app:prf:ng-subgaussian}
\label{app:non-gauss-pz}
\begin{proof}[Proof of Theorem \ref{thm:ng:subgaussian}]

The following Lemma \ref{lem:ng:subgaussian_exp_decay} gives an upper bound on the probability density $\probi$.
    
\begin{lemma} \label{lem:ng:subgaussian_exp_decay}
    Suppose $\P$ satisfies Assumption \ref{asmpt:ng:subgaussian}. Then for any $\lgv$ such that $\nm{\lgv-\mui}^2 \ge \left( \frac{\vz^2}{2} + \frac{\vi^2}{2(1-\cv)} \right) d$, we have $\frac{\probi(\lgv)}{\probz(\lgv)} \le \exp(-\Omega(d))$ and $\frac{\nm{\nabla_\lgv \probi(\lgv)}}{\P(\lgv)} \le \exp(-\Omega(d))$. 
\end{lemma}
\begin{proof}[Proof of Lemma \ref{lem:ng:subgaussian_exp_decay}]
    Similar to the proof of Lemma \ref{lem:ld:subgaussian_exp_decay}, we first give an upper bound on the sub-Gaussian probability density. For any vector $\rvv \in \sR^d$, by considering some vector $\rvm \in \sR^d$, from Markov's inequality and the definition in \eqref{eq:def:subgaussian} we can bound 
    \begin{align*}
        \Pr_{\rvz \sim \probi} \left( \rvm^T (\rvz-\mui) \ge \rvm^T (\rvv-\mui) \right) &\le \frac{\E_{\rvz \sim \probi} \left[ \exp \left( \rvm^T (\rvz-\mui) \right) \right]}{\exp \left( \rvm^T (\rvv-\mui) \right)} \\
        &\le \exp \left( \frac{\vi^2 \nm{\rvm}^2}{2} - \rvm^T (\rvv-\mui) \right). 
    \end{align*}
    Upon optimizing the last term at $\rvm = \frac{\rvv-\mui}{\vi^2}$, we obtain 
    \begin{equation} \label{eq:ng:subgaussian_exp_decay:step1}
        \Pr_{\rvz \sim \probi} \left( (\rvv-\mui)^T (\rvv-\rvz) \le 0 \right) \le \exp \left( -\frac{\nm{\rvv-\mui}^2}{2 \vi^2} \right). 
    \end{equation}
    Denote $\Ball := \set{\rvz: (\rvv-\mui)^T (\rvv-\rvz) \le 0}$. To bound $\Pr_{\rvz \sim \probi} (\rvz \in \Ball)$, we first note that 
    \begin{align*}
        &\log \probi(\rvv) - \log \probi(\rvz) \\
        &= \int_0^1 \langle \rvv-\rvz, \nabla \log \probi(\rvv + \lambda(\rvz-\rvv)) \rangle \dif \lambda \\ 
        &= \langle \rvv-\rvz, \nabla \log \probi(\rvv) \rangle + \int_0^1 \langle \rvv-\rvz, \nabla \log \probi(\rvv + \lambda(\rvz-\rvv)) - \nabla \log \probi(\rvv) \rangle \dif \lambda \\ 
        &\le \nm{\rvv-\rvz} \nm{\nabla \log \probi(\rvv)} + \int_0^1 \nm{\rvv-\rvz} \nm{\nabla \log \probi(\rvv + \lambda(\rvz-\rvv)) - \nabla \log \probi(\rvv)} \dif \lambda \\ 
        &\le \nm{\rvv-\rvz} \cdot \lipi \nm{\rvv-\mui} + \int_0^1 \nm{\rvv-\rvz} \cdot \lipi \nm{\lambda (\rvz-\rvv)} \dif \lambda \\ 
        &\le \frac{\lipi\cv}{2\clip} \nm{\rvv-\mui}^2 + \left( \frac{\clip + \cv}{2\cv} \right) \lipi \nm{\rvv-\rvz}^2, 
    \end{align*}
    where the second last inequality follows from Assumption \ref{asmpt:ng:subgaussian}.\ref{asmpt:ng:pi:differentiable} that $\nabla \log \probi(\mui) = \0_d$ and Assumption \ref{asmpt:ng:subgaussian}.\ref{asmpt:ng:pi:score_lipschitz} that the score function $\nabla \log \probi$ is $\lipi$-Lipschitz. Therefore we obtain 
    \begin{align}
        &\Pr_{\rvz \sim \probi} (\rvz \in \Ball) = \int_{\rvz \in \Ball} \probi(\rvz) \dif \rvz \notag \\ 
        &\ge \int_{\rvz \in \Ball} \probi(\rvv) \exp \left( - \frac{\lipi\cv}{2\clip} \nm{\rvv-\mui}^2 - \frac{\clip + \cv}{2\cv} \lipi \nm{\rvv-\rvz}^2 \right) \dif \rvz \notag \\ 
        &= \probi(\rvv) \exp \left( - \frac{\lipi\cv}{2\clip} \nm{\rvv-\mui}^2 \right) \int_{\rvz \in \Ball} \exp \left( - \frac{\clip + \cv}{2\cv} \lipi \nm{\rvv-\rvz}^2 \right) \dif \rvz.  \label{eq:ng:subgaussian_exp_decay:step2}
    \end{align}
    By observing that $g: \Ball \to \set{\rvz: (\rvv-\mui)^T (\rvv-\rvz) \ge 0}$ with $g(\rvz) = 2\rvv-\rvz$ is a bijection such that $\nm{\rvv-\rvz} = \nm{\rvv-g(\rvz)}$ for any $\rvz \in \Ball$, we have 
    \begin{align}
        \int_{\rvz \in \Ball} \exp \left( - \frac{\clip + \cv}{2\cv} \lipi \nm{\rvv-\rvz}^2 \right) \dif \rvz &= \frac{1}{2} \int_{\rvz \in \sR^d} \exp \left( - \frac{\clip + \cv}{2\cv} \lipi \nm{\rvv-\rvz}^2 \right) \dif \rvz \notag \\ 
        &= \frac{1}{2} \left( \frac{2\pi\cv}{(\clip+\cv)\lipi} \right)^{\frac{d}{2}}. \label{eq:ng:subgaussian_exp_decay:step3}
    \end{align}
    Hence, by combining \eqref{eq:ng:subgaussian_exp_decay:step1}, \eqref{eq:ng:subgaussian_exp_decay:step2}, and \eqref{eq:ng:subgaussian_exp_decay:step3}, we obtain 
    \begin{align*}
        \exp \left( -\frac{\nm{\rvv-\mui}^2}{2 \vi^2} \right) &\ge \Pr_{\rvz \sim \probi} \left( (\rvv-\mui)^T (\rvv-\rvz) \le 0 \right) \\ 
        &\ge \probi(\rvv) \exp \left( - \frac{\lipi\cv}{2\clip} \nm{\rvv-\mui}^2 \right) \cdot \frac{1}{2} \left( \frac{2\pi\cv}{(\clip+\cv)\lipi} \right)^{\frac{d}{2}}. 
    \end{align*}
    By Assumption \ref{asmpt:ng:subgaussian}.\ref{asmpt:ng:pi:score_lipschitz} that $\lipi \le \frac{\clip}{\vi^2}$ we obtain the following bound on the probability density:
    \begin{equation} \label{eq:ng:subgaussian_bound_pdf}
        \probi(\rvv) \le 2 \left( \frac{2\pi\cv \vi^2}{(\clip+\cv)\clip} \right)^{-\frac{d}{2}} \exp \left( -\frac{1-\cv}{2\vi^2} \nm{\rvv-\mui}^2 \right). 
    \end{equation}
    Then we can bound the ratio of $\probi$ and $\probz$. For all $i \in [k]$, we have 
    \begin{align}
        \frac{\probi(\lgv)}{\probz(\lgv)} &\le \frac{ 2 (2\pi\cv\vi^2/(\clip^2+\cv\clip))^{-d/2} \exp \left( -(1-\cv) \nm{\lgv-\mui}^2 / 2\vi^2 \right) }{(2\pi \vz^2)^{-d/2} \exp \left( - (1+\cz) \nm{\lgv}^2 / 2\vz^2 \right)} \\ 
        &= 2 \left( \frac{(\clip^2+\cv\clip)\vz^2}{\cv\vi^2} \right)^{\frac{d}{2}} \exp \left( \frac{(1+\cz)\nm{\lgv}^2}{2\vz^2} - \frac{(1-\cv)\nm{\lgv-\mui}^2}{2\vi^2} \right) \\ 
        &\le 2 \left( \frac{(\clip^2+\cv\clip)\vz^2}{\cv\vi^2} \right)^{\frac{d}{2}} \exp \left( \frac{(1+\cz)\nm{\lgv-\mui}^2 + (1+\cz)\nm{\mui}^2}{\vz^2} - \frac{(1-\cv)\nm{\lgv-\mui}^2}{2\vi^2} \right) \label{eq:ng:bound_density_ratio} \\ 
        &\le 2 \exp \left( - \left( \left( \frac{1-\cv}{2\vi^2} - \frac{1+\cz}{\vz^2} \right) \left( \frac{\vz^2}{2} + \frac{\vi^2}{2(1-\cv)} \right) - \frac{1+\cz}{\vz^2} r^2 - \frac{1}{2} \log \frac{\cv\vi^2}{(\clip^2+\cv\clip)\vz^2} \right) d \right)
    \end{align}
    where the second last step follows from triangle inequality, and the last step follows from $\nm{\lgv-\mui}^2 \ge \left( \frac{\vz^2}{2} + \frac{\vi^2}{2(1-\cv)} \right) d$ and Assumption \ref{asmpt:ng:subgaussian}.\ref{asmpt:ng:pi:random} that $\mui$ is chosen from ball $\S$ of radius $r$. Therefore, from Assumption \ref{asmpt:ng:subgaussian}, we obtain $\frac{\probi(\lgv)}{\probz(\lgv)} \le \exp(-\Omega(d))$. 

    To show $\frac{\nm{\nabla_\lgv \probi(\lgv)}}{\P(\lgv)} \le \exp(-\Omega(d))$, from Assumptions \ref{asmpt:ng:subgaussian}.\ref{asmpt:ng:pi:differentiable} and \ref{asmpt:ng:subgaussian}.\ref{asmpt:ng:pi:score_lipschitz} we have 
    \begin{align*}
        \nm{ \frac{\nabla_\lgv \probi(\lgv)}{\probi(\lgv)} } &= \nm{ \frac{\nabla_\lgv \probi(\lgv)}{\probi(\lgv)} -  \frac{\nabla_\lgv \probi(\mui)}{\probi(\mui)}} = \nm{\nabla_\lgv \log \probi(\lgv) - \nabla_\lgv \log \probi(\mui)} \\ 
        &\le \lipi \nm{\lgv - \mui} \le \frac{\clip}{\vi^2} \nm{\lgv - \mui}. 
    \end{align*}
    Therefore, we can bound $\frac{\nm{\nabla_\lgv \probi(\lgv)}}{\P(\lgv)} \le \frac{\clip}{\vi^2} \frac{\probi(\lgv)}{\P(\lgv)} \nm{\lgv - \mui}$. When $\nm{\lgv - \mui} = \exp(o(d))$ is small, by $\frac{\probi(\lgv)}{\probz(\lgv)} \le \exp(-\Omega(d))$ we directly have $\frac{\nm{\nabla_\lgv \probi(\lgv)}}{\P(\lgv)} \le \exp(-\Omega(d))$. When $\nm{\lgv - \mui} = \exp(\Omega(d))$ is exceedingly large, from \eqref{eq:ng:bound_density_ratio} we have 
    \begin{align*}
        &\frac{\nm{\nabla_\lgv \probi(\lgv)}}{\P(\lgv)} \le \frac{2\clip}{\vi^2} \left( \frac{(\clip^2+\cv\clip)\vz^2}{\cv\vi^2} \right)^{\frac{d}{2}} \exp \left( \frac{(1+\cz)\nm{\lgv}^2}{2\vz^2} - \frac{(1-\cv)\nm{\lgv-\mui}^2}{2\vi^2} \right) \nm{\lgv - \mui} \\ 
        &\le \frac{2\clip}{\vi^2} \exp \left( - \left( \left( \frac{1-\cv}{2\vi^2} - \frac{1+\cz}{\vz^2} \right) \frac{\nm{\lgv-\mui}}{d} - \frac{1+\cz}{\vz^2} r^2 - \frac{1}{2} \log \frac{\cv\vi^2}{(\clip^2+\cv\clip)\vz^2} \right) d \right) \nm{\lgv - \mui} 
    \end{align*}
    Since $\frac{1-\cv}{2\vi^2} > \frac{1+\cz}{\vz^2}$, when $\nm{\lgv - \mui} = \exp(\Omega(d))$ we have $\frac{\nm{\nabla_\lgv \probi(\lgv)}}{\P(\lgv)} \le \exp(-\Omega(d))$. Thus we complete the proof of Lemma \ref{lem:ng:subgaussian_exp_decay}. 
\end{proof}

\begin{lemma} \label{lem:ng:nabla_difference}
    Suppose $\P$ satisfies Assumption \ref{asmpt:ng:subgaussian}. If $\lgv$ satisfies $\nm{\lgv-\mui}^2 \ge \left( \frac{\vz^2}{2} + \frac{\vi^2}{2(1-\cv)} \right) d$ for all $i \in [k]$, we have $\nm{\nabla_\lgv \log \P(\lgv) - \nabla_\lgv \log \probz(\lgv)} \le \exp(-\Omega(d))$. 
\end{lemma}

\begin{proof}[Proof of Lemma \ref{lem:ng:nabla_difference}]
    Since $\P = \sum_{i=0}^k \wi \probi$, we can decompose $\nabla_\lgv \log \P(\lgv)$ as 
    \begin{align*}
        \nabla_\lgv \log \P(\lgv) &= \frac{\nabla_\lgv \P(\lgv)}{\P(\lgv)} = \frac{\sum_{i=0}^k \wi \nabla_\lgv \probi(\lgv)}{\sum_{i=0}^k \wi \probi(\lgv)} \\ 
        &= \frac{\nabla \probz(\lgv)}{\probz(\lgv)} - \frac{\sum_{i=1}^k \wi \probi(\lgv)}{\wz\P(\lgv)} \cdot \frac{\nabla \probz(\lgv)}{\probz(\lgv)} + \frac{\sum_{i=1}^k \wi \nabla_\lgv \probi(\lgv)}{\wz\P(\lgv)} \\ 
        &= \nabla_\lgv \log \probz(\lgv) - \frac{\sum_{i=1}^k \wi \probi(\lgv)}{\wz\P(\lgv)} \nabla_\lgv \log \probz(\lgv) + \frac{\sum_{i=1}^k \wi \nabla_\lgv \probi(\lgv)}{\wz\P(\lgv)}
    \end{align*}
    From Lemma \ref{lem:ng:subgaussian_exp_decay} we know $\frac{\nm{\nabla_\lgv \probi(\lgv)}}{\P(\lgv)} \le \exp(-\Omega(d))$. It remains to show $\nm{\frac{\probi(\lgv)}{\P(\lgv)} \nabla_\lgv \log \probz(\lgv)} \le \exp (-\Omega(d))$. Since by Assumption \ref{asmpt:ng:subgaussian} the score function of $\probz$ is $\lipz$-Lipschitz, we have 
    \begin{align*}
        \nm{\frac{\probi(\lgv)}{\P(\lgv)} \nabla_\lgv \log \probz(\lgv)} &\le \frac{\probi(\lgv)}{\P(\lgv)} \left( \nm{\nabla \log \probz(0)} + \lipz\nm{\lgv} \right) \\ 
        &\le \frac{\probi(\lgv)}{\P(\lgv)} \lipz \nm{\lgv-\mui} + \exp(-\Omega(d)) 
    \end{align*}
    When $\nm{\lgv-\mui} = \exp(o(d))$ is small, by $\frac{\probi(\lgv)}{\P(\lgv)} \le \exp(-\Omega(d))$ we directly have $\frac{\probi(\lgv)}{\P(\lgv)} \nm{\lgv-\mui} \le \exp(-\Omega(d))$. When $\nm{\lgv-\mui} = \exp(\Omega(d))$ is exceedingly large, by \eqref{eq:ng:bound_density_ratio} we have 
    \begin{align*}
        \frac{\probi(\lgv)}{\P(\lgv)} \nm{\lgv-\mui} &\le 2 \left( \frac{(\clip^2+\cv\clip)\vz^2}{\cv\vi^2} \right)^{\frac{d}{2}} 
        \exp \left( \frac{(1+\cz)(\nm{\lgv-\mui}^2 + \nm{\mui}^2)}{\vz^2} - \frac{(1-\cv)\nm{\lgv-\mui}^2}{2\vi^2} \right) \nm{\lgv-\mui}
    \end{align*}
    Since $\frac{1-\cv}{2\vi^2} > \frac{1+\cz}{\vz^2}$, when $\nm{\lgv - \mui} = \exp(\Omega(d))$ we have $\frac{\probi(\lgv)}{\P(\lgv)} \nm{\lgv-\mui} \le \exp(-\Omega(d))$. Therefore, by combining the above we obtain 
    \begin{align*}
        \nm{\nabla_\lgv \log \P(\lgv) - \nabla_\lgv \log \probz(\lgv)} &\le \sum_{i=1}^k \frac{\wi}{\wz} \nm{\frac{\probi(\lgv)}{\P(\lgv)} \nabla_\lgv \log \probz(\lgv)} + \sum_{i=1}^k \frac{\wi}{\wz} \nm{\frac{\nabla_\lgv \probi(\lgv)}{\P(\lgv)}} \\ 
        &\le \exp(-\Omega(d))
    \end{align*}
    which finishes the proof of Lemma \ref{lem:ng:nabla_difference}. 
\end{proof}

We consider an auxiliary trajectory such that $\lz' = \lz$ and 
\begin{equation*}
    \lt' = \ltm' + \frac{\dt}{2} \nabla_\lgv \log \probz(\ltm') + \sqrt{\dt} \epst. 
\end{equation*}
Since the update rule of the auxiliary trajectory is independent of the modes $\P^{(1)}, \cdots, \P^{(k)}$ and $\mui$ is uniformly randomly initialized in a ball of radius $\r$, for any given $\lt'$ we have 
\begin{equation*}
    \Pr \left(\nm{\lt'-\mui}^2 > \left( \frac{3\vz^2}{4} + \frac{\vi^2}{4(1-\cv)} \right) d \right) \le \exp \left( -d \log \frac{4r^2}{3\vz^2+\vi^2/(1-\cv)} \right)
\end{equation*}
Hence, by the union bound, we have 
\begin{align}
    &\Pr\left(\nm{\lt'-\mui}^2 > \left( \frac{3\vz^2}{4} + \frac{\vi^2}{4(1-\cv)} \right) d \; \forall t \in \{0\} \cup [T], i \in [k]  \right) \\ 
    &\ge 1- \sum_{t=0}^T \sum_{i=1}^k \Pr \left(\nm{\lt'-\mui}^2 > \left( \frac{3\vz^2}{4} + \frac{\vi^2}{4(1-\cv)} \right) d \right) \\ 
    &\ge 1- (T+1)k \exp \left( -d \log \frac{4r^2}{3\vz^2+\vi^2/(1-\cv)} \right).  \label{eq:ng:lgv'_high_prob_bound}
\end{align}

Now we are ready to prove Theorem \ref{thm:ng:subgaussian}. Notice that concavity and $\lipz$-smoothness of $\log \probz(\lgv)$ imply that the gradients are co-coercive, i.e., 
\begin{equation*}
    \left\langle \nabla_\lgv \log \probz(\lgv), \nabla_\lgv \log \probz(\lgv') \right\rangle \le -\frac{1}{\lipz} \nm{\nabla_\lgv \log \probz(\lgv) - \nabla_\lgv \log \probz(\lgv')}^2. 
\end{equation*}
Therefore, for step size $\lss \le \frac{4}{\lipz}$ we have 
\begin{align}
    &\nm{\lgv + \frac{\lss}{2} \nabla_\lgv \log \probz(\lgv) - \lgv' - \frac{\lss}{2} \nabla_\lgv \log \probz(\lgv')}^2 \\ 
    &= \nm{\lgv-\lgv'}^2 + \lss \left\langle \nabla_\lgv \log \probz(\lgv), \nabla_\lgv \log \probz(\lgv') \right\rangle + \frac{\lss^2}{4} \nm{\nabla_\lgv \log \probz(\lgv) - \nabla_\lgv \log \probz(\lgv')}^2 \\ 
    &\le \nm{\lgv-\lgv'}^2 + \left( \frac{\lss^2}{4} - \frac{\lss}{\lipz} \right) \nm{\nabla_\lgv \log \probz(\lgv) - \nabla_\lgv \log \probz(\lgv')}^2 \\ 
    &\le \nm{\lgv-\lgv'}^2 \label{eq:ng:co-coercive}
\end{align}
If $\ltm$ satisfies $\nm{\ltm-\mui}^2 \ge \left( \frac{\vz^2}{2} + \frac{\vi^2}{2(1-\cv)} \right) d$ for all $i \in [k]$, combining Lemma \ref{lem:ng:nabla_difference} and \eqref{eq:ng:co-coercive} gives 
\begin{align*}
    \nm{\lt-\lt'} &= \nm{\ltm + \frac{\dt}{2} \nabla_\lgv \log \P(\ltm) - \ltm' - \frac{\dt}{2} \nabla_\lgv \log \probz(\ltm')} \\ 
    &\le \nm{\ltm + \frac{\dt}{2} \nabla_\lgv \log \probz(\ltm) - \ltm' - \frac{\dt}{2} \nabla_\lgv \log \probz(\ltm')} \\ 
    &\quad \quad + \frac{\dt}{2} \nm{\nabla_\lgv \log \P(\ltm) - \nabla_\lgv \log \probz(\ltm)} \\ 
    &\le \nm{\ltm-\ltm} + \exp(-\Omega(d))
\end{align*}
Assuming $\nm{\lt'-\mui}^2 > \left( \frac{3\vz^2}{4} + \frac{\vi^2}{4(1-\cv)} \right) d$ for all $t \in \{0\} \cup [T] $ and $ i \in [k]$, which holds with probability $1-T \cdot \exp(-\Omega(d))$ due to \eqref{eq:ng:lgv'_high_prob_bound}, by induction we can easily obtain that when 
\begin{equation*}
    T \cdot \exp(-\Omega(d)) \le \sqrt{\left( \frac{3\vz^2}{4} + \frac{\vi^2}{4(1-\cv)} \right) d} - \sqrt{\left( \frac{\vz^2}{2} + \frac{\vi^2}{2(1-\cv)} \right) d}
\end{equation*}
we have $\nm{\lt-\lt'} \le T \cdot \exp(-\Omega(d)) $ and $\nm{\lt-\mui}^2 > \left( \frac{\vz^2}{2} + \frac{\vi^2}{2(1-\cv)} \right) d$ for all $t \in \{0\} \cup [T] $ and $ i \in [k]$, which completes the proof of Theorem \ref{thm:ng:subgaussian}. 
\end{proof}

\section{Proof of Theorem \ref{thm:chained}}
\label{app:prf:chainedld}

\begin{assumption} \label{asmpt:chained}
    For a target distribution $\P$, denote $U_q\left( \lgv^{(q)} \right) := U\left( \lgv^{(q)} | \lgv^{(1)}, \cdots, \lgv^{(q-1)} \right) = -\log \P \left( \lgv^{(q)} | \lgv^{(1)}, \cdots, \lgv^{(q-1)} \right)$. For all $q \in [d/Q]$ and $\lgv^{(1)}, \cdots, \lgv^{(q-1)} \in \sR^Q$, assume that $U_q\left( \lgv^{(q)} \right)$ satisfies: 
    \begin{enumerate}[label=\roman*., leftmargin=*, noitemsep]
        \vspace{-2.5mm}
        \item $U_q\left( \lgv^{(q)} \right)$ is $\LQ$-smooth, and the Hessian exists for all $\lgv^{(q)} \in \sR^Q$. \label{asmpt:chained:lipschitz}
        \\ That is: $\forall \rva,\rvb \in \sR^Q$, $\nm{\nabla U_q(\rva) - \nabla U_q(\rvb)} \le \LQ \nm{\rva - \rvb}$, and $\nabla^2 U_q(\rva)$ exists. 
        \vspace{0.9mm}
        \item $U_q\left( \lgv^{(q)} \right)$ is $\mQ$-strongly convex for $\nm{\lgv^{(q)}} > \RQ$. \label{asmpt:chained:strongly-convex}
        \vspace{0.9mm}
        \\ That is: $V^q (\rva) := U^q(\rva) - \frac{\mQ}{2} \nm{\rva}^2$ is convex on $\Gamma := \sR^{Q} \backslash \set{\rva : \nm{\rva}\le \RQ }$. We follow the definition of convexity on non-convex domains \citep{peters1986convex, min2012extension, ma2019sampling} that $\forall \rva \in \Gamma$, any convex combination of $\rva = \lambda_1 \rva_1 + \cdots + \lambda_m \rva_m$ with $\rva_1 , \cdots, \rva_m \in \Gamma$ satisfies 
        $$V^q(\rva) \le \lambda_1 V^q(\rva_1) + \cdots + \lambda_m V^q(\rva_m).$$
        \item $\nabla U_q(\0_Q) = \0_Q$. \label{asmpt:chained:local-extremum}
    \end{enumerate}
\end{assumption}

\begin{proposition} \label{prop:chained:convergence_patch}
     Consider a data distribution $\P$ satisfying Assumption \ref{asmpt:chained}. We initialize $\lz \sim \gN(\0_d, \frac{1}{\LQ} \mI_d)$ and apply chained Langevin dynamics in Algorithm \ref{alg:ChainedLD} with constant patch size $Q$, noise level $\nlt = 0$, and step size $\dt = \frac{\mQ\varepsilon^2 Q}{64\LQ^2 d^2} \exp(-16 \LQ \RQ^2)$. Then, Algorithm \ref{alg:ChainedLD} can achieve
    \begin{equation*}
        \TV \left( \hP \left( \lgv^{(q)} \mid \lgv^{(1)}, \cdots, \lgv^{(q-1)} \right) , \P \left( \lgv^{(q)} \mid \lgv^{(1)}, \cdots, \lgv^{(q-1)} \right) \right) \le \varepsilon \cdot \frac{Q}{d}
    \end{equation*}
    in $T = \frac{128 \LQ^2 d^3}{\mQ^2 Q^2 \varepsilon^2} \exp(32\LQ \RQ^2)  \log \gO \left( \frac{d^3}{\varepsilon^2 Q^2} \right)$ iterations. 
\end{proposition}
\begin{proof}[Proof of Proposition \ref{prop:chained:convergence_patch}]
    First, for $U_q$ satisfying Assumptions \ref{asmpt:chained},  by Proposition 1 of \citep{ma2019sampling}, the conditional distribution $\P \left( \lgv^{(q)} | \lgv^{(1)}, \cdots, \lgv^{(q-1)} \right)$ satisfies log-Sobolev inequality with constant $\rho_Q = \frac{\mQ}{2} \exp(-16\LQ\RQ^2)$. 
    
    Then, we note that chained Langevin dynamics in Algorithm \ref{alg:ChainedLD} applies $TQ/d$ iterations to sample patch $\lgv^{(q)}$ from the conditional distribution $\P \left( \lgv^{(q)} | \lgv^{(1)}, \cdots, \lgv^{(q-1)} \right)$. Denote $\hP \left( \lgv^{(q)}_t | \lgv^{(1)}, \cdots, \lgv^{(q-1)} \right)$ the law of the generated sample $\lgv^{(q)}_t$ at time $t$. Since $\lz \sim \gN(\0_d, \frac{1}{\LQ}\mI_d)$, we have 
    \begin{equation*}
        \hP \left( \lgv^{(q)}_0 | \lgv^{(1)}, \cdots, \lgv^{(q-1)} \right) = \gN(\0_Q, \frac{1}{\LQ} \mI_Q). 
    \end{equation*}
    Therefore, by Lemma 7 in \citep{ma2019sampling} and Assumption \ref{asmpt:chained}, we have 
    \begin{equation*}
        \KL \left( \hP \left( \lgv^{(q)}_0 | \lgv^{(1)}, \cdots, \lgv^{(q-1)} \right) || \P \left( \lgv^{(q)} | \lgv^{(1)}, \cdots, \lgv^{(q-1)} \right) \right) \le \frac{Q}{2} \log \frac{2\LQ}{\mQ} + \frac{32 \LQ^2}{\mQ^2} \LQ \RQ^2 \ll d. 
    \end{equation*}
    Since the conditional distribution $\P \left( \lgv^{(q)} | \lgv^{(1)}, \cdots, \lgv^{(q-1)} \right)$ satisfies log-Sobolev inequality with constant $\rhoQ = \frac{\mQ}{2} \exp(-16\LQ\RQ^2)$, for step size $\delta = \frac{\mQ\varepsilon^2 Q}{64\LQ^2 d^2} \exp(-16 \LQ \RQ^2)$, by Theorem 1 in \citep{vempala2019rapid} we obtain that at iteration $t$,
    \begin{align*}
        &\KL \left( \hP \left( \lgv^{(q)}_t | \lgv^{(1)}, \cdots, \lgv^{(q-1)} \right) || \P \left( \lgv^{(q)} | \lgv^{(1)}, \cdots, \lgv^{(q-1)} \right) \right) \\ 
        &\le \exp(- \rhoQ \delta t) \KL \left( \hP \left( \lgv^{(q)}_0 | \lgv^{(1)}, \cdots, \lgv^{(q-1)} \right) || \P \left( \lgv^{(q)} | \lgv^{(1)}, \cdots, \lgv^{(q-1)} \right) \right) + \frac{8 \delta Q \LQ^2}{\rhoQ} \\ 
        &\le \exp(- \rhoQ \delta t) d + \frac{\varepsilon^2 Q^2}{4d^2}. 
    \end{align*}
    Therefore, when the total number of iterations $T$ satisfies
    \begin{equation*}
        T \ge \frac{d}{Q\rhoQ \delta} \log \gO \left( \frac{d^3}{\varepsilon^2 Q^2} \right) = \frac{128 \LQ^2 d^3}{\mQ^2 Q^2 \varepsilon^2} \exp(32\LQ \RQ^2)  \log \left( \frac{d^3}{\varepsilon^2 Q^2} \right),
    \end{equation*}
    at iteration $t = TQ/d$ we have 
    \begin{equation*}
        \KL \left( \hP \left( \lgv^{(q)}_{TQ/d} | \lgv^{(1)}, \cdots, \lgv^{(q-1)} \right) || \P \left( \lgv^{(q)} | \lgv^{(1)}, \cdots, \lgv^{(q-1)} \right) \right) \le \frac{\varepsilon^2 Q^2}{2d^2}. 
    \end{equation*}
    Finally, by Pinsker's inequality, we have the total variation bound
    \begin{align*}
        &\TV \left( \hP \left( \lgv^{(q)} \mid \lgv^{(1)}, \cdots, \lgv^{(q-1)} \right) , \P \left( \lgv^{(q)} \mid \lgv^{(1)}, \cdots, \lgv^{(q-1)} \right) \right) \\ 
        &\le \sqrt{2 \left( \hP \left( \lgv^{(q)} \mid \lgv^{(1)}, \cdots, \lgv^{(q-1)} \right) , \P \left( \lgv^{(q)} \mid \lgv^{(1)}, \cdots, \lgv^{(q-1)} \right) \right)} \le \frac{\varepsilon Q}{d}. 
    \end{align*}
    Thus we finish the proof of Proposition \ref{prop:chained:convergence_patch}. 
\end{proof}

\begin{proposition} \label{prop:ChainedLD:linear_reduction}
    Consider a sampler algorithm taking the first $q-1$ patches $\lgv^{(1)}, \cdots, \lgv^{(q-1)}$ as input and outputing a sample of the next patch $\lgv^{(q)}$ with probability $\hP \left( \lgv^{(q)} \mid \lgv^{(1)}, \cdots, \lgv^{(q-1)} \right)$ for all $q \in [d/Q]$. Suppose that for every $q \in [d/Q]$ and any given previous patches $\lgv^{(1)}, \cdots, \lgv^{(q-1)}$, the sampler algorithm can achieve
    \begin{equation*}
        \TV \left( \hP \left( \lgv^{(q)} \mid \lgv^{(1)}, \cdots, \lgv^{(q-1)} \right) , \P \left( \lgv^{(q)} \mid \lgv^{(1)}, \cdots, \lgv^{(q-1)} \right) \right) \le \varepsilon \cdot \frac{Q}{d}
    \end{equation*}
    for some $\varepsilon > 0$. Then, equipped with the sampler algorithm, the Chained-LD algorithm can achieve 
    \begin{equation*}
        \TV \left( \hP(\lgv) , \P(\lgv) \right) \le \varepsilon. 
    \end{equation*}
\end{proposition}

\begin{proof}[Proof of Proposition \ref{prop:ChainedLD:linear_reduction}]
For simplicity, denote $\lgv^{[q]} = \set{\lgv^{(1)}, \cdots, \lgv^{(q)}}$. By the definition of total variation distance, for all $q \in [d/Q]$ we have 
\begin{align*}
    &\TV \left( \hP \left( \lgv^{[q]} \right) , \P \left( \lgv^{[q]} \right) \right) \\
    &= \frac{1}{2} \int \left\lvert \hP \left( \lgv^{[q]} \right) - \P \left( \lgv^{[q]} \right) \right\rvert \dif \lgv^{[q]} \\ 
    &= \frac{1}{2} \int \left\lvert \hP \left( \lgv^{(q)} \mid \lgv^{[q-1]} \right) \hP \left( \lgv^{[q-1]} \right) - \P \left( \lgv^{(q)} \mid \lgv^{[q-1]} \right) \P \left( \lgv^{[q-1]} \right) \right\rvert \dif \lgv^{[q]} \\ 
    &\le \frac{1}{2} \int \left\lvert \hP \left( \lgv^{(q)} \mid \lgv^{[q-1]} \right) \hP \left( \lgv^{[q-1]} \right) - \hP \left( \lgv^{(q)} \mid \lgv^{[q-1]} \right) \P \left( \lgv^{[q-1]} \right) \right\rvert \dif \lgv^{[q]} \\ 
    &\quad\quad + \frac{1}{2} \int \left\lvert \hP \left( \lgv^{(q)} \mid \lgv^{[q-1]} \right) \P \left( \lgv^{[q-1]} \right) - \P \left( \lgv^{(q)} \mid \lgv^{[q-1]} \right) \P \left( \lgv^{[q-1]} \right) \right\rvert \dif \lgv^{[q]} \\ 
    &= \frac{1}{2} \int \hP \left( \lgv^{(q)} \mid \lgv^{[q-1]} \right) \dif \lgv^{(q)} \int \left\lvert \hP \left( \lgv^{[q-1]} \right) - \P \left( \lgv^{[q-1]} \right) \right\rvert \dif \lgv^{[q-1]} \\ 
    &\quad\quad + \frac{1}{2} \int \left\lvert \hP \left( \lgv^{(q)} \mid \lgv^{[q-1]} \right) - \P \left( \lgv^{(q)} \mid \lgv^{[q-1]} \right)  \right\rvert \dif \lgv^{(q)} \int \P \left( \lgv^{[q-1]} \right) \dif \lgv^{[q-1]} \\ 
    &= \TV \left( \hP \left( \lgv^{[q-1]} \right) , \P \left( \lgv^{[q-1]} \right) \right) + \TV \left( \hP \left( \lgv^{(q)} \mid \lgv^{[q-1]} \right) , \P \left( \lgv^{(q)} \mid \lgv^{[q-1]} \right) \right) \\ 
    &\le \TV \left( \hP \left( \lgv^{[q-1]} \right) , \P \left( \lgv^{[q-1]} \right) \right) + \varepsilon \cdot \frac{Q}{d}. 
\end{align*}
Upon summing up the above inequality for all $q \in [d/Q]$, we obtain 
\begin{align*}
     \TV \left( \hP(\lgv) , \P(\lgv) \right) &= \sum_{q=1}^{d/Q} \left( \TV \left( \hP \left( \lgv^{[q]} \right) , \P \left( \lgv^{[q]} \right) \right) - \TV \left( \hP \left( \lgv^{[q-1]} \right) , \P \left( \lgv^{[q-1]} \right) \right) \right) \\ 
     &\le \sum_{q=1}^{d/Q} \varepsilon \cdot \frac{Q}{d} = \varepsilon
\end{align*}
Thus we finish the proof of Proposition \ref{prop:ChainedLD:linear_reduction}.
\end{proof}

Finally, upon combining Propositions \ref{prop:chained:convergence_patch} and \ref{prop:ChainedLD:linear_reduction}, we finish the proof of Theorem \ref{thm:chained}.

\section{Experimental Details and Additional Experiments} \label{app:exp}
\textbf{Algorithm Setup:} 
Our choices of algorithm hyperparameters are based on \citep{song2019generative} and \citep{song2020improved}. For $\sigma_{\max}=1$, following from \citep{song2019generative}, we consider $L=10$ different standard deviations such that $\set{\lambda_i}_{i \in [L]}$ is a geometric sequence with $\lambda_1 = 1$ and $\lambda_{10} = 0.01$. For annealed Langevin dynamics with $T$ iterations, we choose the noise levels $\set{\nl_t}_{t\in [T]}$ by repeating every element of $\set{\lambda_i}_{i \in [L]}$ for $T/L$ times and we set the step size as $\dt = 2 \times 10^{-5} \cdot \nlt^2 / \nlT^2$ for every $t \in [T]$. For vanilla Langevin dynamics with $T$ iterations, we use the same step size as annealed Langevin dynamics. For Chained-VLD and Chained-ALD, the patch size $Q$ is chosen depending on different tasks. For every patch of Chained-ALD, we choose the noise levels $\set{\nl_t}_{t\in [TQ/d]}$ by repeating every element of $\set{\lambda_i}_{i \in [L]}$ for $TQ/dL$ times and we set the step size as $\dt = 2 \times 10^{-5} \cdot \nlt^2 / \nl_{TQ/d}^2$ for every $t \in [TQ/d]$. The step size of Chained-VLD is the same as Chained-ALD.

We would like to highlight that the inference time of Chained-LD is significantly lower than vanilla LD in practice. Our theoretical comparison between Chained-LD and vanilla LD is based on iteration complexity, i.e., the number of queries to the score function $\nabla \log P(x^{(q)}|x^{(1)}, \cdots, x^{(q-1)})$ or $\nabla \log P(x)$. Since Chained-LD only updates one patch at every iteration while vanilla LD updates the whole image, Chained-LD will be significantly faster than vanilla LD.

\subsection{Synthetic Gaussian Mixture Model} \label{app:exp:synthetic}
We choose the data distribution $\P$ as a mixture of three Gaussian components in dimension $d=100$:
\begin{equation*}
    \P = 0.2 \P^{(0)} + 0.4 \P^{(1)} + 0.4 \P^{(2)} = 0.2 \gN(\0_d, 3\mI_d) + 0.4 \gN(\1_d, \mI_d) + 0.4 \gN(-\1_d, \mI_d). 
\end{equation*}
Since the distribution is given, we assume that the sampling algorithms have access to the ground-truth score function. We set the batch size as 1000 and patch size $Q=10$ for chained Langevin dynamics. We use $T \in \set{10^3, 10^4, 10^5, 10^6}$ iterations for vanilla and chained Langevin dynamics. A sample $\lgv$ is clustered in mode 1 if it satisfies $\nm{\lgv-\muo}^2 \le 5d$ and $\nm{\lgv-\muo}^2 \le \nm{\lgv-\rvmu_2}^2$; in mode 2 if $\nm{\lgv-\rvmu_2}^2 \le 5d$ and $\nm{\lgv-\muo}^2 > \nm{\lgv-\rvmu_2}^2$; and in mode 0 otherwise. The initial samples are i.i.d. chosen from $\P^{(0)}$, $\P^{(1)}$, or $\P^{(2)}$, and the results are presented in Figures \ref{fig:syn:mode0}, \ref{fig:syn:mode1}, and \ref{fig:syn:mode2} respectively. The two subfigures above the dashed line illustrate the samples from the initial distribution and target distribution, and the subfigures below the dashed line are the samples generated by different algorithms. Furthermore, in Figures \ref{fig:syn:mode0_patch}, \ref{fig:syn:mode1_patch} and \ref{fig:syn:mode2_patch} we demonstrate the effect of different values of $Q \in \{1,4,10,20,50\}$ on the convergence of Chained-LD. We can observe that for dimension $d=100$, a moderate patch size $Q \in \{1,4,10\}$ has similar performance, a large patch size $Q=20$ needs more steps to find the other two modes, while an overly-large patch size $Q=50$ almost cannot find other modes.

We further numerically evaluate the performance of LD and Chained-LD in other Gaussian mixture models. We consider an in-between mode $\probz = \gN(\0_d, 10\mI_d)$ in dimension $d=100$ with weight $\wz = 0.01$, and the other modes have the same weight, the same covariance matrix but different mean, i.e., $\probi = \gN(\mui, 0.1\mI_d)$ and $\wi = 0.99/k$. The first two coordinates of $\mui$ are chosen as shown in Figures \ref{fig:syn:mix9} and \ref{fig:syn:mix10}, and the other coordinates of $\mui$ are set to be 0. The numerical results in Figures \ref{fig:syn:mix9} and \ref{fig:syn:mix10} are consistent with our previous analysis.

\subsection{Score Function Estimator} \label{app:score_estimator}
In realistic scenarios, since we do not have direct access to the (perturbed) score function, \citep{song2019generative} proposed the Noise Conditional Score Network (NCSN) $\rvs_{\rvtheta}(\lgv, \nl)$ to jointly estimate the scores of all perturbed data distributions, i.e., 
\begin{equation*}
    \forall \nl \in \set{\nlt}_{t \in [T]},\; \rvs_{\rvtheta}(\lgv, \nl) \approx \nabla_\lgv \log \P_{\nl} (\lgv). 
\end{equation*} 
To train the NCSN, \citep{song2019generative} adopted denoising score matching, which minimizes the following loss
\begin{equation*}
    \gL \left(\rvtheta; \set{\nlt}_{t \in [T]} \right) := \frac{1}{2T} \sum_{t \in [T]} \nlt^2 \E_{\lgv \sim \P} \E_{\Tilde{\lgv} \sim \gN(\lgv, \nlt^2 \mI_d)} \biggl[ \nm{\rvs_{\rvtheta} (\Tilde{\lgv}, \nlt) - \frac{\Tilde{\lgv} - \lgv}{\nlt^2}}^2 \biggr]. 
\end{equation*}
Assuming the NCSN has enough capacity and sufficient training samples, $\rvs_{\rvtheta^*}(\lgv, \nl)$ minimizes the loss $\gL \left(\rvtheta; \set{\nlt}_{t \in [T]} \right)$ if and only if $\rvs_{\rvtheta^*}(\lgv, \nlt) = \nabla_\lgv \log \P_{\nlt} (\lgv)$ almost surely for all $t \in [T]$.

In Chained Langevin dynamics, an ideal conditional score function estimator $\rvs_{\rvtheta}$ could jointly estimate the scores of all perturbed conditional patch distribution, i.e., $\forall \nl \in \set{\nlt}_{t \in [TQ/d]}, q \in [d/Q]$, 
\begin{equation*}
    \rvs_{\rvtheta} \left( \lgv^{(q)} \mid \nl, \lgv^{(1)}, \cdots, \lgv^{(q-1)} \right) \approx \nabla_{\lgv^{(q)}} \log \P_{\nl}(\lgv^{(q)} \mid \lgv^{(1)}, \cdots \lgv^{(q-1)}). 
\end{equation*} 
Following from \citep{song2019generative}, we use the denoising score matching to train the estimator. For a given $\nl$, the denoising score matching objective is 
\begin{equation*}
    \ell(\rvtheta; \nl) := \frac{1}{2} \E_{\lgv \sim \P} \E_{\Tilde{\lgv} \sim \gN(\lgv, \nl^2 \mI_d)} \sum_{q \in [d/Q]} \left[ \nm{\rvs_{\rvtheta} \left( \lgv^{(q)} \mid \nl, \lgv^{(1)}, \cdots, \lgv^{(q-1)} \right) - \frac{\Tilde{\lgv}^{(q)} - \lgv^{(q)}}{\nl^2}}^2 \right]. 
\end{equation*}
Then, combining the objectives gives the following loss 
\begin{equation*}
    \gL \left(\rvtheta; \set{\nlt}_{t \in [TQ/d]} \right) := \frac{d}{TQ} \sum_{t \in [TQ/d]} \nlt^2 \ell(\rvtheta; \nlt). 
\end{equation*}
As shown in \citep{vincent2011connection}, an estimator $\rvs_{\rvtheta}$ with enough capacity and sufficient training samples minimizes the loss $\gL$ if and only if $\rvs_{\rvtheta}$ outputs the scores of all perturbed conditional patch distribution almost surely.

\subsection{Image Datasets} \label{app:exp:image}
Our implementation and hyperparameter selection are based on \citep{song2019generative} and \citep{song2020improved}. During training, we i.i.d. randomly flip an image with probability 0.5 to construct the two modes (i.e., original and flipped images). All models are optimized by Adam with learning rate 0.001 and batch size 128 for a total of 200000 training steps, and we use the model at the last iteration to generate the samples. We perform experiments on MNIST \citep{lecun1998mnist} (CC BY-SA 3.0 License) and Fashion-MNIST \citep{xiao2017fashion} (MIT License) datasets and we set the patch size as $Q=14$. All experiments were run with one RTX3090 GPU.

For the score networks of chained annealed Langevin dynamics (Chained-ALD), we use the official PyTorch implementation of an LSTM network \citep{sak2014long} followed by a linear layer. For MNIST and Fashion-MNIST datasets, we set the input size of the LSTM as $Q=14$, the number of features in the hidden state as 1024, and the number of recurrent layers as 2. The inputs of LSTM include inputting tensor, hidden state, and cell state, and the outputs of LSTM include the next hidden state and cell state, which can be fed to the next input. To estimate the noisy score function, we first input the noise level $\nl$ (repeated for $Q$ times to match the input size of LSTM) and all-0 hidden and cell states to obtain an initialization of the hidden and cell states. Then, we divide a sample into $d/Q$ patches and input the sequence of patches to the LSTM. For every output hidden state corresponding to one patch, we apply a linear layer of size $1024 \times Q$ to estimate the noisy score function of the patch.

To generate samples, we use $T \in \set{10000, 30000, 100000}$ iterations for annealed Langevin dynamics (ALD) and Chained-ALD. The initial samples are chosen as either original or flipped images from the dataset, and the results for MNIST and Fashion-MNIST datasets are presented in Figures \ref{fig:mnist_orig_ald}, \ref{fig:mnist_flip_ald}, \ref{fig:fashion_orig_ald}, and \ref{fig:fashion_flip_ald} respectively. The two subfigures above the dashed line illustrate the samples from the initial distribution and target distribution, and the subfigures below the dashed line are the samples generated by different algorithms.

\begin{figure}
    \centering
    \begin{tabular}{c}
       \includegraphics[width=0.95\textwidth]{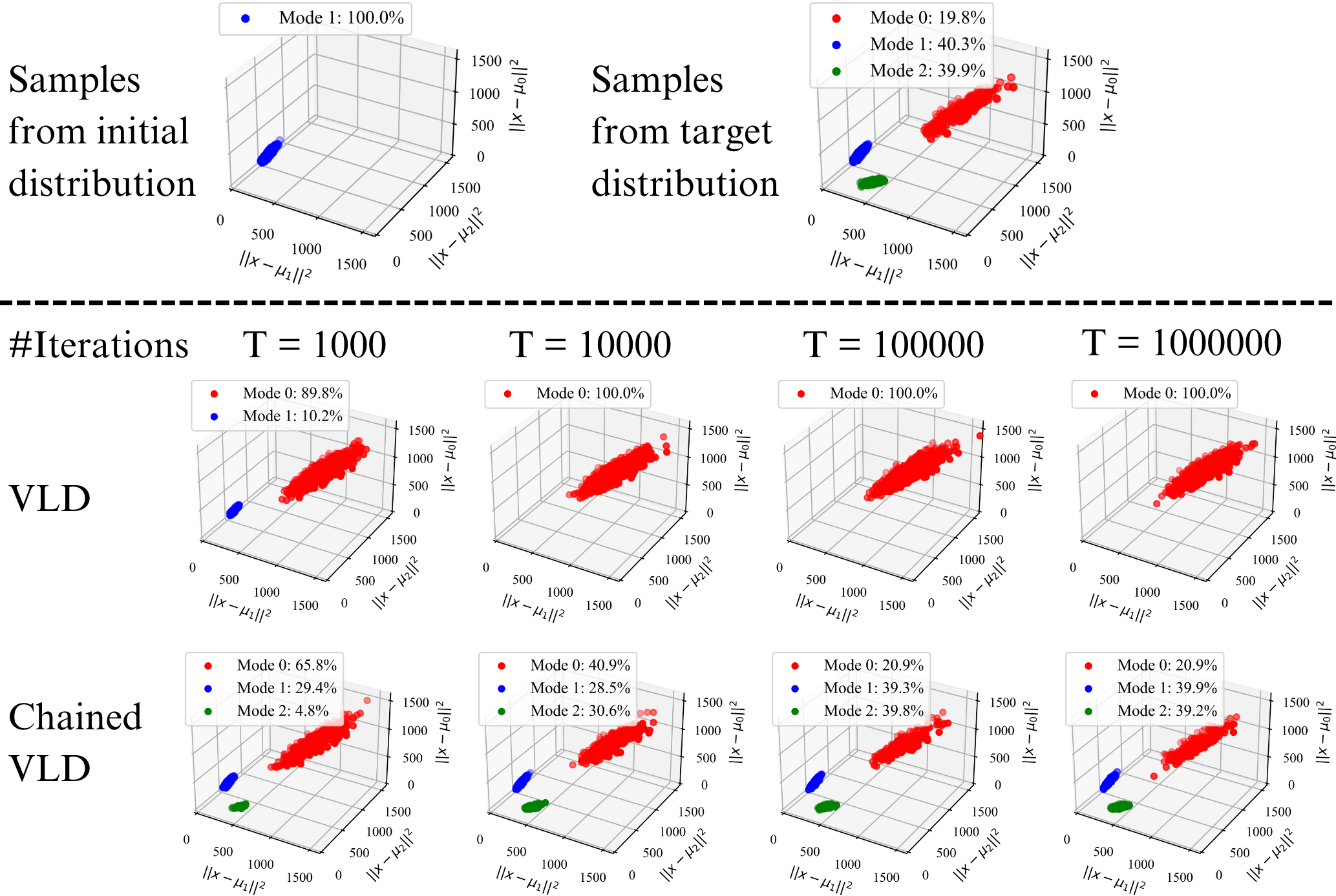}
      \end{tabular}
    \caption{Samples from a mixture of three Gaussian modes generated by vanilla Langevin dynamics (VLD) and chained vanilla Langevin dynamics (Chained-VLD). Three axes are $\ell_2$ distance from samples to the mean of the three modes. The samples are initialized in mode 1. } \label{fig:syn:mode1}
\end{figure}

\begin{figure}
    \centering
    \begin{tabular}{c}
       \includegraphics[width=0.95\textwidth]{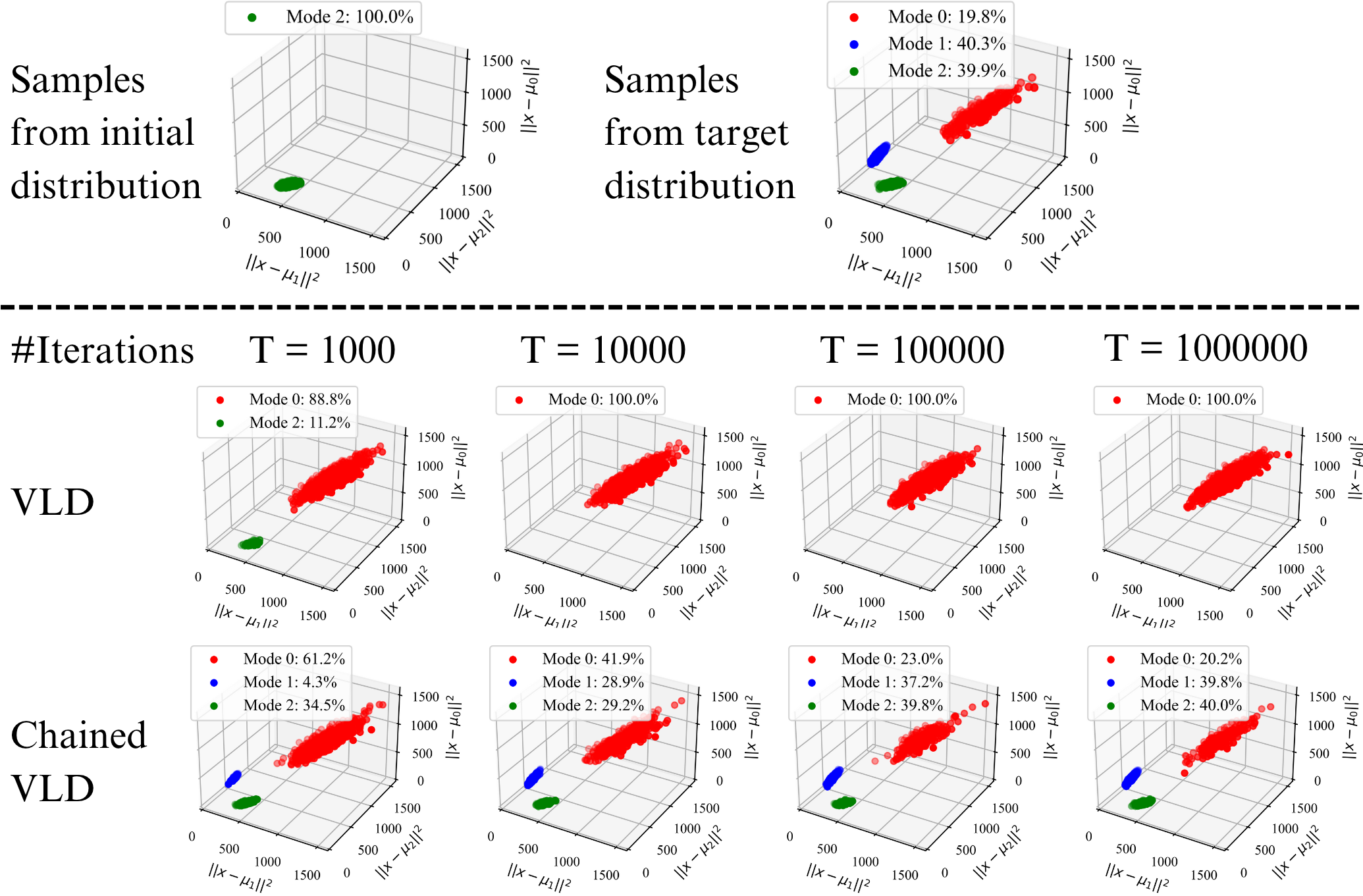}
      \end{tabular}
    \caption{Samples from a mixture of three Gaussian modes generated by vanilla Langevin dynamics (VLD) and chained vanilla Langevin dynamics (Chained-VLD). Three axes are $\ell_2$ distance from samples to the mean of the three modes. The samples are initialized in mode 2. } \label{fig:syn:mode2}
\end{figure}

\begin{figure}
    \centering
    \begin{tabular}{c}
       \includegraphics[width=1\textwidth]{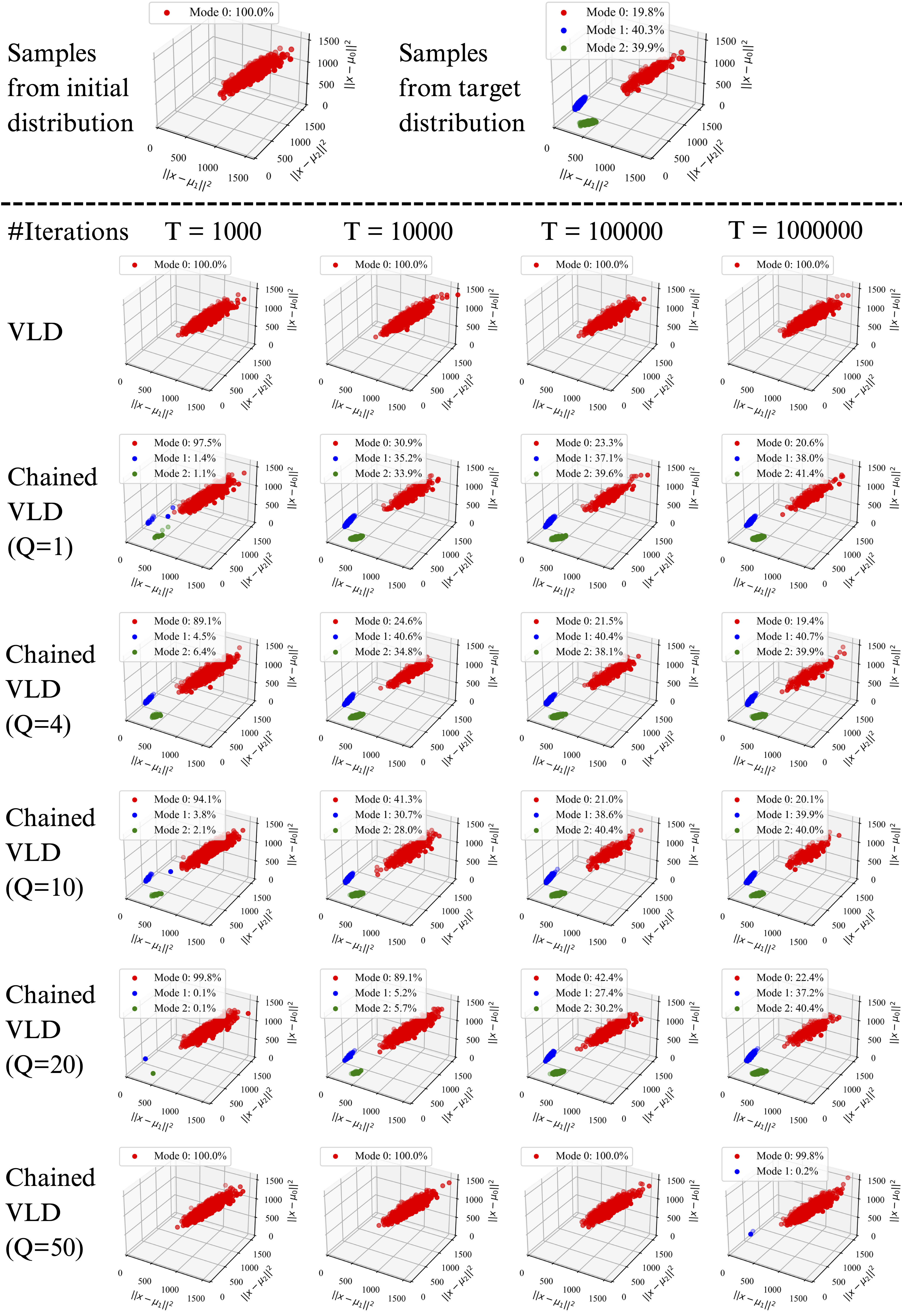}
      \end{tabular}
    \caption{Samples from a mixture of three Gaussian modes generated by vanilla Langevin dynamics (VLD) and chained vanilla Langevin dynamics (Chained-VLD) with patch size $Q \in \{1,4,10,20,50\}$. Three axes are $\ell_2$ distance from samples to the mean of the three modes. The samples are initialized in mode 0. } \label{fig:syn:mode0_patch}
\end{figure}

\begin{figure}
    \centering
    \begin{tabular}{c}
       \includegraphics[width=1\textwidth]{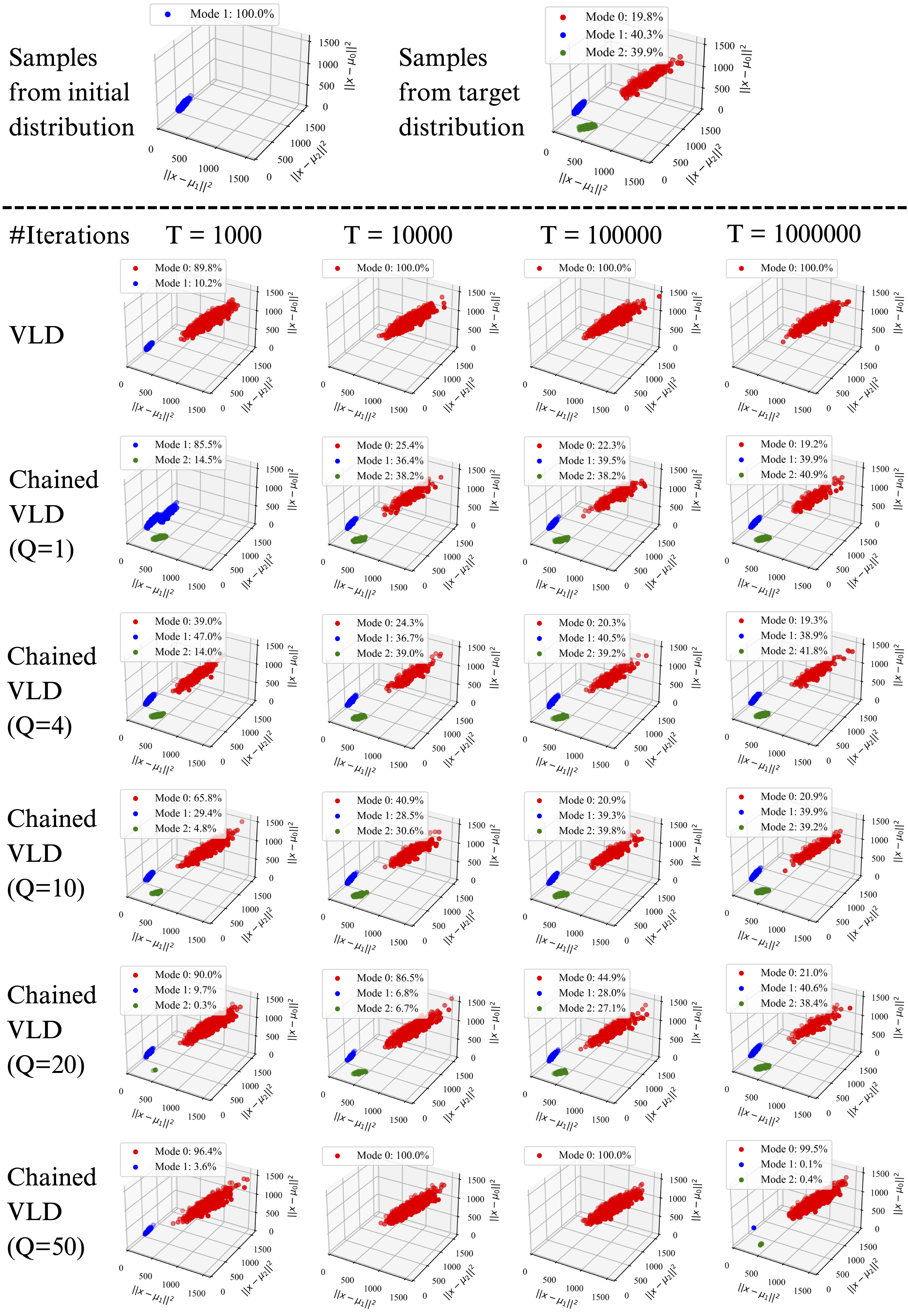}
      \end{tabular}
    \caption{Samples from a mixture of three Gaussian modes generated by vanilla Langevin dynamics (VLD) and chained vanilla Langevin dynamics (Chained-VLD) with patch size $Q \in \{1,4,10,20,50\}$. Three axes are $\ell_2$ distance from samples to the mean of the three modes. The samples are initialized in mode 1. } \label{fig:syn:mode1_patch}
\end{figure}

\begin{figure}
    \centering
    \begin{tabular}{c}
       \includegraphics[width=1\textwidth]{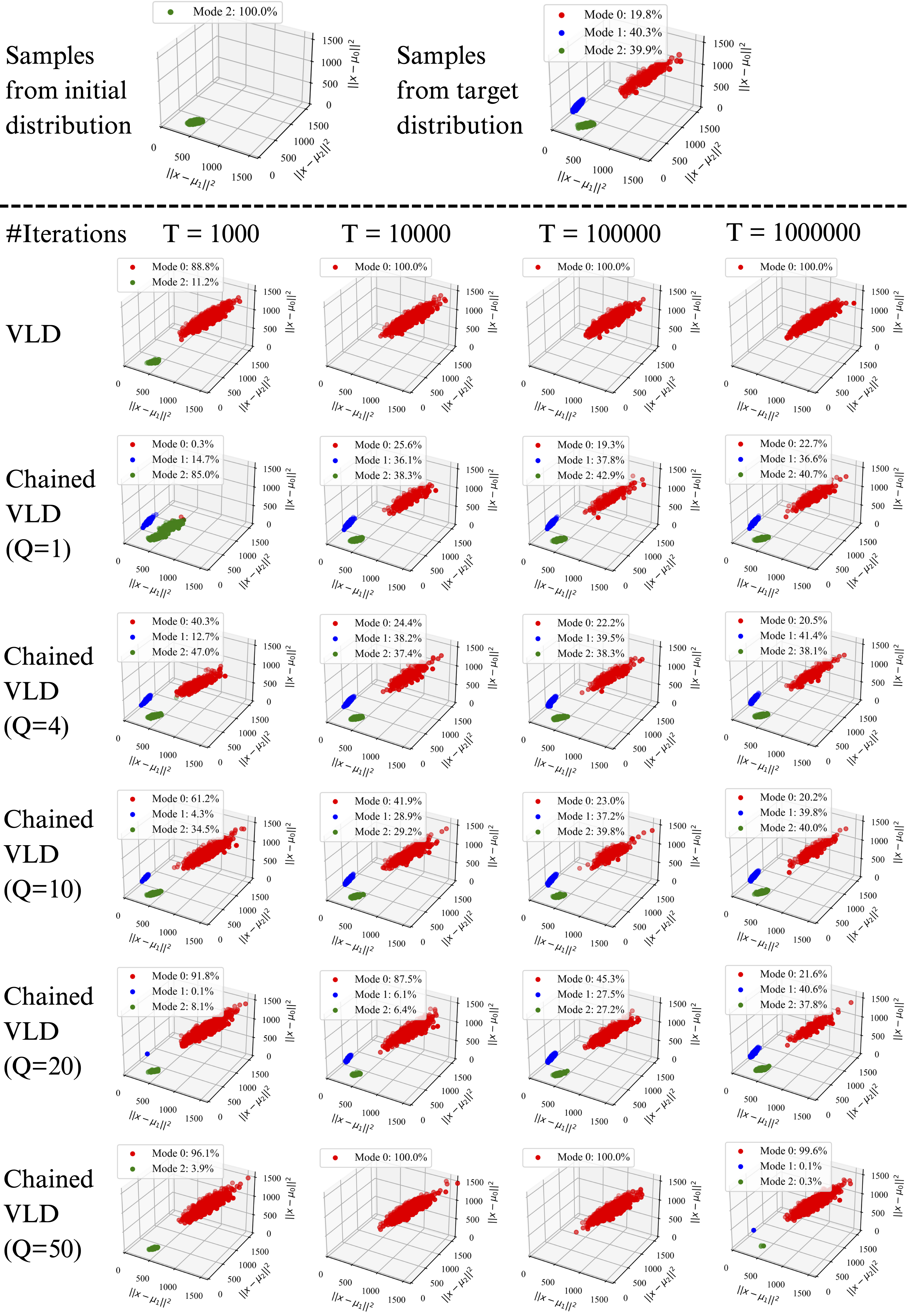}
      \end{tabular}
    \caption{Samples from a mixture of three Gaussian modes generated by vanilla Langevin dynamics (VLD) and chained vanilla Langevin dynamics (Chained-VLD) with patch size $Q \in \{1,4,10,20,50\}$. Three axes are $\ell_2$ distance from samples to the mean of the three modes. The samples are initialized in mode 2. } \label{fig:syn:mode2_patch}
\end{figure}

\begin{figure}
    \centering
    \begin{tabular}{c}
       \includegraphics[width=1\textwidth]{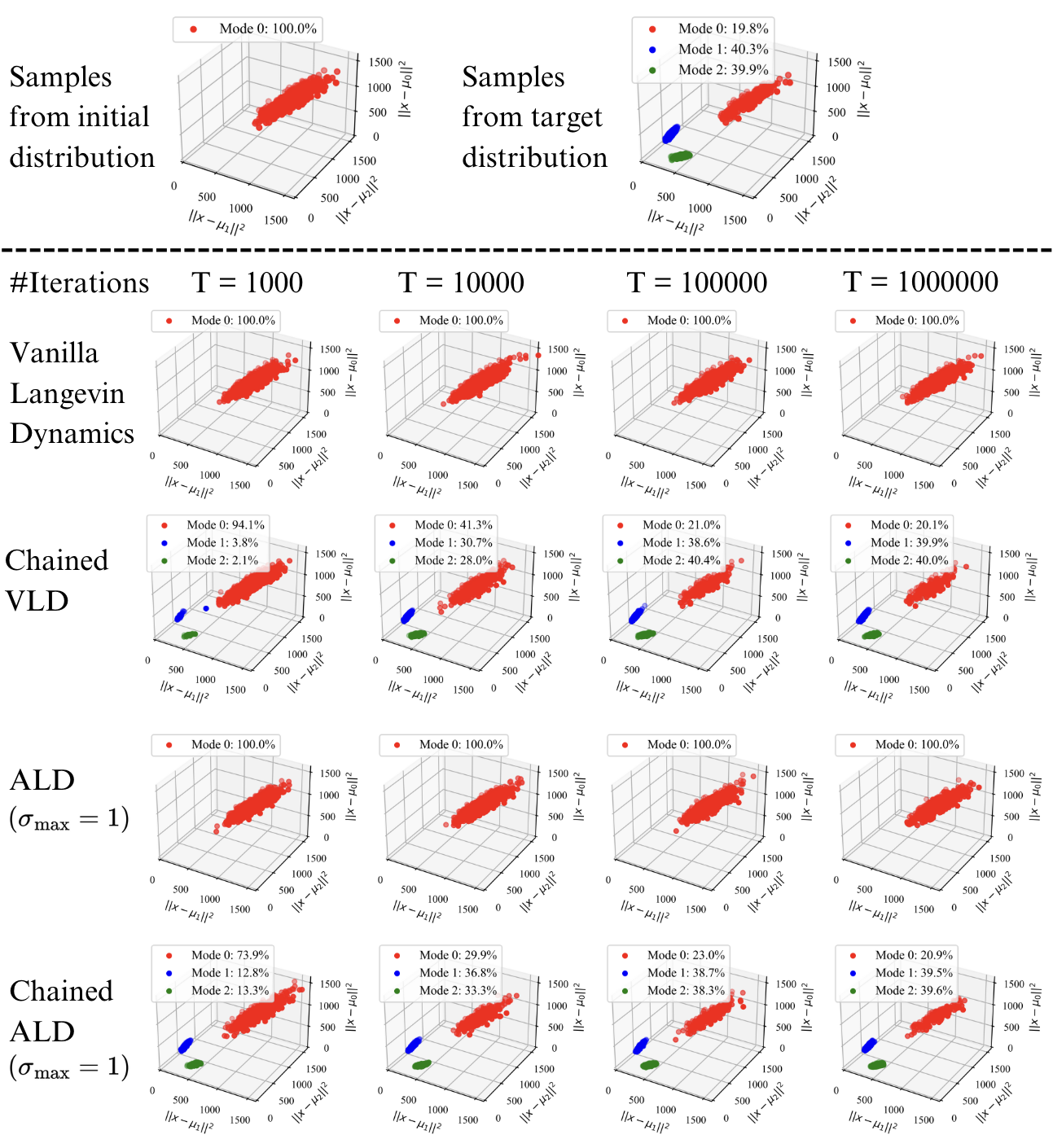}
      \end{tabular}
    \caption{Samples from a mixture of three Gaussian modes generated by Langevin dynamics and chained Langevin dynamics with patch size $Q=10$. Three axes are $\ell_2$ distance from samples to the mean of the three modes. The samples are initialized in mode 0.  }\label{fig:syn:mode0_ald}
\end{figure}

\begin{figure}
    \centering
    \begin{tabular}{c}
       \includegraphics[width=1\textwidth]{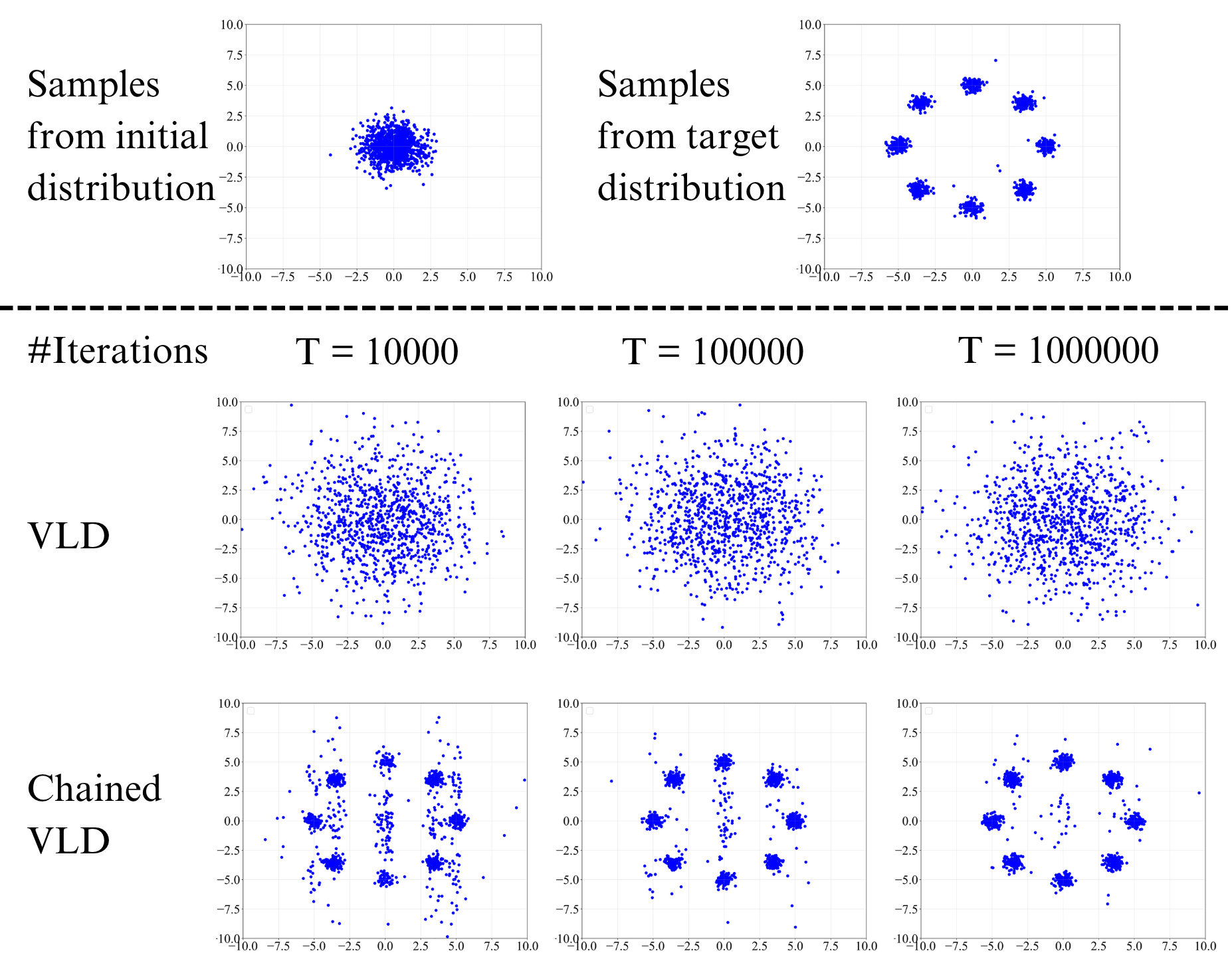}
      \end{tabular}
    \caption{Samples from a mixture of 9 Gaussian modes (including an in-between mode $\probz = \gN(\0_d, 10\mI_d)$) generated by vanilla Langevin dynamics (VLD) and chained vanilla Langevin dynamics (Chained-VLD) with patch size $Q=1$. Two axes are the first 2 coordinates of the samples. The samples are initialized in $\gN(\0_{100},\mI_{100})$.  }\label{fig:syn:mix9}
\end{figure}

\begin{figure}
    \centering
    \begin{tabular}{c}
       \includegraphics[width=1\textwidth]{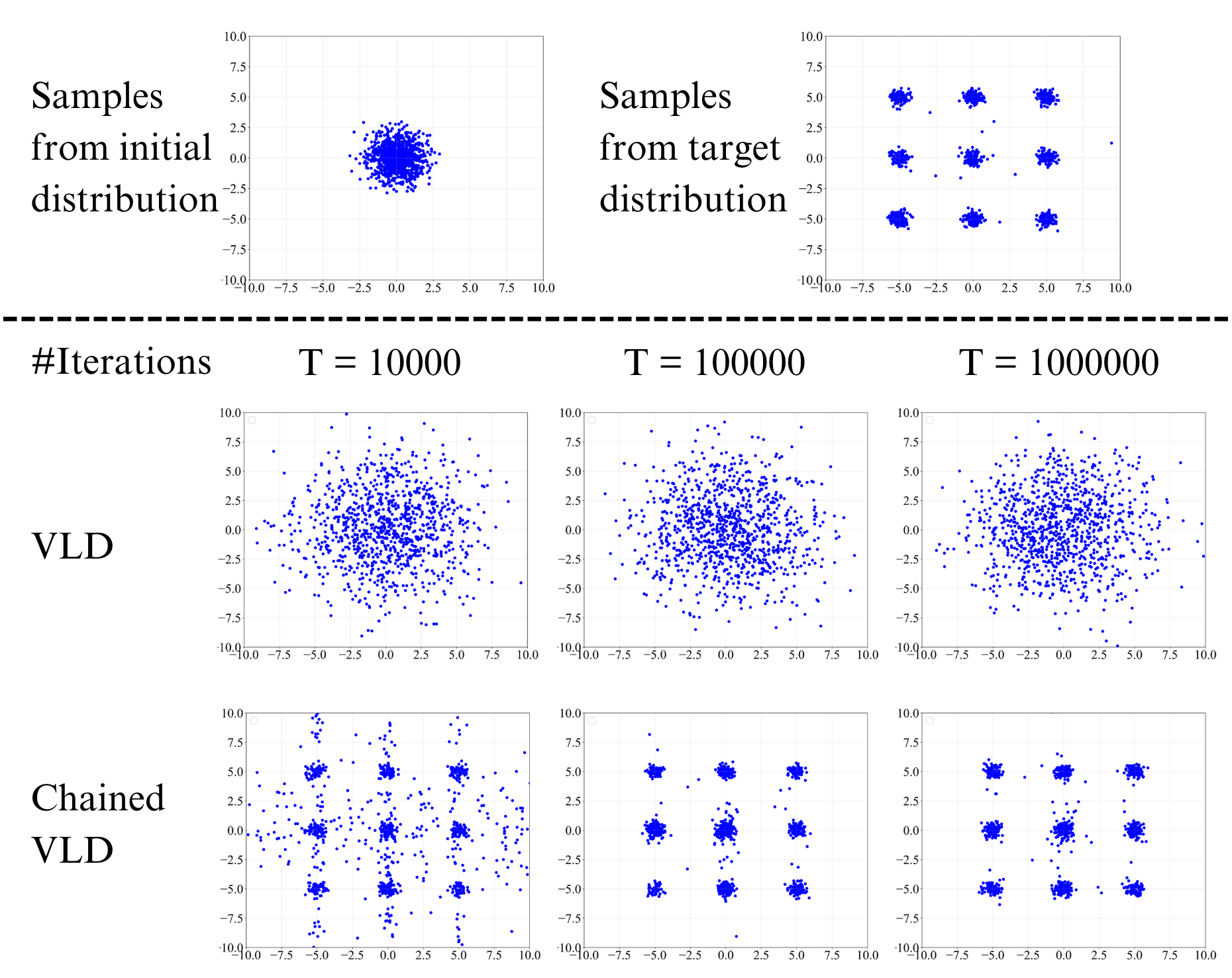}
      \end{tabular}
    \caption{Samples from a mixture of 10 Gaussian modes (including an in-between mode $\probz = \gN(\0_d, 10\mI_d)$) generated by vanilla Langevin dynamics (VLD) and chained vanilla Langevin dynamics (Chained-VLD) with patch size $Q=1$. Two axes are the first 2 coordinates of the samples. The samples are initialized in $\gN(\0_{100},\mI_{100})$.  }\label{fig:syn:mix10}
\end{figure}

\begin{figure}[t]
    \centering
    \begin{tabular}{c}
       \includegraphics[width=1\textwidth]{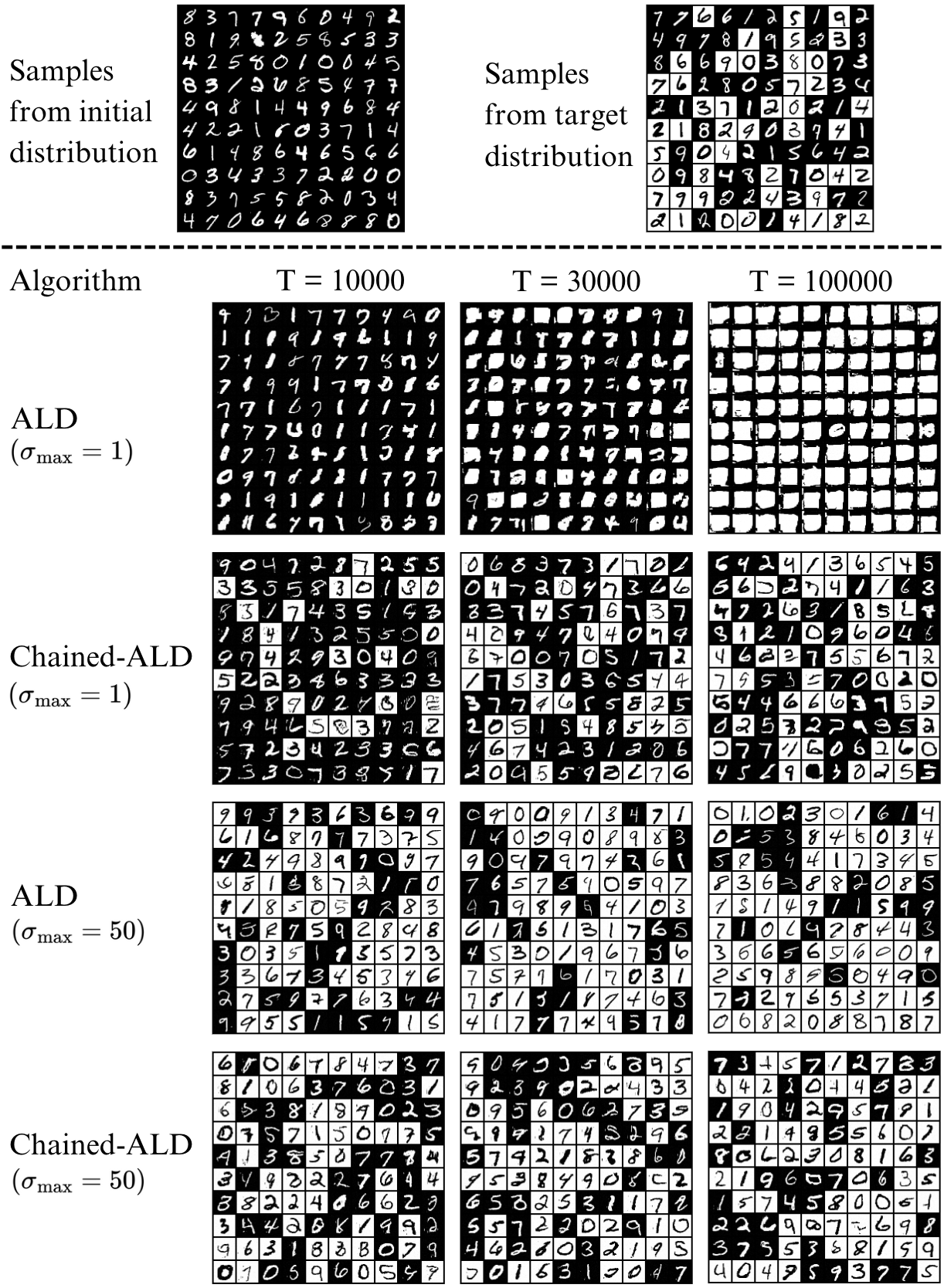}
      \end{tabular}
    \caption{Samples from a mixture distribution of the original and flipped images from the MNIST dataset generated by annealed Langevin dynamics (ALD) and chained annealed Langevin dynamics (Chained-ALD) for different numbers of iterations. The maximum noise level $\sigma_{\max}$ is set to be 1 or 50. The samples are initialized as original images from MNIST. } \label{fig:mnist_orig_ald}
\end{figure}

\begin{figure}
    \centering
    \begin{tabular}{c}
       \includegraphics[width=1\textwidth]{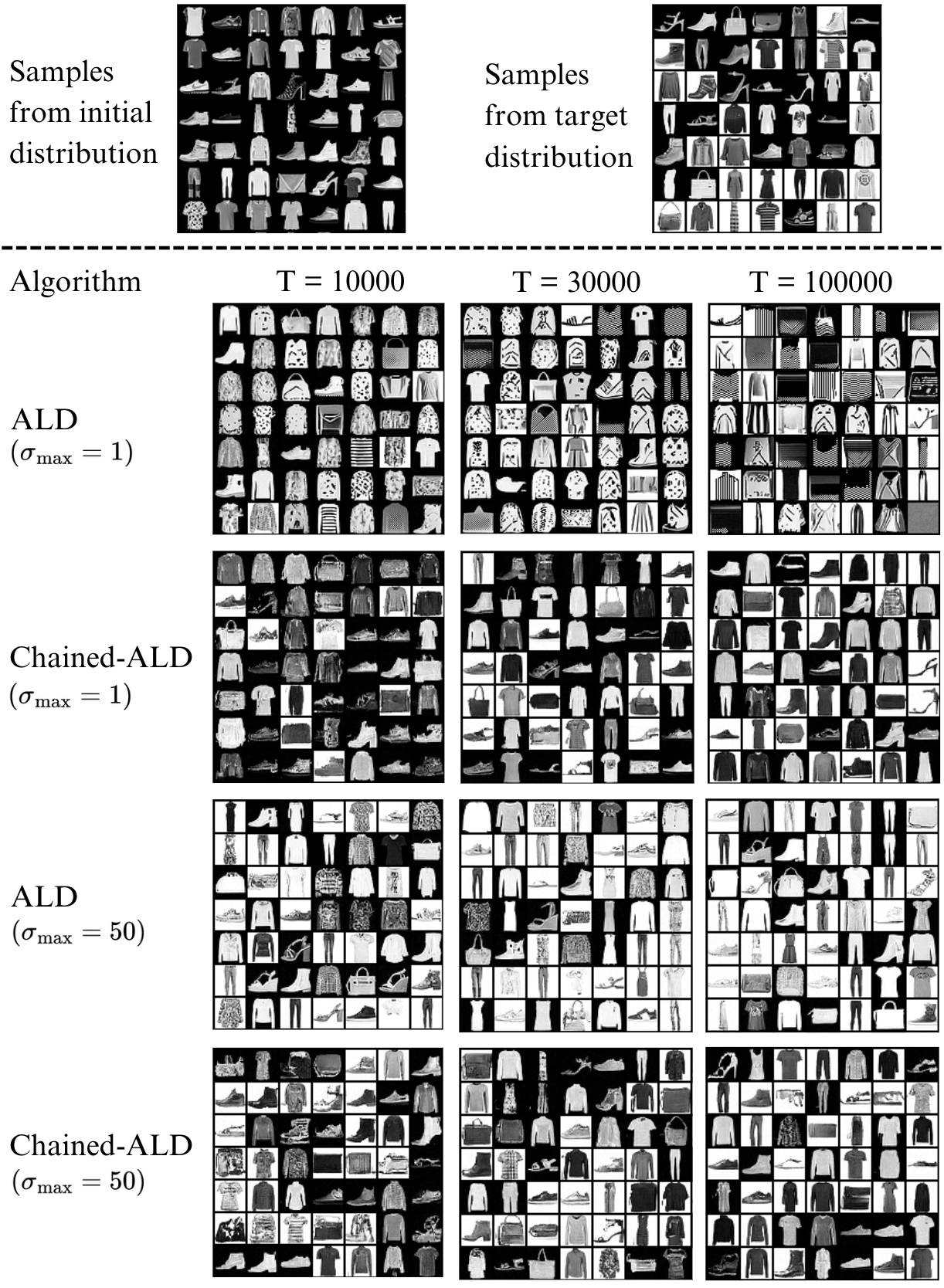}
      \end{tabular}
    \caption{Samples from a mixture distribution of the original and flipped images from the Fashion-MNIST dataset generated by annealed Langevin dynamics (ALD) and chained annealed Langevin dynamics (Chained-ALD) with patch size $Q=14$ for different numbers of iterations. The maximum noise level $\sigma_{\max}$ is set to be 1 or 50. The initialization is original images from Fashion-MNIST. } \label{fig:fashion_orig_ald}
\end{figure}

\begin{figure}[t]
    \centering
    \begin{tabular}{c}
       \includegraphics[width=1\textwidth]{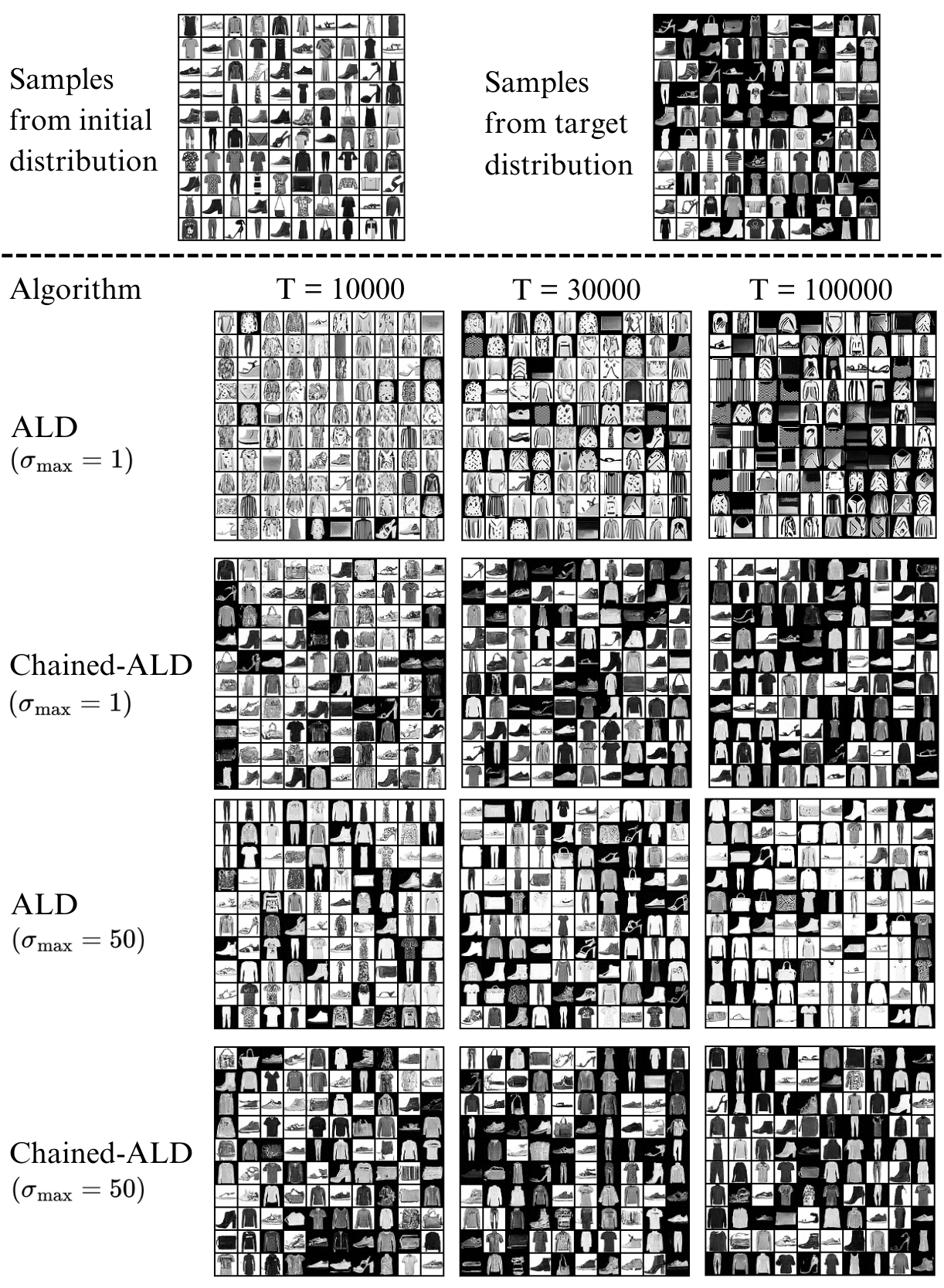}
      \end{tabular}
    \caption{Samples from a mixture distribution of the original and flipped images from the Fashion-MNIST dataset generated by annealed Langevin dynamics (ALD) and chained annealed Langevin dynamics (Chained-ALD) for different numbers of iterations. The maximum noise level $\sigma_{\max}$ is set to be 1 or 50. The samples are initialized as flipped images from FashionMNIST. } \label{fig:fashion_flip_ald}
\end{figure}

\end{document}